\documentclass{article} 
\usepackage[final]{colm2026_conference}

\usepackage{microtype}
\usepackage{hyperref}
\usepackage{url}
\usepackage{booktabs}
\usepackage{physics}

\usepackage{amsthm}
\usepackage{amsmath,amssymb}
\usepackage{bm}
\usepackage{enumitem}
\usepackage{graphicx}
\usepackage{subcaption}
\usepackage{tabularx}
\usepackage{tikz}
\usepackage[capitalize,noabbrev]{cleveref}
\usepackage[most]{tcolorbox}

\usepackage{amsmath,amsfonts,bm}

\def\eqref#1{equation~\ref{#1}}

\def\1{\bm{1}}

\newcommand{\R}{\mathbb{R}}

\DeclareMathAlphabet{\mathsfit}{\encodingdefault}{\sfdefault}{m}{sl}
\SetMathAlphabet{\mathsfit}{bold}{\encodingdefault}{\sfdefault}{bx}{n}

\newcommand{\E}{\mathbb{E}}

\newcommand{\softmax}{\mathrm{softmax}}

\def\*#1{\mathbf{#1}}
\def\+#1{\mathcal{#1}} 
\def\-#1{\mathrm{#1}}
\def\^#1{\mathbb{#1}}
\def\!#1{\mathtt{#1}}
\def\1#1{\mathbb{I}\left[#1\right]}
\newtheorem{theorem}{Theorem}
\newtheorem{lemma}{Lemma}

\newtheorem{corollary}{Corollary}
\newtheorem{definition}{Definition}
\newtheorem{remark}{Remark}

\usepackage{lineno}

\definecolor{darkblue}{rgb}{0, 0, 0.5}
\definecolor{myPurple}{HTML}{6E3C76}
\definecolor{myPurpleLight}{HTML}{F5EDF8}
\definecolor{middlegrey}{rgb}{0.75,0.75,0.75}
\definecolor{lightblue}{HTML}{F9FEFE}
\hypersetup{colorlinks=true, citecolor=darkblue, linkcolor=darkblue, urlcolor=darkblue}

\newcommand{\attn}{\mathrm{attn}}

\newcommand{\Ldeep}{L_{1a}}
\newcommand{\Wzoc}{W^{(1)}_0}
\newcommand{\Wzop}{W^{(1)}_1}
\newcommand{\Wonc}{W^{(2)}_0}
\newcommand{\Wonp}{W^{(2)}_1}
\newcommand{\Szero}{S^{[1]}}
\newcommand{\Sone}{S^{[2]}}
\newcommand{\Azero}{A^{[1]}}

\newcommand{\tctxend}{t^{\mathrm{ctx}}_{\mathrm{end}}}
\newcommand{\tctxv}{t^{\mathrm{ctx}}_v}
\newcommand{\tpromptend}{t^{\mathrm{prompt}}_{\mathrm{end}}}

\newcommand{\ustar}{u_{\mathrm{star}}}
\newcommand{\uend}{u_{\mathrm{end}}}
\DeclareMathOperator{\MASK}{MASK}

\title{How Transformers Learn to Plan via Multi-Token Prediction}

\author{\textbf{Jianhao Huang}\textsuperscript{1,}\thanks{Equal contribution. Email: \texttt{jhhuang2025@cs.ucla.edu,  zzp1012@sjtu.edu.cn}}\hspace{.3cm}
\textbf{Zhanpeng Zhou}\textsuperscript{2,}\footnotemark[1]\hspace{.3cm}
\textbf{Renqiu Xia}\textsuperscript{2}\hspace{.3cm}
\textbf{Baharan Mirzasoleiman}\textsuperscript{1} \\
\textbf{Weijie Su}\textsuperscript{3}\hspace{.3cm}
\textbf{Wei Huang}\textsuperscript{4,5,}\thanks{Correspondence to: \texttt{wei.huang.vr@riken.jp}} \\[0.8em]
\normalsize \textsuperscript{1}University of California, Los Angeles\hspace{.4cm}
\normalsize \textsuperscript{2}Shanghai Jiao Tong University \\
\normalsize \textsuperscript{3}University of Pennsylvania\hspace{.4cm}
\normalsize \textsuperscript{4}RIKEN Center for Advanced Intelligence Project \\
\normalsize \textsuperscript{5}The Institute of Statistical Mathematics}
\vspace{-2em}

\begin{document}

\ifcolmsubmission
\linenumbers
\fi

\maketitle

\begin{abstract}

While next-token prediction (NTP) has been the standard objective for training language models, it often struggles to capture global structure in reasoning tasks.
Multi-token prediction (MTP) has recently emerged as a promising alternative, yet its underlying mechanisms remain poorly understood.
In this paper, we study how MTP facilitates reasoning, with a focus on planning.
Empirically, we show that MTP consistently outperforms NTP on both synthetic graph path-finding tasks and more realistic reasoning benchmarks, such as Countdown and boolean satisfiability problems.
Theoretically, we analyze a simplified two-layer Transformer on a star graph task.
We prove that MTP induces a two-stage \emph{reverse reasoning} process: the model first attends to the end node and then reconstructs the path by tracing intermediate nodes backward.
This behavior arises from a gradient decoupling property of MTP, which provides a cleaner training signal compared to NTP.
Ultimately, our results highlight how multi-token objectives inherently bias optimization toward robust and interpretable reasoning circuits.
\end{abstract}

\section{Introduction}\label{sec:intro}
Recently, large language models (LLMs) have demonstrated a range of emerging abilities as they scale, with reasoning becoming one of the most critical capabilities~\citep{ke2025a}.
Reasoning enables language models to go beyond simple knowledge retrieval, allowing them to actively process and infer solutions to new problems.
Despite these advances, how such reasoning abilities emerge in language models remain poorly understood.
Recent theoretical advances focused on the role of chain-of-thought~\citep{feng2023cot,kim2025transformers,huang2025transformers} and architecture design~\citep{gatmiry2024can,saunshi2025reasoning} in enabling reasoning.
However, a more fundamental factor may lie in the modeling paradigm itself, particularly the training objective that underlies language model training.

Next-token prediction (NTP) with teacher forcing has long been the standard objective for training language models, where models are optimized to produce the next token given all preceding ground-truth context tokens.
Yet, it is limited in capturing long-term dependencies~\citep{qi2020prophetnet} and prone to overfitting to local patterns~\citep{bachmann2024pitfalls}. 
To address these limitations, multi-token prediction (MTP)~\citep{gloeckle2024better} has recently been proposed.
Instead of predicting only the next token, it predicts multiple future tokens in parallel.
MTP substantially improves performance on math and code benchmarks that require complex reasoning, and has been widely adopted in leading models such as DeepSeek-V3~\citep{liu2024deepseek}.
Understanding how MTP leads to such improvements is crucial for uncovering the mechanisms underlying the reasoning capabilities of modern language models.

\begin{figure}[tb!]
    \centering
    \begin{subfigure}[b]{0.14\textwidth}
        \centering
        \includegraphics[width=\textwidth]{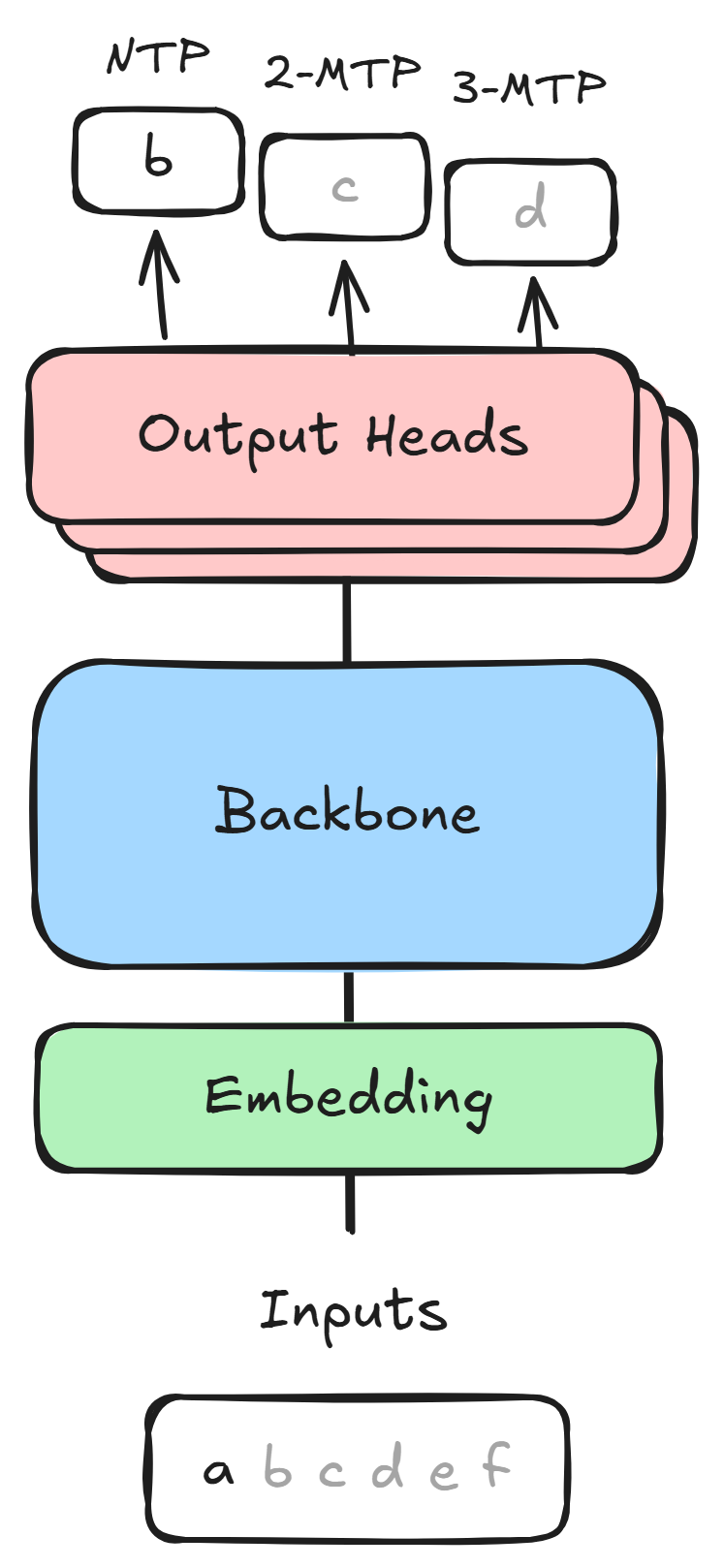}
        \caption{}
        \label{fig:architecture}
    \end{subfigure}
    \hfill 
    \begin{subfigure}[b]{0.80\textwidth}
        \centering
        \includegraphics[width=\textwidth]{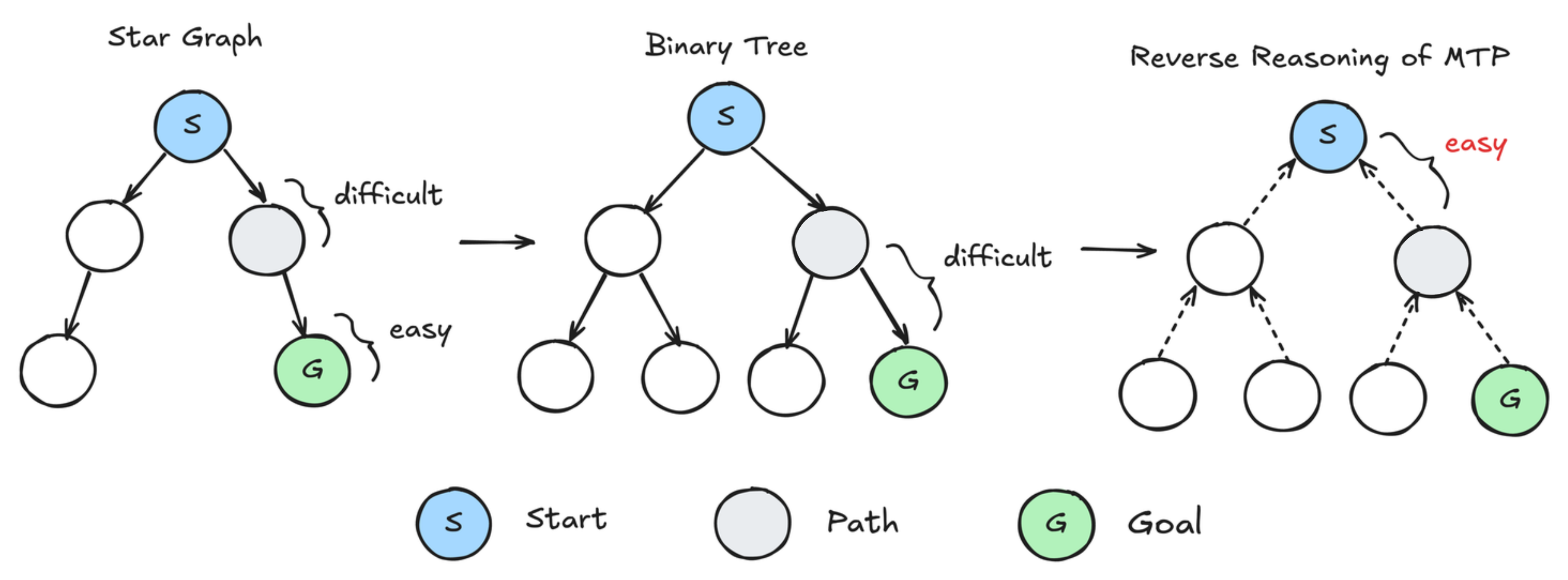}
        \caption{}
        \label{fig:star_vs_bin}
    \end{subfigure}
    \caption{\textbf{Overview of the MTP architecture and its advantage in reasoning tasks.} (a) The MTP architecture. Building upon \citet{gloeckle2024better}, our MTP employs a shared backbone with multiple independent output heads to predict several future tokens simultaneously. (b) Illustration of Star Graph and Binary Tree. In standard forward planning, early steps can be difficult due to a large search space. MTP facilitates ``reverse reasoning" by looking ahead to the goal ($G$), transforming complex forward searches into easy backward steps.}
\end{figure}

In this work, we study how MTP facilitates reasoning in language models, with a focus on \emph{planning}\footnote{Reasoning refers to the ability to perform logical or compositional inference, whereas planning is a specific form of reasoning that requires anticipating future steps and computing a global solution before generating outputs.}, which requires considering future steps before generating the current output.
Among MTP variants, we adopt the formulation of \citet{gloeckle2024better}, which predicts multiple future tokens using parallel heads.
Note that our primary focus is on the reasoning performance rather than inference efficiency.
Though the model is trained to predict multiple tokens, it generates only a single token at each step using the first head during inference. 


We begin with empirical observations on synthetic graph path-finding tasks, which are widely regarded as a standard abstraction for studying reasoning~\citep{kim2025metastable,ye2025beyond}. 
Given the graph structure and start/end nodes, the model is required to generate the full path.
We find that MTP consistently outperforms NTP on these tasks. 
We first study the star graph (see \cref{fig:star_vs_bin}), a directed graph in which multiple paths originate from the start node, with one leading to the target. 
In a star graph, only the first step requires a non-trivial decision. 
Thus, under NTP with teacher forcing, the model tends to follow edges based on previously revealed nodes in the prefix, a phenomenon known as the Clever Hans cheat~\citep{bachmann2024pitfalls}. 
To remove this shortcut, we introduce a more complex task, the \emph{binary tree}, in which decisions must be made at every step. 
Notably, MTP continues to outperform NTP, indicating that its advantage cannot be explained solely by disabling the Clever Hans cheat. 
We evaluate on more realistic tasks, including Countdown~\citep{gandhi2024stream} and boolean satisfiability problems, where MTP still succeeds.




We next delve deeper into the underlying mechanisms behind the success of MTP on planning tasks, with a focus on the star graph.
For simplicity, we study a two-layer disentangled Transformer~\citep{friedman2023learning} on the 2-path 3-node star graph.
Our theory reveals a two-stage \emph{reverse reasoning} process in models trained with MTP: the model first attends to the end node, then retrieves intermediate nodes by following edges that point to it.
Intuitively, solving the star graph is easier when starting from the end node, which extends to the binary tree task.
Importantly, MTP modifies only the training objective, while the inference process remains the same as NTP, indicating that reverse reasoning arises from optimization differences.
Under NTP, learning signals from different layers are entangled, making it difficult for the model to discover the patterns needed for planning.
In contrast, MTP provides an isolated training signal, allowing the first layer to attend to the end node and the second layer to recover the intermediate node through simple edge matching.

\textbf{Technical novelty.}
We are the first to formally analyze the convergence dynamics of MTP and establish its difference from NTP, extending beyond prior results~\citep{zhong2025understanding}.
The key is the \emph{gradient decoupling} property of the MTP head, which provides an isolated training signal to the first layer.
This property enables the emergence of \emph{reverse reasoning} mechanism.
Our technique might have independent interest to learning theory.
\section{Related Work}\label{sec:related_work}
\textbf{Understanding of MTP.} 
Despite the empirical success of MTP, a rigorous theoretical understanding of its underlying mechanisms remains unclear. 
Recent work~\citep{zhong2025understanding} studied how MTP enhances planning on synthetic tasks.
However, its analysis is largely qualitative and does not provide a formal comparison with NTP.
Moreover, its theoretical setup considers a simplified shared-head variant of MTP, rather than the standard independent-head architecture introduced by \citet{gloeckle2024better}.
\emph{In contrast}, our work provides a theoretical comparison between NTP and MTP under the more general independent-head architecture.
We mathematically unpack how the gradient decoupling property of MTP naturally induces a two-stage \emph{reverse reasoning} process on the star graph, providing a theoretical explanation for its empirical advantage over NTP. For an extended discussion of related work on MTP, we refer to \cref{app:related work}.

\textbf{Theories of Transformers.}
Extensive theoretical studies have been conducted to characterize the mechanisms of Transformer architectures, covering aspects such as in-context learning~\citep{zhang2024trained,pmlr-v247-siyu24a,wei2026how}, chain of thought~\citep{kim2025transformers,huang2025transformers,wen2025from}, induction head~\citep{nichani2024how,chen2024unveiling}, benign overfitting~\citep{jiang2024unveil,magen2024benign}, attention sink~\citep{guo2025activedormant} and compositional reasoning~\citep{pmlr-v291-wang25a, wang2026the}.
Within this landscape, disentangled transformers~\citep{friedman2023learning,nichani2024how,chen2024unveiling,wei2026how} have gained significant attention, as they decouple the intertwined features in the residual stream to facilitate a more tractable analytical framework. 
Following this paradigm, we adopt a disentangled Transformer as the structural foundation for our study.






























\section{Preliminaries}

\textbf{Notation and Basic Setups.} 
We use $[N] = \{1, \ldots, N\}$ to denote the vertex set and $\{e_i\}_{i=1}^N$ for the standard basis of $\mathbb{R}^N$. For a sequence of length $T$, we write $x_{1:T} = (x_1, \ldots, x_T)$ and use $[\cdot, :]$ to denote a row slice of a matrix. The softmax function $\sigma : \mathbb{R}^k \to \mathbb{R}^k$ is defined by $\sigma({v})_i = \exp(v_i) / \sum_j \exp(v_j)$, applied row-wise. For a softmax distribution $s$, we write $J(s) := \mathrm{diag}(s) - ss^\top$ for its Jacobian. We use $I_T$ for the $T \times T$ identity matrix and $L \in \mathbb{R}^{T \times T}$ for the strictly lower shift matrix with $L_{i,i-1} = 1$ and zero elsewhere.

\textbf{Training Objectives.}
Let $x_{1:T} = (x_1, \ldots, x_T)$ be an input sequence.
The standard NTP objective minimizes the negative
log-likelihood of each token given its full causal prefix:
\begin{equation}
    \mathcal{L}_{\text{NTP}}(\theta)
    = -\frac{1}{T}\sum_{t=1}^{T} \log p_\theta \left(x_t \mid x_{1:t-1}\right).
\end{equation}
MTP with lookahead $k$ instead predicts the next
$k$ tokens in parallel from the same prefix, averaging the cross-entropy over all
$k$ heads:
\begin{equation} \label{eq:mtp_loss}
    \mathcal{L}_{\text{MTP}}^{(k)}(\theta)
    = -\frac{1}{T}\sum_{t=1}^{T}
      \frac{1}{k}\sum_{m=1}^{k}
      \log p_{\theta,m} \left(x_{t+m} \mid x_{1:t}\right),
\end{equation}
where $p_{\theta,m}$ denotes the $m$-th prediction head.
In both objectives, training uses teacher forcing: the ground-truth prefix
$x_{1:t}$ is supplied as context regardless of the model's own predictions.
Note that NTP is a special case of MTP with $k=1$.


\textbf{Architectural Design.} We adopt a parallel architecture \citep{gloeckle2024better} as illustrated in \cref{fig:architecture}, which utilizes a shared transformer backbone capped with multiple independent linear output heads for simultaneous MTP. This contrasts with the sequential structure employed by \cite{liu2024deepseek}. Our choice is motivated by the goal of deconstructing the underlying mechanisms of MTP rather than maximizing absolute performance. By prioritizing a streamlined architecture, we ensure better theoretical tractability while still effectively capturing the empirical advantages of MTP over standard NTP.

\textbf{Main Experimental Setups.} Across all experiments, models are trained using the MTP objective with the architecture described above. However, to evaluate the core reasoning and generative capabilities acquired during training, all models are evaluated using standard NTP inference. For statistical reliability, all results are averaged over 3 independent runs for each data point. Details of the experimental configurations are provided in \cref{app:exp_dense}. 






\section{Planning Emerges with Multi-Token Prediction}\label{sec:exp}
In this section, we provide empirical evidence that MTP enables the planning abilities in Transformers. 
On complex planning tasks, including star graph~\citep{bachmann2024pitfalls}, countdown~\citep{gandhi2024stream}, and boolean satisfiability, MTP consistently outperforms NTP across a wide range of model and data scales.

\textbf{Star Graph: NTP Fails, but MTP Succeeds.} 
We begin with one of the simplest planning tasks, the star graph. 
A star graph is a directed graph with $v_{\rm start}$ as the central node and at least two paths originating from it. 
One of these paths leads to the target node $v_{\rm end}$. 
Given the graph and the start/end nodes, the goal is to find the path from $v_{\rm start}$ to $v_{\rm end}$. 
See \cref{fig:star_vs_bin} for an example.
In this task, the model cannot fully rely on previously generated tokens to determine the next token; instead, it must combine information across multiple path segments to infer reachability. 
\citet{bachmann2024pitfalls} showed that NTP fails on the star graph task. 
We replicate this finding and further evaluate MTP on the same task.

As shown in \cref{fig:star_bin_exp}a-b, we train Transformers with MTP on star graphs with two paths, each of five nodes.
Specifically, we compare MTPs with different lookahead steps $k$; when $k=1$, MTP naturally reduces to NTP.
Following \citet{kaplan2020scaling,hoffmann2022training}, we evaluate MTP and NTP under both data-scaling and parameter-scaling setups.
In \cref{fig:star_left}, as the training data size increases, the path-finding accuracy of models trained with MTP ($k \geq 2$) surpasses that of NTP.
Surprisingly, even $k=2$, which predicts only one additional future token beyond NTP, achieves $100\%$ accuracy when trained on $0.5\text{M}$ samples, whereas NTP remains at $50\%$ accuracy across different training data sizes.
Similar trends are observed under parameter scaling in \cref{fig:star_right}.
These findings indicate the potential of MTP for planning tasks over NTP.

\textbf{Binary Tree: No Clever Hans Cheat, yet MTP Still Wins.}
\citet{bachmann2024pitfalls} attributed the failure of NTP on star graph tasks to the \emph{Clever Hans Cheat} phenomenon.   
This phenomenon arises from teacher-forced training, where the model is given the prefix of the ground-truth answer while predicting the next token. 
The model may exploit trivial correlations between previously revealed answer tokens and the next token, instead of solving the underlying task.
In the star graph task, when predicting node $v_i$, the model already sees previous node $v_{i-1}$ as input. 
Since every intermediate node has only one outgoing edge, the model can simply follow the edge from $v_{i-1}$ to predict $v_i$. 
Consequently, the model fits the training data but never learns the true path-planning mechanism.

\citet{bachmann2024pitfalls} further introduced \emph{teacherless} training to eliminate the Clever Hans Cheat by replacing the ground-truth prefix with dummy tokens. 
The model must predict multiple future tokens without relying on previously revealed answers. 
Conceptually, this teacherless objective is closely related to modern MTP. 
Then, a natural question arises: 
\begin{center}
    \emph{Does MTP enable better planning simply by bypassing the cheating mechanism?}
\end{center}

\begin{figure}[tb!]
  \centering
  
  \begin{subfigure}[b]{0.45\textwidth}
    \centering
    \caption{Star graph: accuracy vs. data}
    \includegraphics[width=\textwidth]{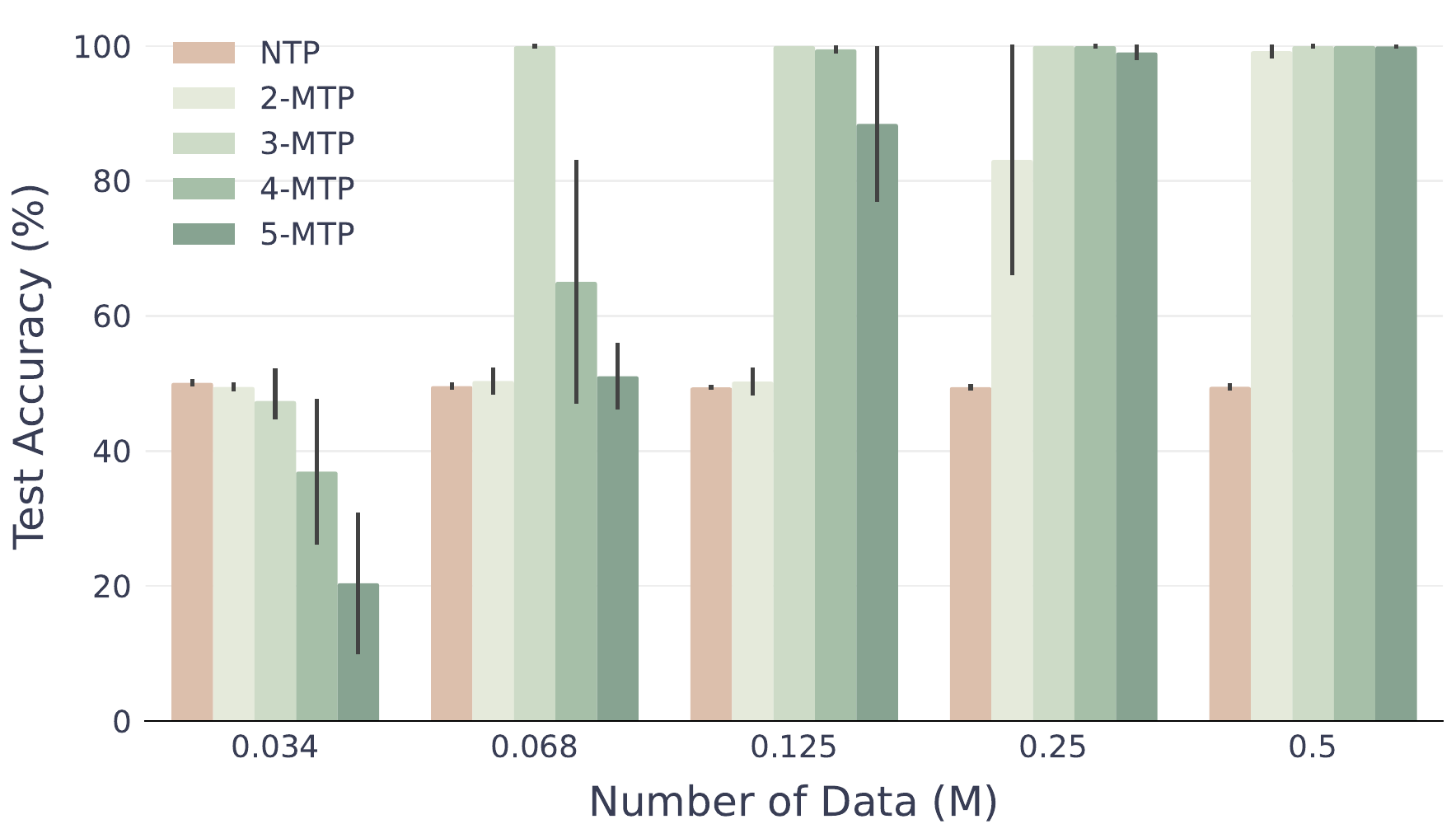}
    \label{fig:star_left}
  \end{subfigure}
  \hfill 
  \begin{subfigure}[b]{0.45\textwidth}
    \centering
    \caption{Star graph: accuracy vs. params}
    \includegraphics[width=\textwidth]{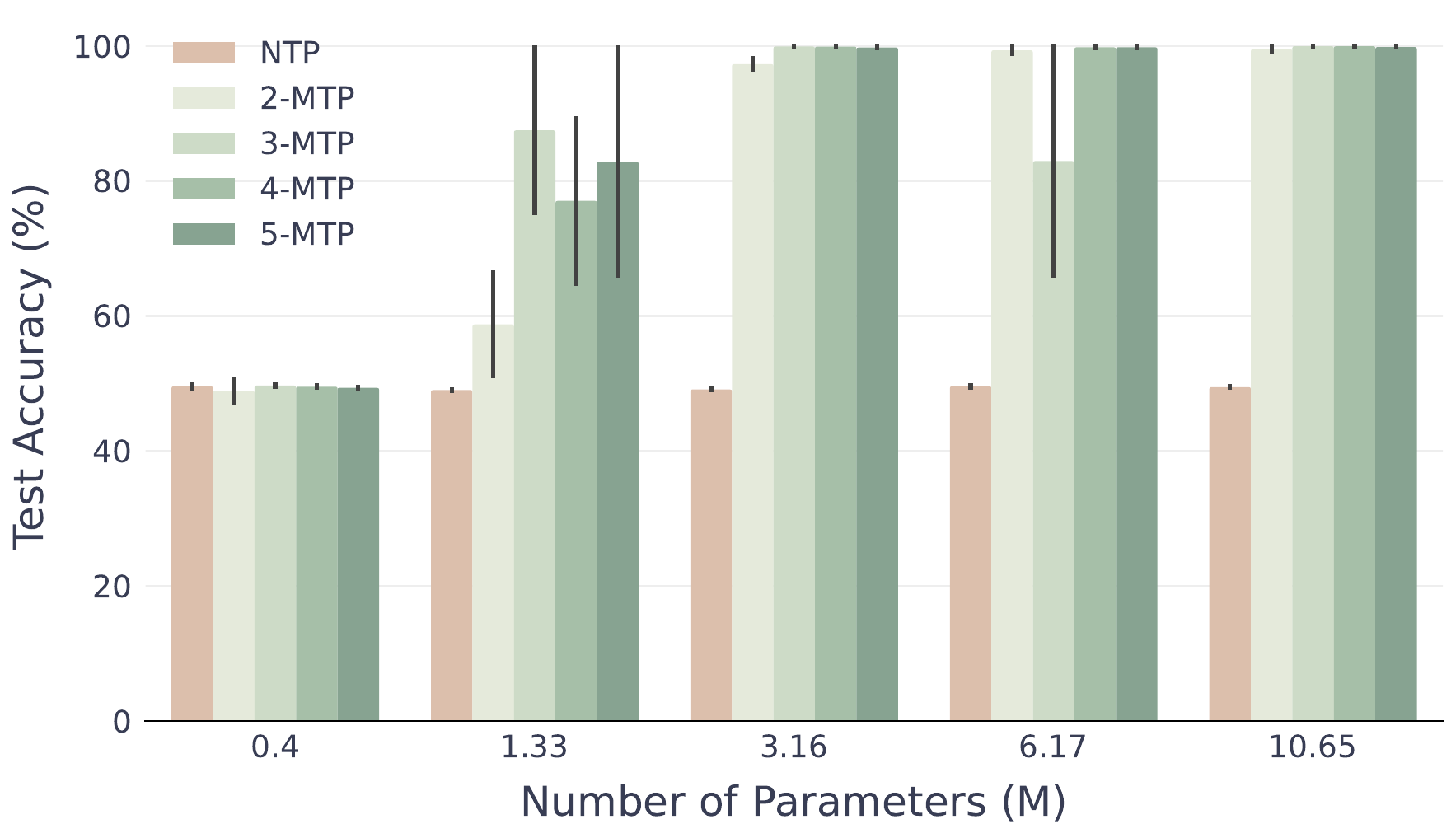}
    \label{fig:star_right}
  \end{subfigure}



  \begin{subfigure}[b]{0.45\textwidth}
    \centering
    \caption{Binary tree: accuracy vs. data}
    \includegraphics[width=\textwidth]{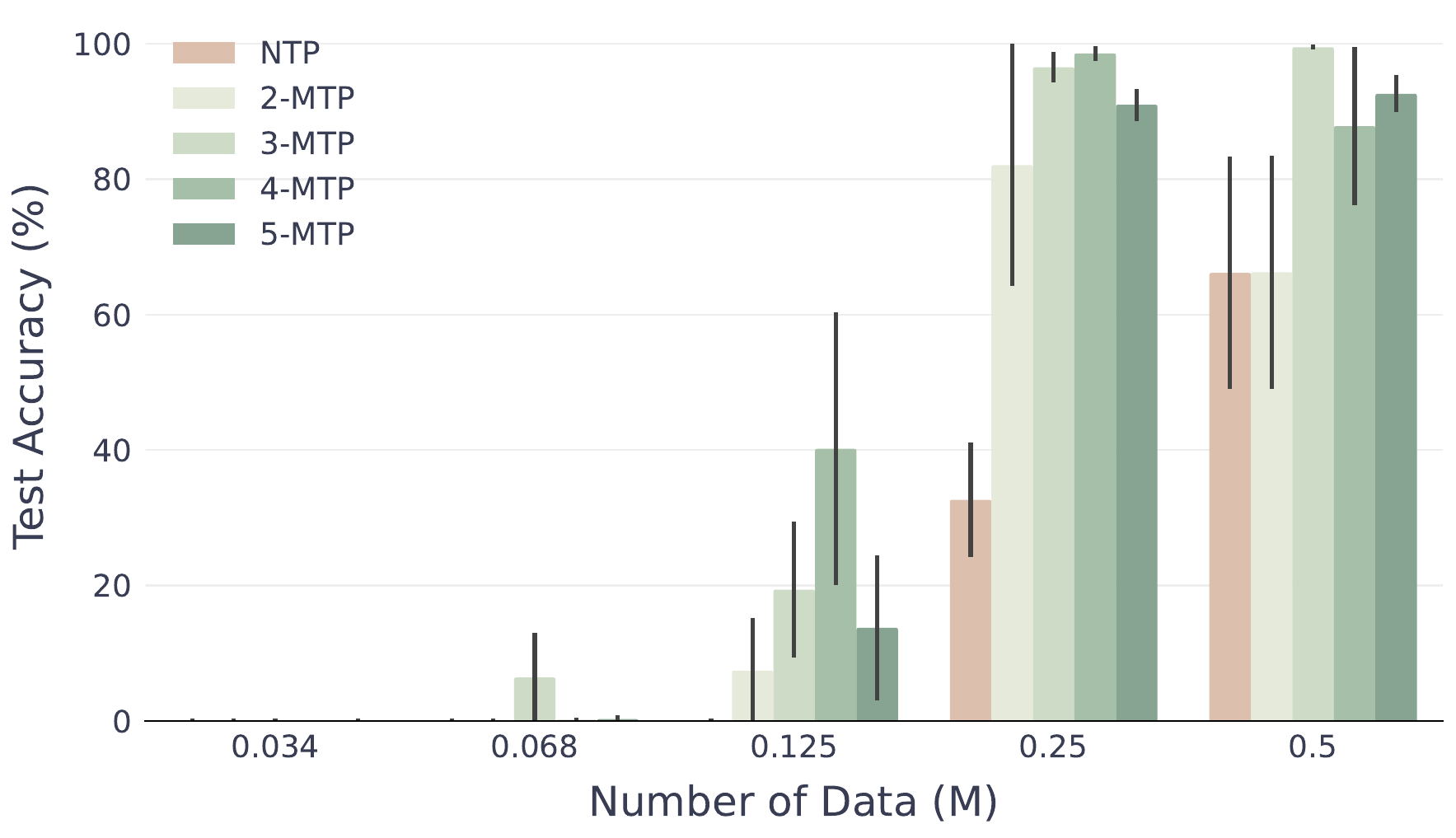}
    \label{fig:bin_left}
  \end{subfigure}
  \hfill 
  \begin{subfigure}[b]{0.45\textwidth}
    \centering
    \caption{Binary tree: accuracy vs. params}
    \includegraphics[width=\textwidth]{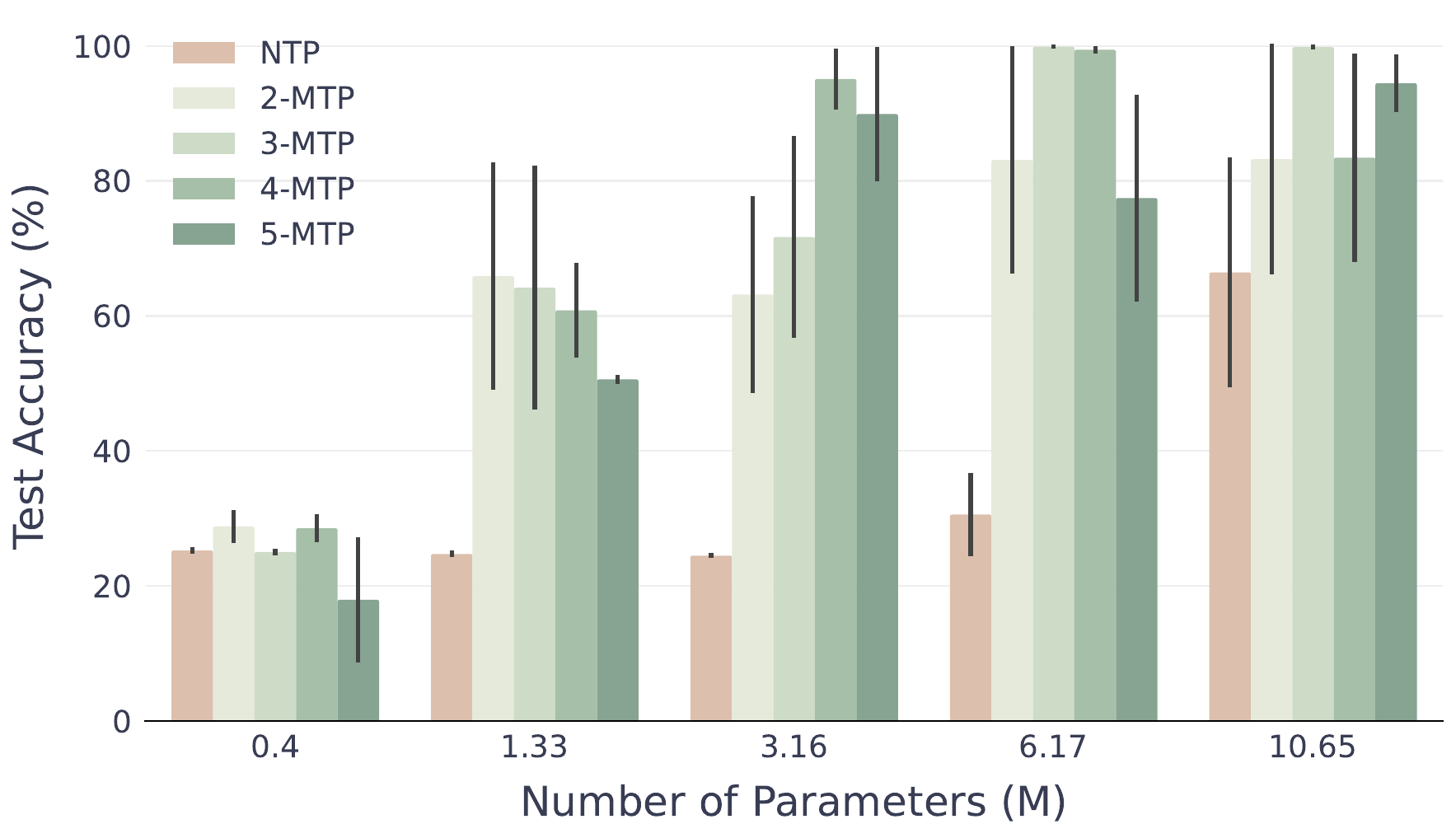}
    \label{fig:bin_right}
  \end{subfigure}
  \caption{Path-finding accuracy of MTP versus NTP under varying data and parameter scales. \textbf{(a, b)} 2-path 5-node star graph: test accuracy as a function of training data size (16M model) and parameter size (0.5M data). \textbf{(c, d)} 5-node binary tree: test accuracy as a function of data size (16M model) and parameter size (1M data). Error bars represent 90\% confidence intervals calculated from three independent replicates ($n = 3$). Notably, while NTP stagnates at 50\% accuracy on star graph, it begins to learn and scale on binary trees. Across both topologies, MTP consistently demonstrates superior planning abilities.}
  \label{fig:star_bin_exp}
\end{figure}
To answer this question, we introduce a new task, the \emph{binary tree} problem, which extends the two-path star graph design.
See \cref{fig:star_vs_bin} as an illustration. 
In the star graph
, only the start node branches into two paths, whereas the binary tree branches at every intermediate node.
Thus, each prediction along the correct path cannot rely on previously revealed nodes, preventing the same cheating as in the star graph.

Now we evaluate MTP and NTP on the binary tree tasks.
As before, we adopt both data-scaling and parameter-scaling setups. 
In \cref{fig:star_bin_exp}c-d, models trained with MTP still outperform those trained with NTP in most settings, especially when data or model capacity is limited. 
Unlike the star graph experiments, the performance of NTP improves as the data size or model size increases. 
It is clear that, without the Clever Hans cheat, NTP begins to learn planning, while MTP still demonstrates stronger planning ability.
We conclude that MTP facilitates planning not merely by bypassing the cheating mechanism; deeper mechanisms are involved, which we discuss further in \cref{sec:theory}.

\textbf{More Planning Tasks: Countdown and SAT.}
To further validate planning abilities of MTP, we consider two more realistic planning tasks: Countdown and Boolean satisfiability (SAT).
Both tasks require computing the full solution internally before generating the first token, thereby demanding extensive planning over a large number of combinations. See \cref{app:illustration} for illustrations of these tasks.

\begin{itemize}[leftmargin=6.3em]
    \item[\textbf{{\color{blue!30!black} Countdown.}}] Countdown is a generalized version of the Game of 24. 
    The goal is to use the given numbers and arithmetic operations to reach a target number. 
    \citet{yao2023treeofthoughts} showed even frontier models such as GPT-4 struggle with Countdown problems.
    \item[\textbf{{\color{blue!30!black} SAT.}}] SAT is a classic problem in computer science and was proven to be NP-complete~\citep{cook1971complexity}. 
    The goal is to find a Boolean assignment to variables that satisfies all clauses of a given propositional formula.
\end{itemize}

\begin{table}[t!]
  \centering
  \newcolumntype{C}{>{\centering\arraybackslash}X} 
  
  \begin{tabularx}{\textwidth}{lCCCCCCC}
    \toprule
    \textbf{TASK} & \textbf{NTP} & \textbf{2-MTP} & \textbf{3-MTP} & \textbf{4-MTP} & \textbf{5-MTP} & \textbf{6-MTP} & \textbf{7-MTP} \\
    \midrule
    \textbf{Countdown} & 60.27 & 60.36 & 63.20 & 62.09 & 62.75 & 63.17 & \textbf{64.93} \\
    \textbf{3-SAT} & 10.40 & 28.17 & 69.50 & 63.10 & 82.83 & 82.00 & \textbf{87.47} \\
    \bottomrule
  \end{tabularx}
  \caption{Test accuracy of models trained with NTP versus MTP configurations on Countdown and 3-SAT tasks. Models trained with MTP achieve consistently higher test accuracy than the NTP baseline across both tasks.}
  \label{tab:extended_results}
\end{table}

In \cref{tab:extended_results}, we present the performance of models trained with MTP and NTP on these two tasks.
Evidently, models trained with MTP achieve consistently higher test accuracy than those trained with NTP.

In summary, we highlight two key findings.
First, MTP outperforms NTP on star graph tasks, not merely by disabling the Clever Hans Cheat; deeper mechanisms are involved. 
Second, MTP consistently enables better planning in Transformers across both graph path-finding and more realistic tasks.




\section{Mechanisms of Planning under MTP: Reverse Reasoning} 
\label{sec:theory}
We have seen that MTP significantly improves planning in language models. 
In this section, we investigate the underlying mechanisms of this improvement through a theoretical analysis of a two-layer disentangled Transformer on the star graph task. 
Going beyond disabling the Clever Hans cheat, we show that MTP induces a \emph{reverse reasoning} process. All proofs in this section are deferred to Appendices \ref{sec:tech} and \ref{sec:proof}.

\subsection{Problem Setup and Model} \label{sec:prob_model}



We focus on the 2-path, 3-node star graph with $N = T = 10$. Given the
serialized graph and a query pair $(u_{\mathrm{end}}, u_\mathrm{star})$, the model must predict
the path $( u_\mathrm{star}, v, u_{\mathrm{end}})$ autoregressively. The critical challenge is
predicting the intermediate node $v$: this requires determining which path
from $ u_\mathrm{star}$ reaches $u_{\mathrm{end}}$, a decision demanding genuine planning.


\textbf{Input encoding.}
Let $\{e_i\}_{i=1}^{N}$ be the standard basis of~$\R^N$. Given an ordered
edge list $E = \{(u_1,v_1),\ldots,(u_M,v_M)\}$ and the query pair
$(\uend, \ustar)$, content embeddings $z_{1:T} \in (\R^N)^T$ with $T = 2M+2$ is
\[
  z_{2i-1} = e_{u_i},\; z_{2i} = e_{v_i}\ (i = 1,\ldots,M),
  \qquad z_{2M+1} = e_{\uend},\quad z_{2M+2} = e_{\ustar}.
\]
The full input tokens are $x_t = (z_t; e_t) \in \mathbb{R}^{N+T}$. Each edge is encoded as two consecutive tokens; the last two tokens are
$\uend$ then $\ustar$. Placing $\uend$ before $\ustar$ ensures that $\uend$
is visible in the prompt when the model predicts $v$, matching the
autoregressive generation order. We write $\tctxend$ for the position of $\uend$ in the context edge list,
$\tctxv = \tctxend - 1$ for the position of $v$, and
$\tpromptend = T - 1$ for $\uend$ in the prompt.


\textbf{Training objective.}
We apply MTP with lookahead $k=2$, targeting
$y^{(1)} = v$ and $y^{(2)} = \uend$. Let
$ Z = [z_1,\ldots,z_T]^\top$ be the MTP context and
$Z' = [z_1,\ldots,z_{T-1}, e_{y^{(1)}}]^\top$ the AR context (replacing
$z_T$ with $e_{y^{(1)}}$). The loss combines deep-head NTP terms with a
shallow MTP head:
\begin{equation}\label{eq:total-loss}
  \mathcal{L}(G) = -\frac{1}{2}\Bigg[
    \frac{1}{2}\Big(
      \underbrace{\log f_1(Z) e_{y^{(1)}}}_{\mathcal{L}_{1a}}
      + \underbrace{\log f_1(Z') e_{y^{(2)}}}_{\mathcal{L}_{1b}}
    \Big)
    + \underbrace{\log f_2(Z) e_{y^{(2)}}}_{\mathcal{L}_2}
  \Bigg].
\end{equation}

\begin{figure}[!t]
    \centering
    
    \begin{subfigure}[b]{0.23\textwidth}
        \centering
        \caption{NTP $\widetilde W^{(1)}$}
        \includegraphics[width=\textwidth]{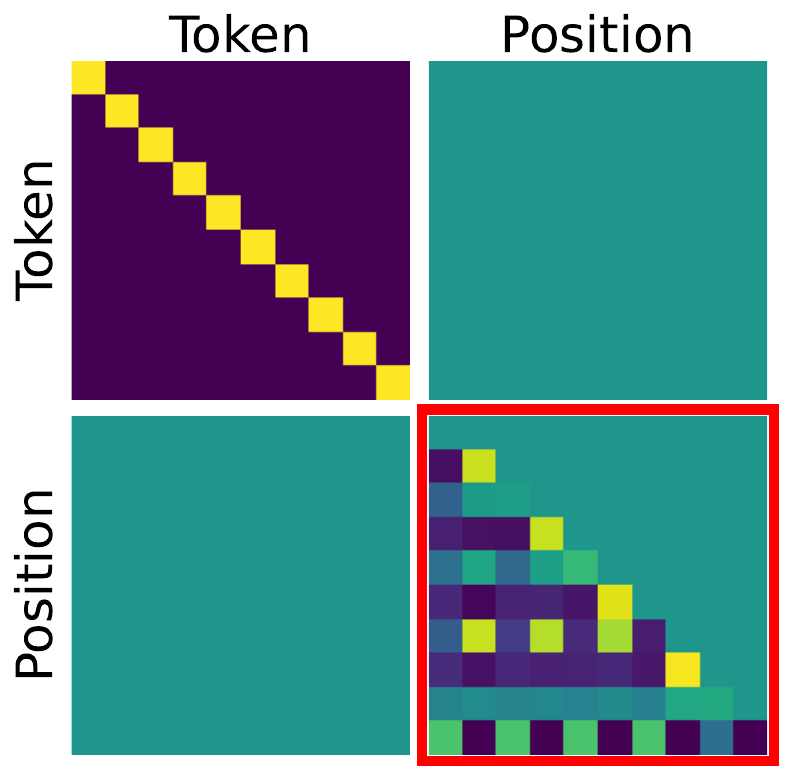}
    \end{subfigure}
    \hfill
    \begin{subfigure}[b]{0.225\textwidth}
        \centering
        \caption{NTP $\widetilde W^{(2)}$}
        \includegraphics[width=\textwidth]{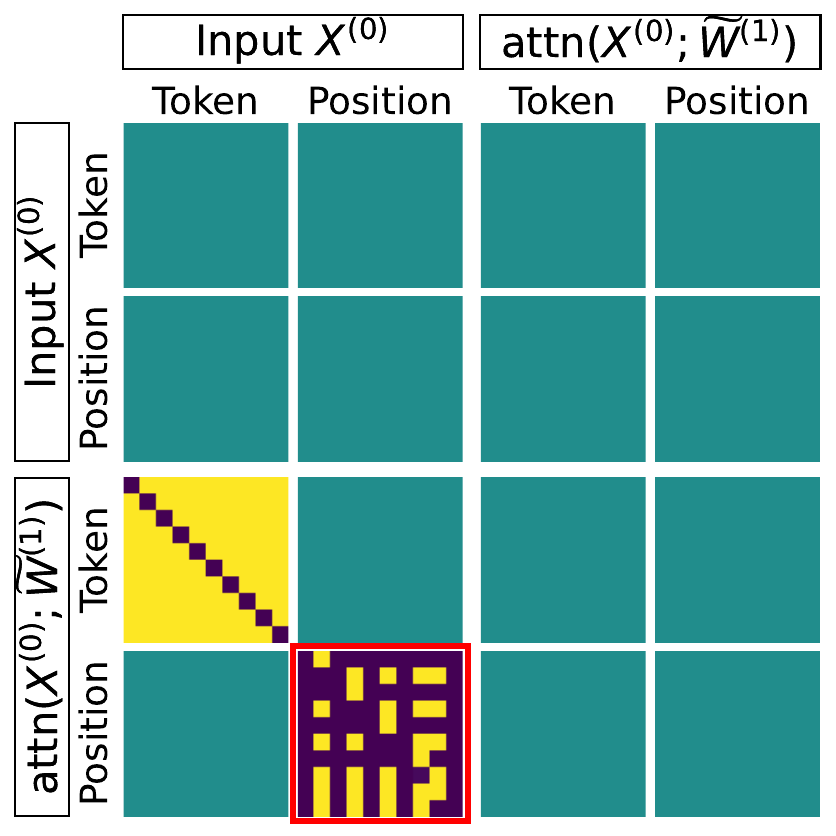}
    \end{subfigure}
    \hfill
    \begin{subfigure}[b]{0.26\textwidth}
        \centering
        \caption{NTP $\attn^{(1)}$}
        \includegraphics[width=\textwidth]{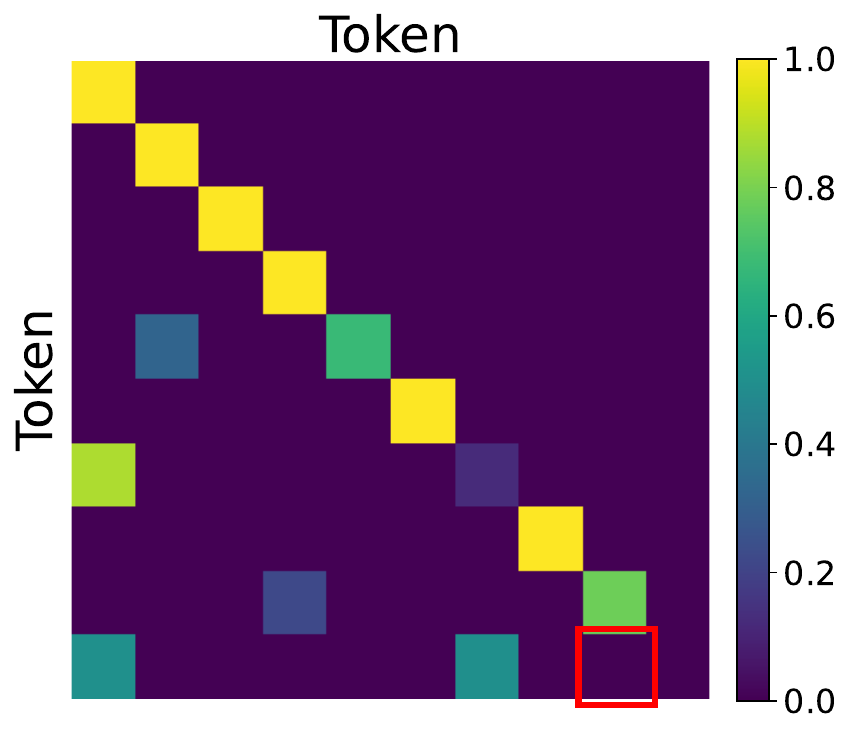}
    \end{subfigure}
    \hfill
    \begin{subfigure}[b]{0.26\textwidth}
        \centering
        \caption{NTP $\attn^{(2)}$}
        \includegraphics[width=\textwidth]{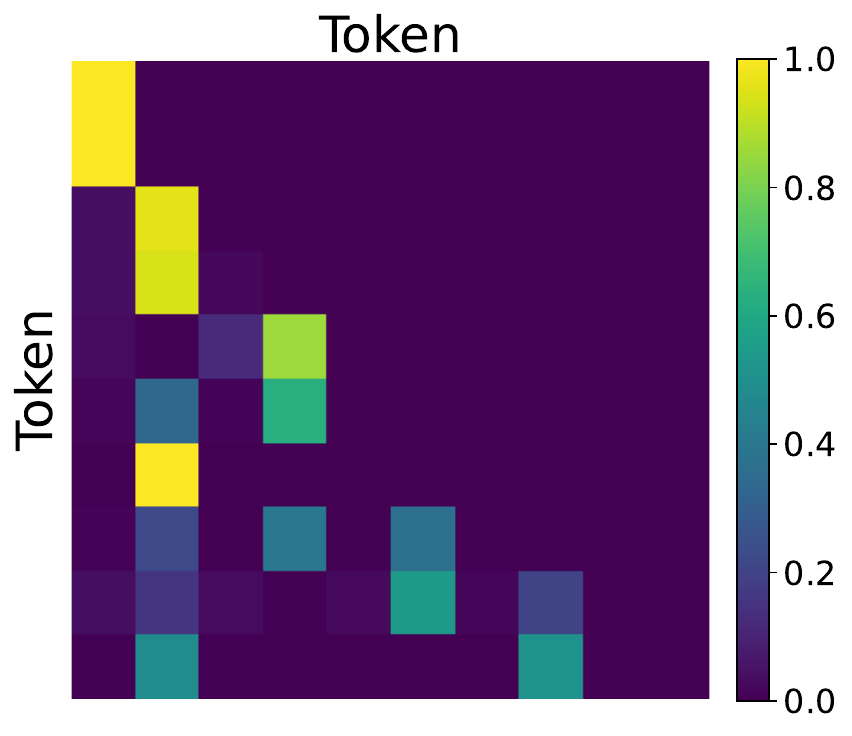}
    \end{subfigure}
    

    \begin{subfigure}[b]{0.23\textwidth}
        \centering
        \caption{MTP $\widetilde W^{(1)}$}
        \includegraphics[width=\textwidth]{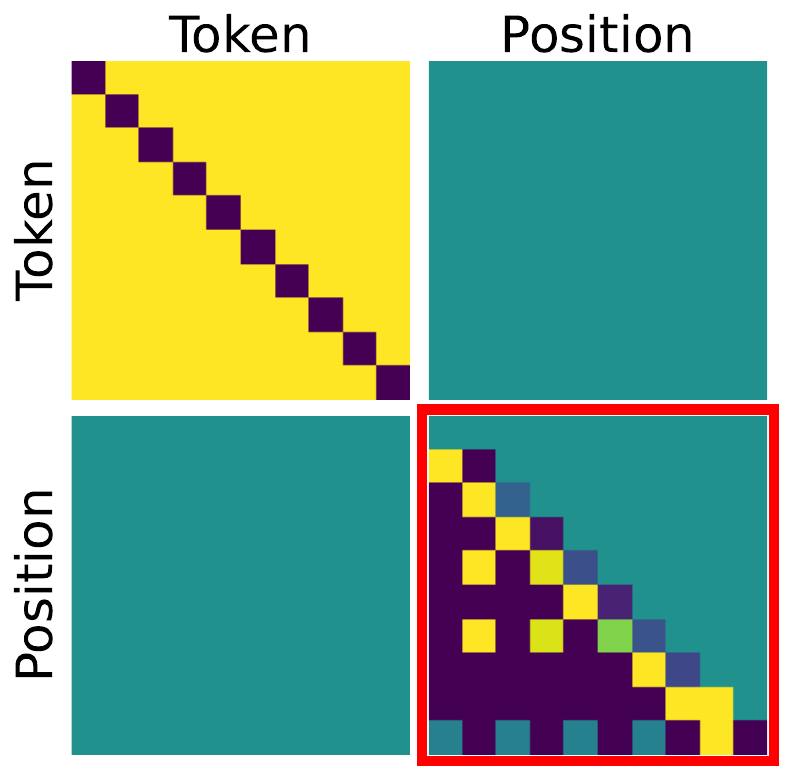}
    \end{subfigure}
    \hfill
    \begin{subfigure}[b]{0.225\textwidth}
        \centering
        \caption{MTP $\widetilde W^{(2)}$}
        \includegraphics[width=\textwidth]{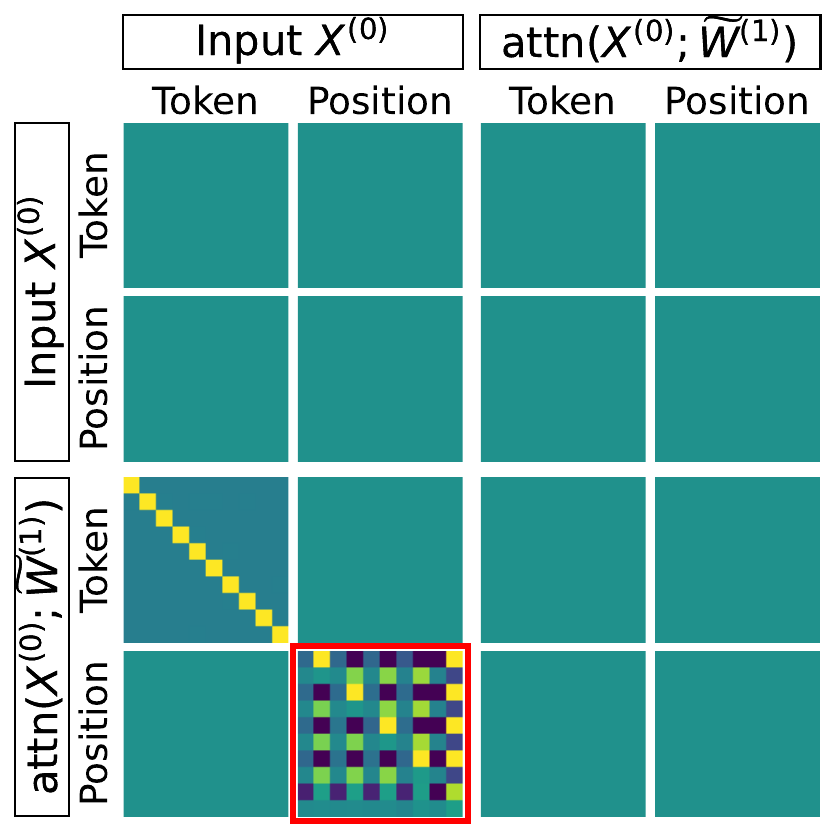}
    \end{subfigure}
    \hfill
    \begin{subfigure}[b]{0.26\textwidth}
        \centering
        \caption{MTP $\attn^{(1)}$}
        \includegraphics[width=\textwidth]{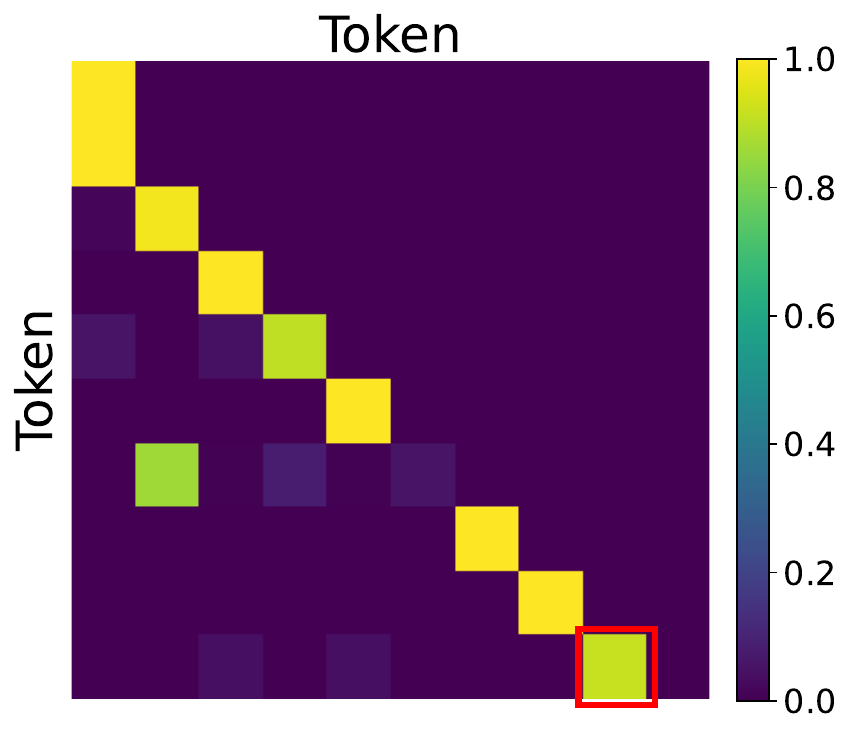}
    \end{subfigure}
    \hfill
    \begin{subfigure}[b]{0.26\textwidth}
        \centering
        \caption{MTP $\attn^{(2)}$}
        \includegraphics[width=\textwidth]{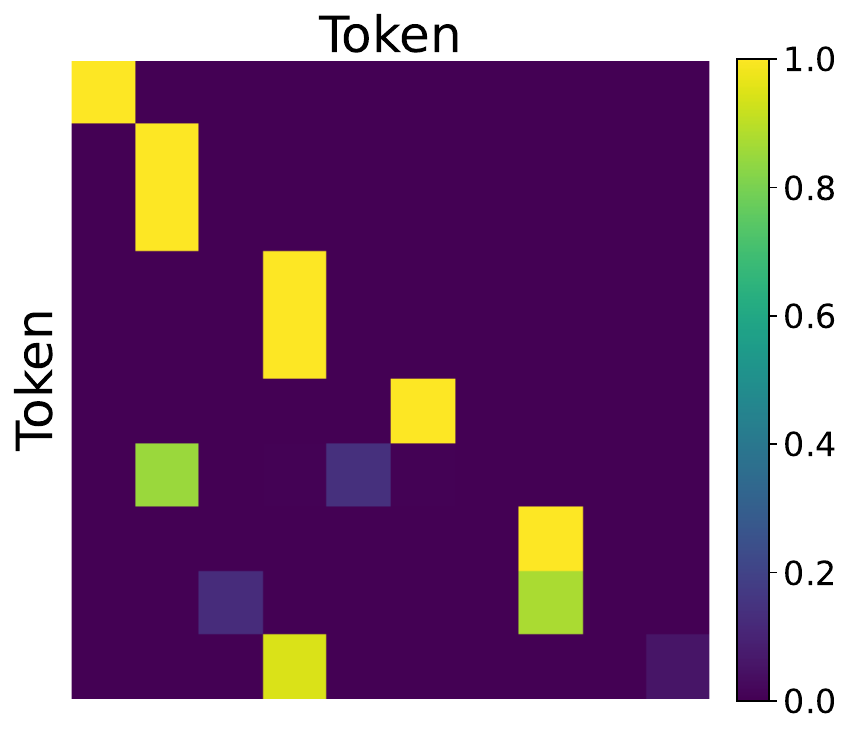}
    \end{subfigure}


    \begin{subfigure}[b]{0.95\textwidth} 
        \centering
        \caption{NTP/MTP Heatmap of $V_0$}
        \includegraphics[width=\textwidth]{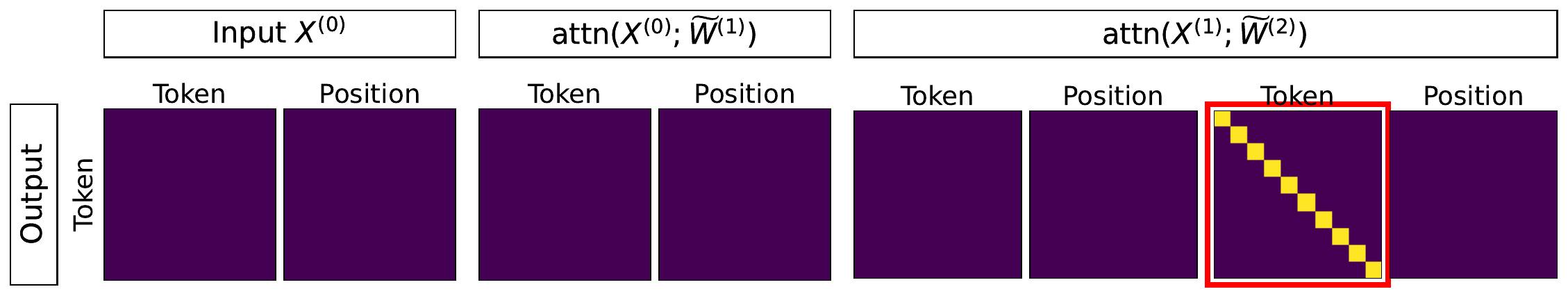}
    \end{subfigure}
    
    
    \begin{subfigure}[b]{0.95\textwidth}
        \centering
        \caption{MTP Heatmap of $V_1$}
        \includegraphics[width=\textwidth]{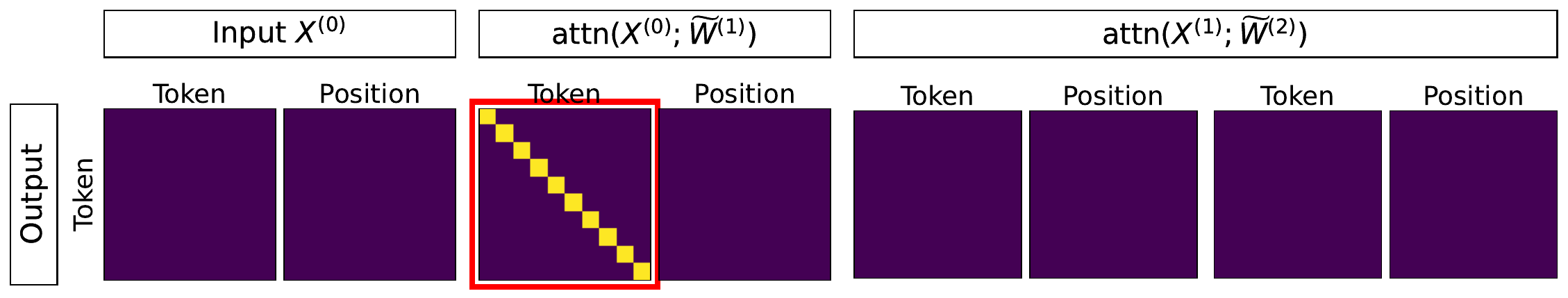}
    \end{subfigure}

    \caption{\textbf{Mechanism comparison between the minimal structures of NTP 
and 2-MTP on the star graph task.} (a--d)~Query-key weights 
($\widetilde W^{(1)}$, $\widetilde W^{(2)}$) and attention patterns ($\mathrm{attn}^{(1)}$, 
$\mathrm{attn}^{(2)}$) for NTP. (e--h)~Corresponding weights and attention 
patterns for 2-MTP. Attention heatmaps are computed on a sequence with input 
graph $[3, 7, 6, 0, 7, 2, 3, 6, 0, 3]$ and target $[3, 6, 0]$. When 
predicting the difficult intermediate node, MTP~(g) exhibits a clear 
predecessor-pointing pattern consistent with the reverse reasoning circuit 
(\cref{thm:stationary}): Layer~1 attends sharply to the end node, whereas NTP~(c) shows 
no such structure. (i)~The output projection $V_0$ for NTP (identical in 
MTP) activates exclusively on the content block of the Layer~2 output. 
(j)~The MTP-specific projection $V_1$ activates on the content block of the 
Layer~1 output. These block-sparse structures motivate the reduced model in 
Section~\ref{sec:prob_model}.}
    \label{fig:attention_heatmap}
\end{figure}

\textbf{The disentangled Transformer.} To enable rigorous analysis, we adopt a simplified but expressive abstraction of causal attention and the Transformer. Following \citet{zhang2024trained,nichani2024how,huang2025transformers}, we merge the key and query matrices into a single matrix $W$.

\textbf{Empirical motivation.}
We train two-layer disentangled transformers (one head per layer) comparing
NTP and 2-MTP on this task. To isolate a minimal structure, we zero out
cross-block entries in the query-key matrices and restrict each output head to
read from a single content block---choices directly motivated by the learned
weight structure (Figure~\ref{fig:attention_heatmap}). This minimal model preserves the core phenomenon:
NTP fails while MTP succeeds. Under MTP, the model exhibits a clear
\emph{reverse reasoning} pattern: Layer~1 attends sharply to $\uend$ when
predicting $v$ (Figure~\ref{fig:attention_heatmap}g), whereas NTP shows no such structure (Figure~\ref{fig:attention_heatmap}c). We refer to \cref{app:exp_sparse} for the detailed setup and verify in Appendix~\ref{sec:valid_full} that this reverse reasoning mechanism generalizes to \textit{a larger, standard 8-layer 8-head Transformer on 5-path 5-node star graphs} with qualitatively identical attention patterns.

\begin{definition}[Causal self-attention head]\label{def:attn}
For $W \in \R^{d \times d}$, define
$$\attn(X; W) := \softmax(\MASK(X W X^\top))\, X,$$
where $\MASK(M)_{i,j} = M_{i,j}$ for $i \geq j$ and $-\infty$ otherwise.
\end{definition}
 
\begin{definition}[Disentangled Transformer]\label{def:disentangled_TF}
Let $L$ be the depth, $d_1=N+T$, and
$d_\ell=2d_{\ell-1}$ for each $\ell\in[L]$.
Let $g_{1:T}=(g_1,\ldots,g_T)$ denote the sequence of node identities.
The input embeds each token as $x_t=(z_t;\,e_t)\in\R^{d_0}$,
where $z_t\in\R^N$ and $e_t\in\R^T$ are one-hot content and position
vectors, respectively. We then define $X^{(0)}=[x_1,\ldots,x_T]^\top$. For each $\ell\in[L]$, let $\widetilde{W}^{(\ell)}\in\R^{d_{\ell}\times d_{\ell}}$ denote the attention-weight matrix at layer $\ell$. The model iterates
\begin{equation}
    X^{(\ell)}
    =
    \left[
        X^{(\ell-1)},
        \;
        \attn\left(
            X^{(\ell-1)};
            \widetilde{W}^{(\ell)}
        \right)
    \right],
    \qquad \ell\in[L],
    \label{eq:full_iterate}
\end{equation}
and its final output is $f_\theta(g_{1:T})=X^{(L)}V$.
\end{definition}

Guided by the empirical block structures of the attention weights and output
heads in Figures~\ref{fig:attention_heatmap}a--b and
\ref{fig:attention_heatmap}i--j, respectively, we adopt the following
simplified two-layer model. We set
\begin{equation}
\widetilde{W}^{(1)} = W^{(1)},
  \qquad
  \widetilde{W}^{(2)}
  =
  \begin{pmatrix}
    0 & 0 \\[2pt]
    W^{(2)} & 0
  \end{pmatrix},
  \label{eq:first_W}
\end{equation}
  
and constrain
\begin{equation*}
  W^{(\ell)}
  =
  \begin{pmatrix}
    W^{(\ell)}_0 & 0 \\[2pt]
    0 & W^{(\ell)}_1
  \end{pmatrix},
  \qquad
  W^{(\ell)}_0 \in \R^{N\times N},
  \quad
  W^{(\ell)}_1 \in \R^{T\times T},
  \qquad \ell\in\{1,2\}.
\end{equation*}
We further fix the two independent heads
$V_0,V_1\in\R^{d_2\times N}$, which read from $X^{(2)}$, to the corresponding
single-block selector forms. These simplifications preserve the
NTP-fails/MTP-succeeds phenomenon, indicating that the underlying mechanism
is captured by the attention dynamics. In particular, the block-diagonal
parameterization yields a clean decomposition of the attention logits into
content and positional terms (\cref{lemma:forward}), enabling the gradient
analysis in Sections~\ref{sec:circuit}--\ref{sec:optimization}.

\begin{lemma}[Forward pass]\label{lemma:forward}
Let $Z = [z_1,\ldots,z_T]^\top$ with $z_t = e_{g_t}$ denoting the content-only submatrix of the full input. Define
\[
  A^{[\ell]} := Z\, W^{(\ell)}_0 Z^\top + W^{(\ell)}_1,
  \qquad S^{[0]} := I_T, \qquad
  S^{[\ell]} := \softmax \big(\MASK(S^{[\ell-1]} A^{[\ell]})\big).
\]
Then the output logits at the last position are
$f_1(Z) = S^{[2]}_{[-1,:]} S^{[1]} Z$ and $f_2(Z) = S^{[1]}_{[-1,:]} Z$.
\end{lemma}

\subsection{The Reverse Reasoning Circuit} \label{sec:circuit}

We now characterize the attention configurations where the MTP loss
achieves a stationary point. Rather than searching the full
parameter space, we identify conditions on the softmax attention distributions
$S^{[1]}$ and $S^{[2]}$ that simultaneously zero out all gradient terms.

\begin{theorem} \label{thm:stationary}
The model achieves a stationary point for both the shallow loss
$\mathcal{L}_2$ and the deep loss $\mathcal{L}_1$ if the attention
distributions satisfy the following two conditions:
\begin{enumerate}[leftmargin=1.5em]
  \item \textbf{Predecessor Pointing (Layer 1):}
    Layer~1 attends to the immediate predecessor at both the query
    position and the target context position:
    $S^{[1]}_{T,:} = e_{T-1}^\top$ and
    $(S^{[1]})_{t_{\mathrm{end}}^{\mathrm{ctx}}, :}
      = e_{t_v^{\mathrm{ctx}}}^\top$.
 
  \item \textbf{Content Matching (Layer 2, Last Row):}
    Layer~2 identifies the context position where $u_{\mathrm{end}}$
    appears as part of an edge:
    $S^{[2]}_{T,:} = e_{t_{\mathrm{end}}^{\mathrm{ctx}}}^\top$.
\end{enumerate}
\end{theorem}

 
 

Note that a stationary point here means vanishing gradient, not necessarily zero loss. In particular, the AR component $\mathcal{L}_{1b}$ incurs nonzero error at this configuration, but its gradient is exponentially small (see Corollary~\ref{cor:ar-collapse} in Appendix~\ref{appdix:circuit}). This circuit is highly minimal. It relies strictly on \emph{intra-edge backward
attention}: forward attention within edges ($v \to \uend$) is neither required
nor utilized. Moreover, only two rows of $S^{[1]}$ are constrained (the query
row and $\tctxend$); all other context rows are free, which is reflected in the
sub-diagonal pattern being sharp at these two rows but diffuse elsewhere
(Figure~\ref{fig:attention_heatmap}g). Having established the stationary manifold in terms of attention conditions,
we now exhibit an explicit weight configuration that realizes it.

\begin{corollary}[Constructive Existence of the Reverse Reasoning Circuit]\label{cor:construction}
Let $\gamma > 0$ be sufficiently large, and let
$L \in \mathbb{R}^{T \times T}$ be the strictly lower shift matrix
($L_{i,i-1} = 1$, zero elsewhere).
The following weight configuration satisfies both conditions of
Theorem~\ref{thm:stationary}:
\begin{itemize}[leftmargin=1.5em]
  \item \textbf{Layer~1 (Predecessor Shift):}
    $W_0^{(1)} = 0$, \quad $W_1^{(1)} = \gamma\, L$.
 
  \item \textbf{Layer~2 (Content Matching):}
    $W_0^{(2)} = \gamma\, I$,
    $(W_1^{(2)})_{T-1,\, T-1} = -\gamma$
    (all other entries of $W_1^{(2)}$ are zero).
\end{itemize}
Under this configuration, the loss and gradient norm satisfy
$\mathcal{L} = O(e^{-\gamma})$ and
$\|\nabla_\theta \mathcal{L}\| = O(e^{-\gamma})$.
\end{corollary}

Corollary~\ref{cor:construction} reveals a fundamental tension. The
disentangled transformer possesses the representational capacity to solve the
task perfectly---the reverse reasoning circuit lies within its hypothesis class.
However, as we show in Section~\ref{sec:optimization}, pure NTP fails to discover this
configuration from zero initialization. The failure is not one of
expressivity but of \emph{optimization}: NTP's gradient dynamics actively steer
the model away from the predecessor pointer (Theorem~\ref{thm:ntp-misdirected}), while MTP provides a
clean two-phase gradient path that converges to this exact solution
(Theorem~\ref{thm:cascaded}). Comparing Figures~\ref{fig:attention_heatmap}c--d with~\ref{fig:attention_heatmap}g--h on the running example
confirms this: MTP recovers the predecessor-pointing pattern in Layer~1 and
sharp content matching in Layer~2, while NTP develops neither despite identical
model capacity.

\subsection{Why MTP Finds It and NTP Cannot} \label{sec:optimization}

The reverse reasoning circuit exists within the model's hypothesis class
(Corollary~\ref{cor:construction}), yet whether gradient-based optimization discovers it depends
critically on the training objective. The key mechanism is a
\emph{gradient decoupling} property: the shallow loss $\mathcal{L}_2$ routes
gradients exclusively through Layer~1, entirely
bypassing the uninitialized Layer~2. This enables a clean two-phase learning
process under MTP, which we now formalize.

\begin{theorem}[Cascaded Convergence to the Reverse Reasoning Circuit]
\label{thm:cascaded}
Phase I: Consider the MTP loss $\mathcal{L}_2$ under gradient flow.
Assume the content weight is fixed at $W_0^{(1)} = 0$, that $W_1^{(1)}$
has RoPE (Toeplitz) structure $W_{i,j} = w(i-j)$, where $w: \mathbb{Z} \rightarrow \mathbb{R}$ is a
scalar function of the offset. Consequently, $S^{[1]}_{t,:} \to e_{t-1}^\top$ for all $t \geq 2$,
i.e., Layer~1 converges to the universal predecessor pointer
$S^{[1]} \to {L}$. Phase II: With $W^{(1)}_1$ frozen at $\gamma L$. Then optimize $\mathcal{L}_{1a}$ over $(W^{(2)}_0, W^{(2)}_1 )$ from the zero initialization through gradient flow. Then, $W ^{(2)}_0 \to \gamma' I+B$, where $\gamma'\to+\infty$ and $B$ is off-diagonal
        with finite negative entries $B_{d,u}<0$ for $d\neq u$; $(W^{(2)}_1)_{T-1,T-1}\to-\infty$; all other entries of $W^{(2)}_1$ remain at~$0$; and $S^{[2]}_{T,:} \to e_{t_{\mathrm{end}}^{\mathrm{ctx}}}^\top$.
\end{theorem}

The two-phase structure is a direct consequence of MTP's gradient
decoupling:
Phase~I uses $\mathcal{L}_2$, which bypasses Layer~2 entirely, to solve
a \emph{pure positional learning problem} for Layer~1.
Phase~II then benefits from a fully converged Layer~1, reducing the
combinatorial task (jointly learn position + content matching) to a clean
contrastive problem over content alone.
Under pure NTP, this cascade cannot form---Layer~1 receives gradient only
through the uninitialized Layer~2, leading to the misdirected gradient of
Theorem~\ref{thm:ntp-misdirected}.

\begin{theorem}[Misdirected Gradient under Pure NTP]
\label{thm:ntp-misdirected}
Consider the same setup as Phase~I of Theorem~\ref{thm:cascaded}: two-layer Disentangled Transformer, star graph task
with $T=10$, $\Wzoc = 0$ fixed, $\Wzop$ with Toeplitz structure $w = 0$
at initialization, $\Wonc = \Wonp = 0$. Optimize the pure NTP deep loss. Then gradient descent \emph{decreases} $w(1)$ (predecessor) and
\emph{increases} $w(k)$ for $k\geq 2$ (context): pure NTP actively repels
the predecessor pointer and diffuses attention across context, structurally
preventing formation of the reverse reasoning circuit.
\end{theorem}

The intuition is straightforward.
Under pure NTP with an uninitialized Layer~2, Layer~2 contributes a
uniform $\frac{1}{T}$ weight to every position.
The deep head's prediction is therefore an average over \emph{all} rows of
$S^{[1]}$, not just the query row.
The target $y^{(1)} = v$ appears somewhere in the context but
\emph{never} at the predecessor position $T-1$ (which encodes
$u_{\mathrm{end}} \neq v$).
So the NTP gradient tells Layer~1: ``attend to the context where the
target might be''---exactly the opposite of the predecessor pointer
that the 2-hop circuit requires.
The shallow MTP loss avoids this entirely because it targets
$y^{(2)} = u_{\mathrm{end}}$, which \emph{always} resides at $T-1$.

\section{Conclusion}\label{sec:conclusion}

In summary, we show that the reasoning capabilities enabled by MTP arise from its impact on optimization dynamics.
While prior work has established MTP’s empirical advantages, our study strengthens these findings and, more importantly, formalizes the underlying mechanisms. 
By avoiding the entangled training signals of standard NTP, MTP leverages \emph{gradient decoupling} to foster global planning abilities. 
We show that this training paradigm biases the network toward discovering interpretable algorithms, particularly a reverse reasoning process on graph path-finding tasks.
Beyond validating these mechanisms across algorithmic and logic-based benchmarks, our work bridges a critical gap in understanding how training objectives shape model reasoning, providing a foundation for designing training paradigms for advanced reasoning.
\paragraph{Limitations and future work.} Our theory focuses on a two-layer disentangled Transformer with a star graph task and lookahead $k=2$. Extending to deeper architectures, general graph topologies, finite-time convergence rates, and sequential MTP variants~\citep{liu2024deepseek} remains open.



\section*{Acknowledgments}

BM is supported in part by the NSF CAREER Award 2146492, NSF-Simons AI Institute for Cosmic Origins (CosmicAI), and NSF AI Institute for Foundations of Machine Learning (IFML). WH is supported by JSPS KAKENHI (24K20848) and JST
BOOST (JPMJBY24G6).

\bibliography{colm2026_conference}

@inproceedings{qi2020prophetnet,
  title={ProphetNet: Predicting future n-gram for sequence-to-Sequence Pre-training},
  author={Qi, Weizhen and Yan, Yu and Gong, Yeyun and Liu, Dayiheng and Duan, Nan and Chen, Jiusheng and Zhang, Ruofei and Zhou, Ming},
  booktitle={Findings of the Association for Computational Linguistics: EMNLP 2020},
  pages={2401--2410},
  year={2020}
}

@inproceedings{bachmann2024pitfalls,
  title={The Pitfalls of Next-Token Prediction},
  author={Bachmann, Gregor and Nagarajan, Vaishnavh},
  booktitle={International Conference on Machine Learning},
  pages={2296--2318},
  year={2024},
  organization={PMLR}
}

@inproceedings{gloeckle2024better,
  title={Better \& Faster Large Language Models via Multi-token Prediction},
  author={Gloeckle, Fabian and Idrissi, Badr Youbi and Roziere, Baptiste and Lopez-Paz, David and Synnaeve, Gabriel},
  booktitle={International Conference on Machine Learning},
  pages={15706--15734},
  year={2024},
  organization={PMLR}
}

@article{liu2024deepseek,
  title={Deepseek-v3 technical report},
  author={Liu, Aixin and Feng, Bei and Xue, Bing and Wang, Bingxuan and Wu, Bochao and Lu, Chengda and Zhao, Chenggang and Deng, Chengqi and Zhang, Chenyu and Ruan, Chong and others},
  journal={arXiv preprint arXiv:2412.19437},
  year={2024}
}

@misc{qwen3_next,
  author       = {{Qwen Team}},
  title        = {Qwen3-Next: Towards Ultimate Training \& Inference Efficiency},
  howpublished = {\url{https://qwen.ai/blog?from=research.latest-advancements-list&id=4074cca80393150c248e508aa62983f9cb7d27cd&}},
  year         = {2025},
  month        = {Sep},
  note         = {Accessed: 2026-03-11}
}

@article{xiao2026mimo,
  title={Mimo-v2-flash technical report},
  author={Xiao, Bangjun and Xia, Bingquan and Yang, Bo and Gao, Bofei and Shen, Bowen and Zhang, Chen and He, Chenhong and Lou, Chiheng and Luo, Fuli and Wang, Gang and others},
  journal={arXiv preprint arXiv:2601.02780},
  year={2026}
}

@article{blakeman2025nvidia,
  title={NVIDIA Nemotron 3: Efficient and Open Intelligence},
  author={Blakeman, Aaron and Grattafiori, Aaron and Basant, Aarti and Gupta, Abhibha and Khattar, Abhinav and Renduchintala, Adi and Vavre, Aditya and Shukla, Akanksha and Bercovich, Akhiad and Ficek, Aleksander and others},
  journal={arXiv preprint arXiv:2512.20856},
  year={2025}
}

@inproceedings{wang2025vocalnet,
  title={VocalNet: Speech LLMs with Multi-Token Prediction for Faster and High-Quality Generation},
  author={Wang, Yuhao and Liu, Heyang and Cheng, Ziyang and Wu, Ronghua and Gu, Qunshan and Wang, Yanfeng and Wang, Yu},
  booktitle={Proceedings of the 2025 Conference on Empirical Methods in Natural Language Processing},
  pages={19595--19612},
  year={2025}
}

@inproceedings{
liu2025lmtp,
title={L-{MTP}: Leap Multi-Token Prediction Beyond Adjacent Context for Large Language Models},
author={Xiaohao Liu and Xiaobo Xia and Weixiang Zhao and Manyi Zhang and Xianzhi Yu and Xiu Su and Shuo Yang and See-Kiong Ng and Tat-Seng Chua},
booktitle={The Thirty-ninth Annual Conference on Neural Information Processing Systems},
year={2025},
url={https://openreview.net/forum?id=0VDmWjW456}
}

@inproceedings{
gerontopoulos2025multitoken,
title={Multi-Token Prediction Needs Registers},
author={Anastasios Gerontopoulos and Spyros Gidaris and Nikos Komodakis},
booktitle={The Thirty-ninth Annual Conference on Neural Information Processing Systems},
year={2025},
url={https://openreview.net/forum?id=WDdBhcwzGe}
}

@article{samragh2025your,
  title={Your llm knows the future: Uncovering its multi-token prediction potential},
  author={Samragh, Mohammad and Kundu, Arnav and Harrison, David and Nishu, Kumari and Naik, Devang and Cho, Minsik and Farajtabar, Mehrdad},
  journal={arXiv preprint arXiv:2507.11851},
  year={2025}
}

@inproceedings{
will2026parallel,
title={Parallel Token Generation for  Language Models},
author={Justus Will and Felix Draxler and Farrin Marouf Sofian and Theofanis Karaletsos and Sameer Singh and Stephan Mandt},
booktitle={The Fourteenth International Conference on Learning Representations},
year={2026},
url={https://openreview.net/forum?id=AGJomYSrUG}
}

@article{ahn2025efficient,
  title={Efficient joint prediction of multiple future tokens},
  author={Ahn, Kwangjun and Lamb, Alex and Langford, John},
  journal={arXiv preprint arXiv:2503.21801},
  year={2025}
}

@inproceedings{
mahajan2026beyond,
title={Beyond Multi-Token Prediction: Pretraining {LLM}s with Future Summaries},
author={Divyat Mahajan and Sachin Goyal and Badr Youbi Idrissi and Mohammad Pezeshki and Ioannis Mitliagkas and David Lopez-Paz and Kartik Ahuja},
booktitle={The Fourteenth International Conference on Learning Representations},
year={2026},
url={https://openreview.net/forum?id=aeYIFVn4vb}
}

@article{zhong2025understanding,
  title={Understanding and Enhancing the Planning Capability of Language Models via Multi-Token Prediction},
  author={Zhong, Qimin and Liao, Hao and Wang, Siwei and Zhou, Mingyang and Wu, Xiaoqun and Mao, Rui and Chen, Wei},
  journal={arXiv preprint arXiv:2509.23186},
  year={2025}
}

@article{kaplan2020scaling,
  title={Scaling laws for neural language models},
  author={Kaplan, Jared and McCandlish, Sam and Henighan, Tom and Brown, Tom B. and Chess, Benjamin and Child, Rewon and Gray, Scott and Radford, Alec and Wu, Jeffrey and Amodei, Dario},
  journal={arXiv preprint arXiv:2001.08361},
  year={2020}
}

@article{hoffmann2022training,
  title={Training compute-optimal large language models},
  author={Hoffmann, Jordan and Borgeaud, Sebastian and Mensch, Arthur and Buchatskaya, Elena and Cai, Trevor and Rutherford, Eliza and Casas, Diego de Las and Hendricks, Lisa Anne and Welbl, Johannes and Clark, Aidan and others},
  journal={arXiv preprint arXiv:2203.15556},
  year={2022}
}

@inproceedings{
gandhi2024stream,
title={Stream of Search (SoS): Learning to Search in Language},
author={Kanishk Gandhi and Denise H J Lee and Gabriel Grand and Muxin Liu and Winson Cheng and Archit Sharma and Noah Goodman},
booktitle={First Conference on Language Modeling},
year={2024},
url={https://openreview.net/forum?id=2cop2jmQVL}
}

@inproceedings{yao2023treeofthoughts,
 author = {Yao, Shunyu and Yu, Dian and Zhao, Jeffrey and Shafran, Izhak and Griffiths, Tom and Cao, Yuan and Narasimhan, Karthik},
 booktitle = {Advances in Neural Information Processing Systems},
 editor = {A. Oh and T. Naumann and A. Globerson and K. Saenko and M. Hardt and S. Levine},
 pages = {11809--11822},
 publisher = {Curran Associates, Inc.},
 title = {Tree of Thoughts: Deliberate Problem Solving with Large Language Models},
 url = {https://proceedings.neurips.cc/paper_files/paper/2023/file/271db9922b8d1f4dd7aaef84ed5ac703-Paper-Conference.pdf},
 volume = {36},
 year = {2023}
}

@inproceedings{cook1971complexity,
author = {Cook, Stephen A.},
title = {The complexity of theorem-proving procedures},
year = {1971},
isbn = {9781450374644},
publisher = {Association for Computing Machinery},
address = {New York, NY, USA},
url = {https://doi.org/10.1145/800157.805047},
doi = {10.1145/800157.805047},
booktitle = {Proceedings of the Third Annual ACM Symposium on Theory of Computing},
pages = {151–158},
numpages = {8},
location = {Shaker Heights, Ohio, USA},
series = {STOC '71}
}

@article{
ke2025a,
title={A Survey of Frontiers in {LLM} Reasoning: Inference Scaling, Learning to Reason, and Agentic Systems},
author={Zixuan Ke and Fangkai Jiao and Yifei Ming and Xuan-Phi Nguyen and Austin Xu and Do Xuan Long and Minzhi Li and Chengwei Qin and PeiFeng Wang and silvio savarese and Caiming Xiong and Shafiq Joty},
journal={Transactions on Machine Learning Research},
issn={2835-8856},
year={2025},
url={https://openreview.net/forum?id=SlsZZ25InC},
note={Survey Certification}
}

@inproceedings{
kim2025transformers,
title={Transformers Provably Solve Parity Efficiently with Chain of Thought},
author={Juno Kim and Taiji Suzuki},
booktitle={The Thirteenth International Conference on Learning Representations},
year={2025},
url={https://openreview.net/forum?id=n2NidsYDop}
}

@inproceedings{feng2023cot,
 author = {Feng, Guhao and Zhang, Bohang and Gu, Yuntian and Ye, Haotian and He, Di and Wang, Liwei},
 booktitle = {Advances in Neural Information Processing Systems},
 editor = {A. Oh and T. Naumann and A. Globerson and K. Saenko and M. Hardt and S. Levine},
 pages = {70757--70798},
 publisher = {Curran Associates, Inc.},
 title = {Towards Revealing the Mystery behind Chain of Thought: A Theoretical Perspective},
 url = {https://proceedings.neurips.cc/paper_files/paper/2023/file/dfc310e81992d2e4cedc09ac47eff13e-Paper-Conference.pdf},
 volume = {36},
 year = {2023}
}

@inproceedings{
huang2025transformers,
title={Transformers Learn to Implement Multi-step Gradient Descent with Chain of Thought},
author={Jianhao Huang and Zixuan Wang and Jason D. Lee},
booktitle={The Thirteenth International Conference on Learning Representations},
year={2025},
url={https://openreview.net/forum?id=r3DF5sOo5B}
}

@inproceedings{
saunshi2025reasoning,
title={Reasoning with Latent Thoughts: On the Power of Looped Transformers},
author={Nikunj Saunshi and Nishanth Dikkala and Zhiyuan Li and Sanjiv Kumar and Sashank J. Reddi},
booktitle={The Thirteenth International Conference on Learning Representations},
year={2025},
url={https://openreview.net/forum?id=din0lGfZFd}
}

@inproceedings{
gatmiry2024can,
title={Can Looped Transformers Learn to Implement Multi-step Gradient Descent for In-context Learning?},
author={Khashayar Gatmiry and Nikunj Saunshi and Sashank J. Reddi and Stefanie Jegelka and Sanjiv Kumar},
booktitle={Forty-first International Conference on Machine Learning},
year={2024},
url={https://openreview.net/forum?id=o8AaRKbP9K}
}

@article{zhang2024trained,
  title={Trained transformers learn linear models in-context},
  author={Zhang, Ruiqi and Frei, Spencer and Bartlett, Peter L},
  journal={Journal of Machine Learning Research},
  volume={25},
  number={49},
  pages={1--55},
  year={2024}
}

@inproceedings{
wen2025from,
title={From Sparse Dependence to Sparse Attention: Unveiling How Chain-of-Thought Enhances Transformer Sample Efficiency},
author={Kaiyue Wen and Huaqing Zhang and Hongzhou Lin and Jingzhao Zhang},
booktitle={The Thirteenth International Conference on Learning Representations},
year={2025},
url={https://openreview.net/forum?id=AmEgWDhmTr}
}

@inproceedings{
wei2026how,
title={How Transformers Learn Causal Structures In-Context: Explainable Mechanism Meets Theoretical Guarantee},
author={Jianzhe Wei and Siyu Chen and Jianliang He and Zhuoran Yang},
booktitle={The Fourteenth International Conference on Learning Representations},
year={2026},
url={https://openreview.net/forum?id=bpF8zgSt41}
}

@inproceedings{
nichani2024how,
title={How Transformers Learn Causal Structure with Gradient Descent},
author={Eshaan Nichani and Alex Damian and Jason D. Lee},
booktitle={Forty-first International Conference on Machine Learning},
year={2024},
url={https://openreview.net/forum?id=jNM4imlHZv}
}

@article{chen2024unveiling,
  title={Unveiling induction heads: Provable training dynamics and feature learning in transformers},
  author={Chen, Siyu and Sheen, Heejune and Wang, Tianhao and Yang, Zhuoran},
  journal={Advances in Neural Information Processing Systems},
  volume={37},
  pages={66479--66567},
  year={2024}
}

@InProceedings{pmlr-v247-siyu24a,
  title = 	 {Training Dynamics of Multi-Head Softmax Attention for In-Context Learning: Emergence, Convergence, and Optimality (extended abstract)},
  author =       {Chen, Siyu and Sheen, Heejune and Wang, Tianhao and Yang, Zhuoran},
  booktitle = 	 {Proceedings of Thirty Seventh Conference on Learning Theory},
  pages = 	 {4573--4573},
  year = 	 {2024},
  editor = 	 {Agrawal, Shipra and Roth, Aaron},
  volume = 	 {247},
  series = 	 {Proceedings of Machine Learning Research},
  month = 	 {30 Jun--03 Jul},
  publisher =    {PMLR},
  url = 	 {https://proceedings.mlr.press/v247/siyu24a.html}
}

@article{friedman2023learning,
  title={Learning transformer programs},
  author={Friedman, Dan and Wettig, Alexander and Chen, Danqi},
  journal={Advances in Neural Information Processing Systems},
  volume={36},
  pages={49044--49067},
  year={2023}
}

@inproceedings{
guo2025activedormant,
title={Active-Dormant Attention Heads: Mechanistically Demystifying Extreme-Token Phenomena in {LLM}s},
author={Tianyu Guo and Druv Pai and Yu Bai and Jiantao Jiao and Michael I. Jordan and Song Mei},
booktitle={The Second Conference on Parsimony and Learning (Recent Spotlight Track)},
year={2025},
url={https://openreview.net/forum?id=Zx6WUbE9J7}
}

@inproceedings{
loshchilov2018decoupled,
title={Decoupled Weight Decay Regularization},
author={Ilya Loshchilov and Frank Hutter},
booktitle={International Conference on Learning Representations},
year={2019},
url={https://openreview.net/forum?id=Bkg6RiCqY7},
}

@inproceedings{
kim2025metastable,
title={Metastable Dynamics of Chain-of-Thought Reasoning: Provable Benefits of Search, {RL} and Distillation},
author={Juno Kim and Denny Wu and Jason D. Lee and Taiji Suzuki},
booktitle={Forty-second International Conference on Machine Learning},
year={2025},
url={https://openreview.net/forum?id=2HJcVtuovs}
}

@inproceedings{
ye2025beyond,
title={Beyond Autoregression: Discrete Diffusion for Complex Reasoning and Planning},
author={Jiacheng Ye and Jiahui Gao and Shansan Gong and Lin Zheng and Xin Jiang and Zhenguo Li and Lingpeng Kong},
booktitle={The Thirteenth International Conference on Learning Representations},
year={2025},
url={https://openreview.net/forum?id=NRYgUzSPZz}
}

@article{jiang2024unveil,
  title={Unveil benign overfitting for transformer in vision: Training dynamics, convergence, and generalization},
  author={Jiang, Jiarui and Huang, Wei and Zhang, Miao and Suzuki, Taiji and Nie, Liqiang},
  journal={Advances in Neural Information Processing Systems},
  volume={37},
  pages={135464--135625},
  year={2024}
}

@article{magen2024benign,
  title={Benign overfitting in single-head attention},
  author={Magen, Roey and Shang, Shuning and Xu, Zhiwei and Frei, Spencer and Hu, Wei and Vardi, Gal},
  journal={arXiv preprint arXiv:2410.07746},
  year={2024}
}

@inproceedings{
wang2026the,
title={The Power of Power Law: Asymmetry Enables Compositional Reasoning},
author={Zixuan Wang and Xingyu Dang and Jason D. Lee and Kaifeng Lyu},
booktitle={Forty-third International Conference on Machine Learning},
year={2026},
url={https://openreview.net/forum?id=K83orsg2X9}
}

@InProceedings{pmlr-v291-wang25a,
  title = 	 {Learning Compositional Functions with Transformers from Easy-to-Hard Data},
  author =       {Wang, Zixuan and Nichani, Eshaan and Bietti, Alberto and Damian, Alex and Hsu, Daniel and Lee, Jason D and Wu, Denny},
  booktitle = 	 {Proceedings of Thirty Eighth Conference on Learning Theory},
  pages = 	 {5632--5711},
  year = 	 {2025},
  editor = 	 {Haghtalab, Nika and Moitra, Ankur},
  volume = 	 {291},
  series = 	 {Proceedings of Machine Learning Research},
  month = 	 {30 Jun--04 Jul},
  publisher =    {PMLR},
  url = 	 {https://proceedings.mlr.press/v291/wang25a.html}
}
\bibliographystyle{colm2026_conference}

\newpage
\tableofcontents

\appendix

\section{Additional Related Work of MTP}
\label{app:related work}

\textbf{Foundations and Large-scale Integration.} The standard NTP paradigm has driven the recent success of LLMs, yet it inherently suffers from limitations regarding planning. Early efforts to mitigate these pitfalls include ProphetNet \citep{qi2020prophetnet}, which proposed predicting future n-grams for sequence-to-sequence pre-training. Building on this, subsequent research identified the planning deficits of NTP and introduced teacherless, MTP style objectives \citep{bachmann2024pitfalls}. This trajectory culminated in  \cite{gloeckle2024better}, which formalized modern MTP frameworks and demonstrated that predicting multiple future tokens simultaneously can yield models that are both better at reasoning and significantly faster during inference.

Following its initial success, MTP has seen widespread adoption in state-of-the-art industry models to enhance both generation speed and performance. Technical reports for leading models such as DeepSeek-V3 \citep{liu2024deepseek}, Qwen3-Next \citep{qwen3_next}, Nemotron 3 \citep{blakeman2025nvidia} and MiMo-V2-Flash \citep{xiao2026mimo} highlight MTP as a core component of their training pipelines. Furthermore, the applicability of MTP has recently extended beyond text, showing significant promise in multi-modal architectures \citep{wang2025vocalnet}.

\textbf{Post-training Adaptation.} While industrial models integrate MTP during the compute-intensive pre-training phase, an emerging line of research focuses on post-training adaptation. Specifically, several recent works \citep{liu2025lmtp, gerontopoulos2025multitoken, samragh2025your} utilize fine-tuning strategies to efficiently transform standard NTP models into MTP models. 

\textbf{Future Prediction.} Beyond standard MTP, a parallel line of research re-examines the fundamental modeling objective by shifting from point-wise prediction to holistic future modeling. At the sequence level, \citet{mahajan2026beyond} bypass discrete tokens entirely by pre-training models to predict high-level summaries. For finer-grained dependencies, Joint Token Prediction (JTP) \citep{ahn2025efficient} and Parallel Token Prediction (PTP) \citep{will2026parallel} move beyond independent token generation; while JTP enforces future-state compression via a representation bottleneck, PTP explicitly captures simultaneous token dependencies through latent variable modeling.

\begin{figure}[tb!]
    \centering

    \begin{subfigure}[b]{0.4\textwidth}
        \centering
        \includegraphics[width=\textwidth]{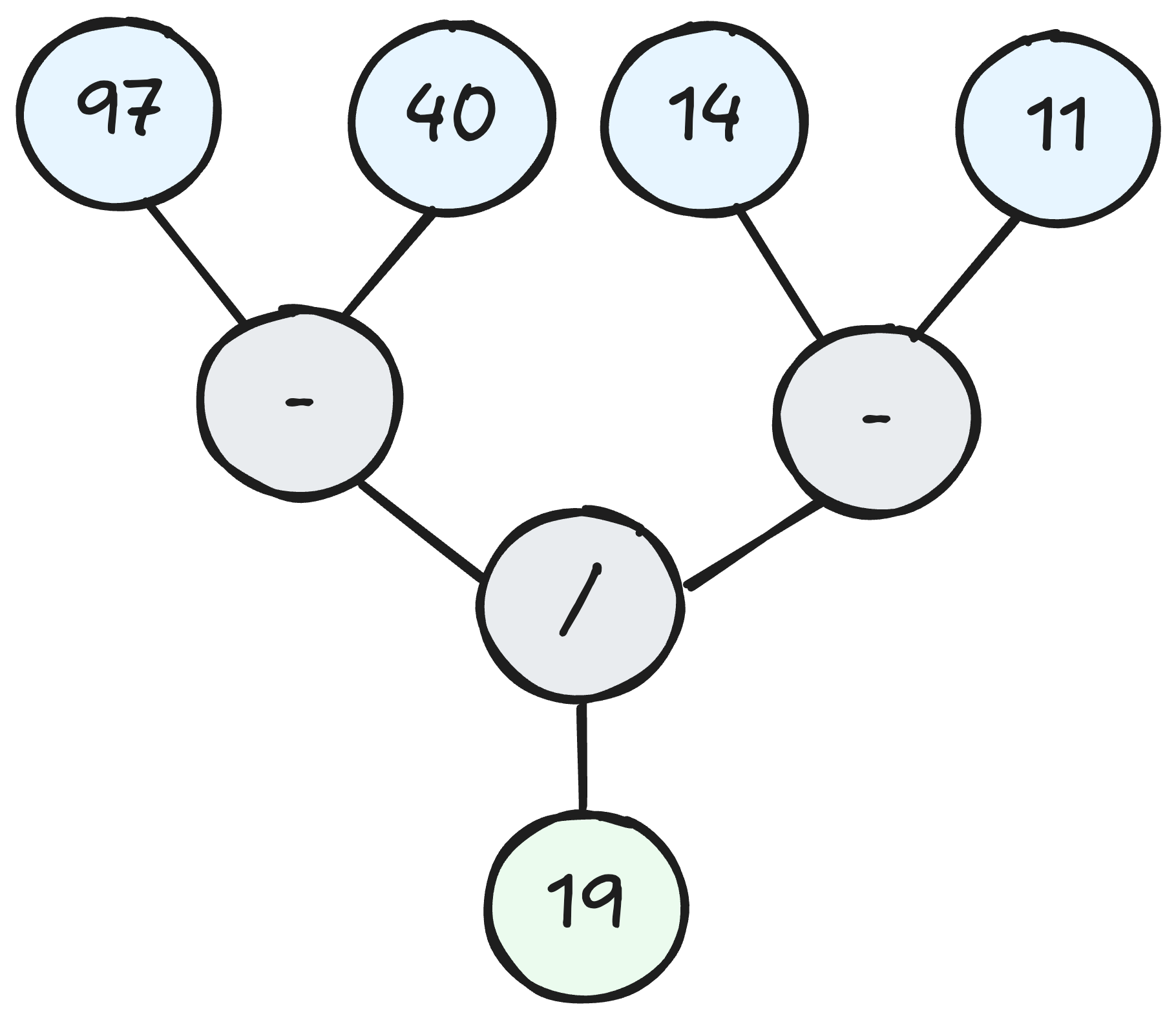}
        \caption{Countdown}
        \label{fig:countdown}
    \end{subfigure}
    \hfill 
    \begin{subfigure}[b]{0.4\textwidth}
        \centering
        \includegraphics[width=\textwidth]{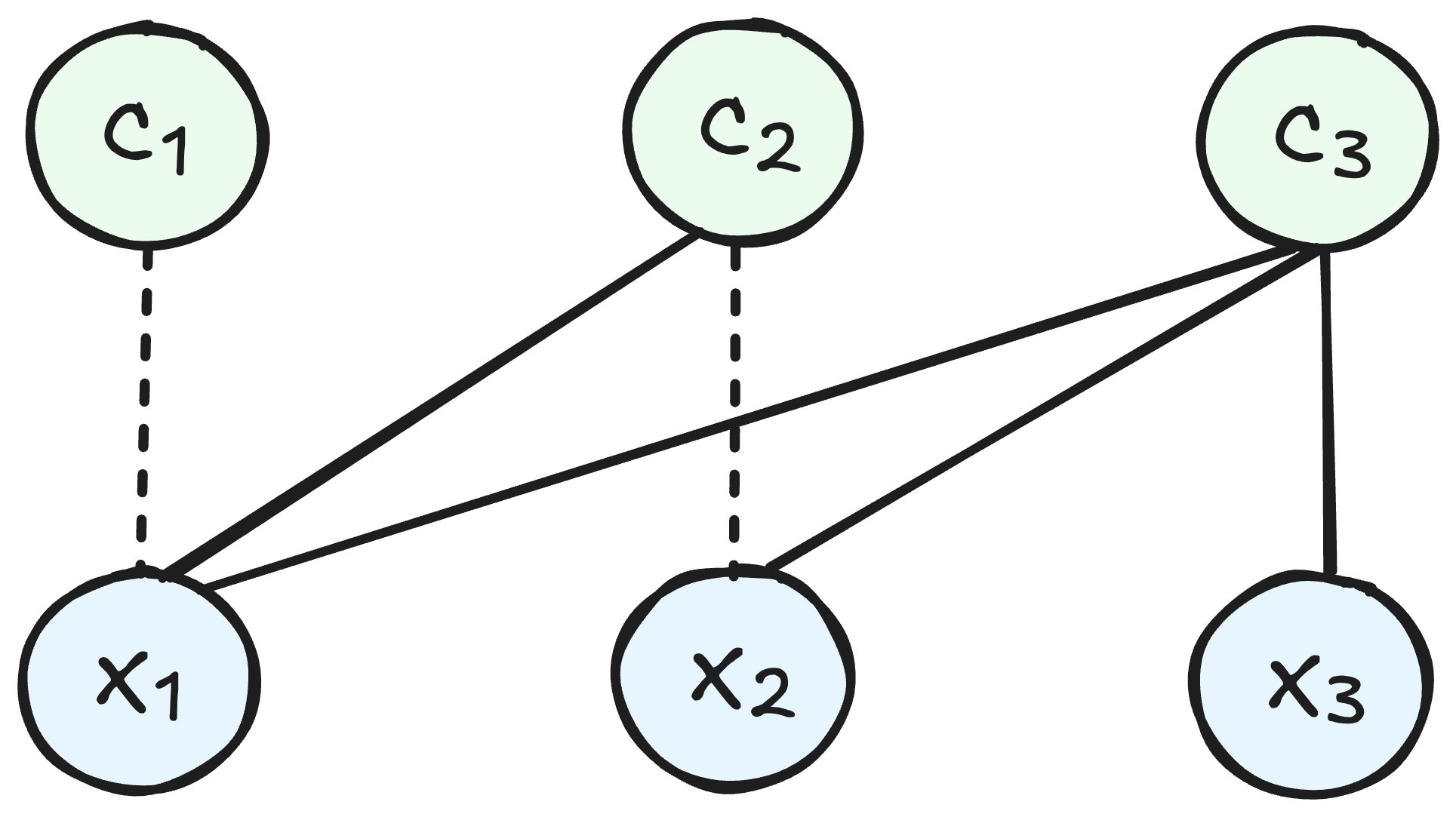}
        
        \caption{SAT}
        \label{fig:SAT}
    \end{subfigure}
    \caption{\textbf{Structural representations of reasoning tasks.} (a) Countdown: An expression tree showing the hierarchical computation to reach the target 19 using the numbers $\{11,14,40,97\}$. The internal nodes represent arithmetic operators, and the root node represents the final result. (b) SAT: A bipartite factor graph representing a Boolean Satisfiability problem $\qty(\neg x_1)\land\qty(x_1\lor\neg x_2)\land\qty(x_1\lor x_2\lor x_3)$. Nodes $\{x_1, x_2, x_3\}$ denote variables, and nodes $\{c_1, c_2, c_3\}$ denote logical clauses. Solid lines represent positive literals, while dashed lines indicate negated literals.}
    \label{fig:illustration_of_countdown_SAT}
    
\end{figure}

\section{Illustrations of Planning Tasks}
\label{app:illustration}

In this section, we provide visual examples and formal definitions of the Countdown and SAT tasks used in our experiments. These tasks are designed to evaluate the model's ability to perform long-range planning and structural reasoning.

\subsection{Countdown}
Countdown is a mathematical planning puzzle where the model is given a set of source numbers and a target value. The goal is to construct a valid arithmetic expression using basic operators $\qty(+,-,\times,/)$ to reach the target.

As illustrated in \cref{fig:illustration_of_countdown_SAT}a, the task can be represented as an expression tree. To solve this, a model cannot simply generate numbers greedily. It must internally plan the hierarchical structure of the computation as a single incorrect move would immediately render the target mathematically unreachable. For instance, to reach the target 19 from $\qty{11,14,40,97}$, the model must lock in the nested structure $\qty(-)/\qty(-)$ before outputting the first token. This ensures the very first operation is a committed step toward a globally valid solution.

\subsection{SAT}
The SAT task requires finding a truth assignment $\qty{0,1}$ that satisfies all clauses of a propositional formula in Conjunctive Normal Form (CNF).

As shown in \cref{fig:illustration_of_countdown_SAT}b, the task is represented as a bipartite factor graph where variables and clauses are interconnected. Within this structure, a single variable often appears in conflicting forms (e.g., $x$ and $\lnot x$) across multiple clauses, meaning any local assignment immediately propagates constraints throughout the entire network. Consequently, an initial choice that satisfies one local clause may preclude any valid assignment for the remaining variables, rendering the entire formula unsatisfiable. The model must therefore internally resolve these interconnected dependencies to ensure that the very first assignment remains compatible with a globally valid configuration.

\section{Technical Preliminaries} \label{sec:tech}


This appendix contains all lemma proofs, reorganised into three groups that mirror
the logical structure of the argument.

\paragraph{Notation summary.}
For convenience, we collect the key symbols used throughout the appendix.
After the block-diagonal factorisation in Lemma~\ref{lemma:forward}, the full input $X^{(0)}$ no longer appears;
all subsequent analysis is expressed in terms of the content matrix~$Z$.

\begin{center}
\renewcommand{\arraystretch}{1.25}
\small
\begin{tabular}{@{}c l c l@{}}
\toprule
\textbf{Symbol} & \textbf{Definition} & \textbf{Space} & \textbf{Introduced} \\
\midrule
$g_t$ & Node identity at position $t$ & $[N]$ & Definition~\ref{def:disentangled_TF} \\
$ {x}_t = ( {e}_{g_t};\, {e}_t)$ & Full token (content + position) & $\mathbb{R}^{N+T}$ & Definition~2 \\
$ {z}_t =  {e}_{g_t}$ & Content embedding at position $t$ & $\mathbb{R}^{N}$ & Lemma~\ref{lemma:forward} \\
$X^{(0)} = [ {x}_1,\dots, {x}_T]^\top$ & Full input matrix (Lemma~\ref{lemma:forward} proof only) & $\mathbb{R}^{T\times(N+T)}$ & Definition~\ref{def:disentangled_TF} \\
$Z = [ {z}_1,\dots, {z}_T]^\top$ & Content matrix & $\mathbb{R}^{T\times N}$ & Lemma~\ref{lemma:forward} \\
$Z'$ & AR content context ($ {z}_T\to {e}_{y^{(1)}}$) & $\mathbb{R}^{T\times N}$ & Equation~\ref{eq:total-loss} \\
\midrule
$\widetilde W^{(\ell)}$ & Full weight, layer $\ell$ & $\mathbb{R}^{2^{\ell-1}(N+T)\times 2^{\ell-1}(N+T)}$ & Equation~\ref{eq:full_iterate} \\
$W_0^{(\ell)}$ & Content weight, layer $\ell$ & $\mathbb{R}^{N\times N}$ & Equation~\ref{eq:first_W} \\
$W_1^{(\ell)}$ & Positional weight, layer $\ell$ & $\mathbb{R}^{T\times T}$ & Equation~\ref{eq:first_W} \\
$A^{[\ell]} = Z W_0^{(\ell)} Z^\top + W_1^{(\ell)}$ & Attention logit matrix, layer $\ell$ & $\mathbb{R}^{T\times T}$ & Lemma~\ref{lemma:forward} \\
$S^{[\ell]}$ & Attention distribution, layer $\ell$; \,$S^{[0]}=I_T$ & $\mathbb{R}^{T\times T}$ & Lemma~\ref{lemma:forward} \\
\midrule
$J( {s}) = \mathrm{diag}( {s}) -  {s} {s}^\top$ & Softmax Jacobian & $\mathbb{R}^{T\times T}$ &  \\
$L$ & Strictly lower shift matrix ($L_{i,i-1}=1$) & $\mathbb{R}^{T\times T}$ &  \\
\midrule
$ {e}_t \in \mathbb{R}^T$ & Positional one-hot (selects row $t$) & $\mathbb{R}^{T}$ & \\
$ {e}_{g_t} \in \mathbb{R}^N$ & Content one-hot (selects node $g_t$) & $\mathbb{R}^{N}$ & \\
\bottomrule
\end{tabular}
\end{center}
 
We follow the notation of Lemma~1 (main text): layers are {1-indexed},
so Layer~1 has weight matrices $(W^{(1)}_0, W^{(1)}_1)$ and produces attention
distribution $S^{[1]}$, while Layer~2 has $(W^{(2)}_0, W^{(2)}_1)$ and produces
$S^{[2]}$. The initialisation $S^{[0]} := I_T$ is not a layer but the identity
prior to any attention. Throughout, we write $J({s}) := \mathrm{diag}({s}) - {s}{s}^\top$
for the Jacobian of the softmax evaluated at a distribution ${s}$,
$L \in \mathbb{R}^{T \times T}$ for the strictly lower shift matrix with $L_{i,i-1}=1$,
and $I_T$ for the $T\times T$ identity.

\begin{itemize}
  \item \textbf{Section~\ref{sec:forward}} establishes the forward-pass factorisation
        that underlies all subsequent gradient calculations (Lemma~\ref{lemma:forward}).
  \item \textbf{Section~\ref{sec:deep_head}} derives the gradient structure of the
        \emph{deep} head $\mathcal{L}_1$, which couples both layers
        (Lemmas~\ref{lemma:dl0} and~\ref{lemma:dl0w}).
  \item \textbf{Section~\ref{sec:shallow_head}} derives the gradient structure of the
        \emph{shallow} MTP head $\mathcal{L}_2$ and proves the gradient decoupling
        property that is the key mechanism of MTP
        (Lemmas~\ref{lemma:dl1} and~\ref{lemma:dl1w}).
\end{itemize}

\subsection{Forward-Pass Factorisation}
\label{sec:forward}

\begin{proof}[Proof of Lemma \ref{lemma:forward}]
We specialize Definition~\ref{def:disentangled_TF} to $L=2$ layers with one head per layer and unpack the disentangled structure.

By Definition~\ref{def:disentangled_TF}, the input at position $t$ is
\[
  x_t = \begin{bmatrix} z_t \\ e_t \end{bmatrix} \in \mathbb{R}^{d_0},
  \qquad d_0 = N + T,
\]
where $z_t = e_{g_t} \in \mathbb{R}^N$ is the one-hot content embedding and
$e_t \in \mathbb{R}^T$ is the one-hot positional embedding.
Stack these as $X^{(0)} = [x_1, \dots, x_T]^\top \in \mathbb{R}^{T \times d_0}$.

\medskip\noindent\textbf{Layer~1 attention logits.} The disentangled assumption (Definition~\ref{def:disentangled_TF}) means the attention weight matrix
decomposes into content and positional blocks with no cross-terms:
\[
  W^{(\ell)}
  = \begin{bmatrix}
      W_0^{(\ell)} & 0 \\
      0 & W_1^{(\ell)}
    \end{bmatrix},
  \qquad
  W_0^{(\ell)} \in \mathbb{R}^{N \times N},\;
  W_1^{(\ell)} \in \mathbb{R}^{T \times T}.
\]
The $(i,j)$-entry of the raw attention logit matrix is therefore
\[
  \bigl(X^{(0)} W^{(1)} X^{(0)\top}\bigr)_{i,j}
  = x_i^\top W^{(1)} x_j
  = z_i^\top W_0^{(1)} z_j + e_i^\top W_1^{(1)} e_j.
\]
The first term contributes $(Z W_0^{(\ell)} Z^\top)_{i,j}$ (content matching)
and the second contributes $(W_1^{(\ell)})_{i,j}$ (positional bias, since
$e_i^\top W_1^{(1)} e_j$ selects entry $(i,j)$). Hence
\[
  A^{[1]}
  := X^{(0)} W^{(1)} X^{(0)\top}
  = Z\, W_0^{(1)}\, Z^\top + W_1^{(1)}
  \;\in \mathbb{R}^{T \times T}.
\]

Applying the causal mask and row-wise softmax:
\[
  S^{[1]} := \sigma \bigl(\mathrm{MASK}(A^{[1]})\bigr) \in \mathbb{R}^{T \times T}.
\]
The Layer~1 attention output (Definition~\ref{def:attn}) is
$\mathrm{attn}(X^{(0)}; W^{(1)}) = S^{[1]}\, X^{(0)} \in \mathbb{R}^{T \times d_0}$.
Its content part at position $t$ is
\begin{equation}\label{eq:layer0-content}
  \bigl(S^{[1]}\, X^{(0)}\bigr)_{t,\,1:N}
  = S^{[1]}_{t,:}\, Z
  \;\in \mathbb{R}^{1 \times N}.
\end{equation}

The shallow (1-hop) prediction head reads the Layer~1 attention output at the
last position $T$ and projects onto the node vocabulary $[N]$.
By~\eqref{eq:layer0-content}:
\[
  \boxed{\;f_2(Z) = S^{[1]}_{T,:}\, Z \;\in \mathbb{R}^{1 \times N}.\;}
\]

\medskip\noindent\textbf{Layer~2 attention logits.} By Definition~\ref{def:disentangled_TF},
$X^{(1)} = \bigl[X^{(0)},\; S^{[1]}\, X^{(0)}\bigr] \in \mathbb{R}^{T \times 2d_0}$.
Under the disentangled architecture, Layer~2's attention query at position $i$
is derived from the Layer~1 attention output $(S^{[1]} X^{(0)})_i$, while the
keys remain the original input tokens $x_j$.\footnote{This is the property of the disentangled transformer with sparse parameters: each
layer's query reads from the previous attention output, while the keys are
always the original input tokens. This avoids the quadratic blowup from the
concatenated representation and yields the clean recursive formula below.}

The Layer~2 attention logit at $(i,j)$ is then
\[
  \bigl(S^{[1]} X^{(0)}\bigr)_i^\top\, W^{(2)}\, x_j.
\]
Applying the disentangled block-diagonal decomposition of $W^{(2)}$:
\begin{align*}
  \bigl(S^{[1]} X^{(0)}\bigr)_i^\top\, W^{(2)}\, x_j
  &= \bigl(S^{[1]}_{i,:}\, Z\bigr)\, W_0^{(2)}\, z_j
     \;+\; \bigl(S^{[1]}_{i,:}\bigr)\, W_1^{(2)}\, e_j.
\end{align*}
In full matrix form, the content term gives
$S^{[1]}\, Z\, W_0^{(2)}\, Z^\top$
and the positional term gives
$S^{[1]}\, W_1^{(2)}$.
Combining:
\[
  \text{Layer~2 logit matrix}
  = S^{[1]}\bigl(Z\, W_0^{(2)}\, Z^\top + W_1^{(2)}\bigr)
  = S^{[1]}\, A^{[2]}.
\]
Applying the causal mask and row-wise softmax:
\[
  S^{[2]} := \sigma \bigl(\mathrm{MASK}(S^{[1]}\, A^{[2]})\bigr)
  \;\in \mathbb{R}^{T \times T},
\]
which confirms the recursive formula
$S^{[\ell]} = \sigma \bigl(\mathrm{MASK}(S^{[\ell-1]}\, A^{[\ell]})\bigr)$
with initial condition $S^{[0]} = I_T$.

The deep (2-hop) head reads from Layer~2's attention output at the last
position.  Layer~2 aggregates the Layer~1 attention outputs weighted by
$S^{[2]}$, so the content part at position $T$ is
\[
  S^{[2]}_{T,:}\,\bigl(S^{[1]}\, Z\bigr)
  = S^{[2]}_{T,:}\, S^{[1]}\, Z
  \;\in \mathbb{R}^{1 \times N}.
\]
Hence:
\[
  \boxed{\;f_1(Z) = S^{[2]}_{T,:}\, S^{[1]}\, Z \;\in \mathbb{R}^{1 \times N}.\;}
\]

Substituting definitions, we can verify the expanded expression.
For the last row $T$, all positions $j \le T$ are visible under the causal mask,
so the mask is vacuous for row $T$ of $S^{[1]} A^{[2]}$:
\begin{align*}
  S^{[2]}_{T,:}
  &= \sigma \bigl(S^{[1]}_{T,:}\, A^{[2]}\bigr) \\
  &= \sigma \Bigl(
       \underbrace{\sigma \bigl(z_T^\top W_0^{(1)} Z^\top + W_1^{(1)}[T,:]\bigr)}_{S^{[1]}_{T,:}}
       \;\cdot\;
       \bigl(Z\,W_0^{(2)}\,Z^\top + W_1^{(2)}\bigr)
     \Bigr),
\end{align*}
and therefore
\[
  f_1(Z)
  = \sigma \Bigl(
      \sigma \bigl(z_T^\top W_0^{(1)} Z^\top + W_1^{(1)}[T,:]\bigr)
      \bigl(Z\,W_0^{(2)}\,Z^\top + W_1^{(2)}\bigr)
    \Bigr)
    \sigma \bigl(\mathrm{MASK}(Z\,W_0^{(1)}\,Z^\top + W_1^{(1)})\bigr) Z,
\]
matching the stated expression.
\end{proof}

\subsection{Gradient Structure of the Deep Head}
\label{sec:deep_head}

The deep loss $\mathcal{L}_1=-\log f_1(Z)\cdot e_{y^{(1)}}$ passes gradients through
\emph{both} Layer~1 and Layer~2.
Lemma~\ref{lemma:dl0} identifies the three back-propagation paths; Lemma~\ref{lemma:dl0w} makes these explicit
for each parameter matrix.

\begin{lemma} \label{lemma:dl0}
 Let $\mathcal{L}_1(Z) = -\log(f_1(Z) \cdot e_{y^{(1)}})$ be the loss for the deep head evaluated on context $Z \in \{ Z, Z'\}$ with target $y$. Let $S^{[1]}, A^{[2]}, S^{[2]}$ be the forward activations for $Z$. Then:
\begin{align*}
\mathrm{d}\mathcal{L}_1(Z) 
&= -\frac{1}{f_1(Z) \cdot e_{y^{(1)}} } \cdot \Big( S^{[1]}_{[-1, :]} \mathrm{d}A^{[2]}  J(S^{[2]}_{[-1,:]}) S^{[1]} \Big)Z   e_{y^{(1)}}  \\
&\quad - \frac{1}{f_1(Z) \cdot  e_{y^{(1)}} } \cdot \Big( \mathrm{d}S^{[1]}_{[-1, :]}A^{[2]}  J(S^{[2]}_{[-1,:]}) S^{[1]} \Big)Z   e_{y^{(1)}} \\
&\quad - \frac{1}{f_1(Z) \cdot  e_{y^{(1)}}} \cdot \Big( S^{[2]}_{[-1, :]}\mathrm{d}S^{[1]} \Big)Z   e_{y^{(1)}} .
\end{align*}
\end{lemma}

\begin{proof}[Proof of Lemma \ref{lemma:dl0}]
We apply the chain rule to
$\mathcal{L}_1(Z) = -\log \bigl(f_1(Z) \cdot e_{y^{(1)}} \bigr)$
in following steps.
 
Outer log derivative.
\[
  \mathrm{d}\mathcal{L}_1
  = -\frac{1}{f_1(Z) \cdot e_{y^{(1)}} }\;\mathrm{d} \bigl(f_1(Z) \cdot  e_{y^{(1)}} \bigr).
\]
 
Expand $f_1$.
Substituting $f_1(Z) = S^{[2]}_{[-1,:]}\,S^{[1]}\,Z$ from
Lemma~\ref{lemma:forward}:
\[
  \mathrm{d} \bigl(f_1(Z) \cdot  e_{y^{(1)}} \bigr)
  = \mathrm{d} \Bigl(S^{[2]}_{[-1,:]}\,S^{[1]}\,Z\, e_{y^{(1)}} \Bigr).
\]
 
Product rule over the two attention layers.
$S^{[2]}_{[-1,:]}$ and $S^{[1]}$ are both functions of the parameters, so:
\[
  \mathrm{d} \Bigl(S^{[2]}_{[-1,:]}\,S^{[1]}\,Z\,e_{y^{(1)}} \Bigr)
  = \bigl(\mathrm{d}S^{[2]}_{[-1,:]}\bigr)\,S^{[1]}\,Z\,e_{y^{(1)}}
    \;+\; S^{[2]}_{[-1,:]}\,\bigl(\mathrm{d}S^{[1]}\bigr)\,Z\,e_{y^{(1)}} .
\]
The second term is path~(iii) directly.
For the first term, we further expand $\mathrm{d}S^{[1]}_{[-1,:]}$
using the softmax Jacobian.

By the recursive formula $S^{[2]} = \sigma(\mathrm{MASK}(S^{[1]}\,A^{[2]}))$,
the last-row logit vector is
$\ell^{[2]} := S^{[1]}_{[-1,:]}\,A^{[2]}$
(the causal mask is vacuous for the last row). Its differential is
\[
  \mathrm{d}\ell^{[2]}
  = S^{[1]}_{[-1,:]}\;\mathrm{d}A^{[2]}
    + \mathrm{d}S^{[1]}_{[-1,:]}\;A^{[2]}.
\]
Applying the softmax Jacobian $J(S^{[2]}_{[-1,:]})$:
\[
  \mathrm{d}S^{[2]}_{[-1,:]}
  = \mathrm{d}\ell^{[2]}\;J \bigl(S^{[2]}_{[-1,:]}\bigr)
  = \Bigl(
      S^{[1]}_{[-1,:]}\;\mathrm{d}A^{[2]}
      + \mathrm{d}S^{[1]}_{[-1,:]}\;A^{[2]}
    \Bigr)\;
    J \bigl(S^{[2]}_{[-1,:]}\bigr).
\]

Substituting back and grouping by differential:
\begin{align*}
  \mathrm{d} \bigl(f_1 \cdot e_{y^{(1)}} \bigr)
  &= \underbrace{%
       S^{[1]}_{[-1,:]}\;\mathrm{d}A^{[2]}\;
       J(S^{[2]}_{[-1,:]})\;S^{[1]}\,Z\, e_{y^{(1)}}
     }_{\text{path (i)}}
  \\[4pt]
  &\quad+ \underbrace{%
       \mathrm{d}S^{[1]}_{[-1,:]}\;A^{[2]}\;
       J(S^{[2]}_{[-1,:]})\;S^{[1]}\,Z\, e_{y^{(1)}}
     }_{\text{path (ii)}}
  \\[4pt]
  &\quad+ \underbrace{%
       S^{[2]}_{[-1,:]}\;\mathrm{d}S^{[1]}\;Z\, e_{y^{(1)}}
     }_{\text{path (iii)}}.
\end{align*}
Dividing by $- f_1(Z) \cdot e_{y^{(1)}}  $ yields the stated expression.
\end{proof}

The three terms correspond to distinct backpropagation paths through the two-layer network.

\begin{lemma}  \label{lemma:dl0w}
     {Explicit gradients for $\mathcal{L}_1(Z)$ w.r.t. all four parameter matrices.} 
\begin{align*}
\frac{\partial \mathcal{L}_1(Z)}{\partial W^{(1)}_0}  = Z^\top \left[ \frac{\partial \mathcal{L}_1}{\partial A^{[1]}} \right] Z. \qquad 
\frac{\partial \mathcal{L}_1(Z)}{\partial W^{(1)}_1} = \left[ \frac{\partial \mathcal{L}_1}{\partial A^{[1]}} \right]. 
\end{align*}
where $\frac{\partial \mathcal{L}_1}{\partial A^{[1]}}
=     M \odot \left[  S^{[1]} \odot \left( \frac{\partial \mathcal{L}_1}{ \partial  S^{[1]}} -  \frac{\partial \mathcal{L}_1}{ \partial  S^{[1]}} \odot   S^{[1]}  {1}  {1} ^\top  \right)  \right] \in \mathbb{R}^{T \times T}$ is the loss gradient w.r.t.\ the Layer~1
logit matrix.
\begin{align*}
\frac{\partial \mathcal{L}_1(Z)}{\partial W^{(2)}_0} &= -\frac{1}{f_1(Z) \cdot  e_{y^{(1)}} } \cdot Z^\top   J \bigl(S^{[2]}_{[-1,:]}\bigr) S^{[1]}Z  e_{y^{(1)}} S^{[1]}_{[-1, :]}Z. \\
\frac{\partial \mathcal{L}_1(Z)}{\partial W^{(2)}_1} &= -\frac{1}{f_1(Z) \cdot e_{y^{(1)}} } \cdot   J \bigl(S^{[2]}_{[-1,:]}\bigr) S^{[1]}Z  e_{y^{(1)}} S^{[1]}_{[-1, :]}.
\end{align*}
\end{lemma}

\begin{proof}[Proof of Lemma~\ref{lemma:dl0w}]
We derive each gradient by applying the chain rule through
$A^{[\ell]} = Z\,W_0^{(\ell)}\,Z^\top + W_1^{(\ell)}$.
 
\textbf{Layer~1 parameters.}
Since $A^{[1]} = Z\,W_0^{(1)}\,Z^\top + W_1^{(1)}$, we have
$\mathrm{d}A^{[1]} = Z\,\mathrm{d}W_0^{(1)}\,Z^\top + \mathrm{d}W_1^{(1)}$.
By the chain rule and the trace identity
$\mathrm{tr}(G\,Z\,\mathrm{d}W\,Z^\top) = \mathrm{tr}(Z^\top G\,Z\,\mathrm{d}W)$:
\[
  \frac{\partial \mathcal{L}_1}{\partial W_0^{(1)}}
  = Z^\top\,\frac{\partial \mathcal{L}_1}{\partial A^{[1]}}\,Z,
  \qquad
  \frac{\partial \mathcal{L}_1}{\partial W_1^{(1)}}
  = \frac{\partial \mathcal{L}_1}{\partial A^{[1]}}.
\]

  There are two paths into $S^{[1]}$
\begin{align*}
   \frac{\partial \mathcal{L}_1}{ \partial S^{[1]}} =   -\frac{1}{ f_1\cdot e_{y^\qty(1)}}\cdot \left[   (S^{[2]}_{[-1,:]})^\top (Z   e_{y^\qty(1)})^\top   +  A^\qty[2]  {J}(S^{[2]}_{[-1,:]}) S^{[1]}    Z  e_{y^\qty(1)} e^\top_{T} \right],  
\end{align*}

For each row $r$ with mask $ M_r = \mathrm{diag} ({m}_r)$, then
\begin{align*}
   \frac{\partial \mathcal{L}_1}{ \partial A^{[1]}_{[r,:]}} =  \frac{\partial \mathcal{L}_1}{ \partial S^{[1]}_{[r,:]}}  {J}(s_1(r)) M_r.
\end{align*}
Then the matrix format is
\begin{align*}
   \frac{\partial \mathcal{L}_1}{ \partial A^{[1]}} =  M \odot \left[  S^{[1]} \odot \left( \frac{\partial \mathcal{L}_1}{ \partial  S^{[1]}} -  \frac{\partial \mathcal{L}_1}{ \partial  S^{[1]}} \odot   S^{[1]}  {1}  {1} ^\top  \right)  \right].
\end{align*}
 
\textbf{Layer~2 parameters.}
Only path~(i) of Lemma~\ref{lemma:dl0w} involves $A^{[2]}$
(paths~(ii) and~(iii) are independent of Layer~2 parameters).
The relevant differential term is:
\begin{equation}\label{eq:path-i-scalar}
  -\frac{1}{ f_1\cdot e_{y^\qty(1)} }\;
  S^{[1]}_{[-1,:]}\;\mathrm{d}A^{[2]}\;
  \underbrace{J \bigl(S^{[2]}_{[-1,:]}\bigr)\,S^{[1]}\,Z\, e_{y^\qty(1)}  }_{=:\,r\;\in\;\mathbb{R}^{T \times 1}}.
\end{equation}
 
\emph{Gradient w.r.t.\ $W_0^{(2)}$.}\;
Substituting $\mathrm{d}A^{[2]} = Z\,\mathrm{d}W_0^{(2)}\,Z^\top$
into~\eqref{eq:path-i-scalar}:
\[
  - \frac{1}{ f_1\cdot e_{y^\qty(1)} } \;
  \underbrace{S^{[1]}_{[-1,:]}\,Z}_{1 \times N}\;
  \mathrm{d}W_0^{(2)}\;
  \underbrace{Z^\top r}_{N \times 1}.
\]
This is a scalar of the form $c^\top\,\mathrm{d}W\,d$
with $c = Z^\top(S^{[1]}_{[-1,:]})^\top$ and $d = Z^\top r$.
Applying the trace identity
$c^\top\,\mathrm{d}W\,d = \mathrm{tr}(d\,c^\top\,\mathrm{d}W)$:
\begin{align*}
  \frac{\partial \mathcal{L}_1}{\partial W_0^{(2)}}
  & = -\frac{1}{ f_1\cdot e_{y^\qty(1)} } \;d\,c^\top
  = -\frac{1}{ f_1\cdot e_{y^\qty(1)} } \;Z^\top r\;S^{[1]}_{[-1,:]}\,Z \\
  & = -\frac{1}{ f_1\cdot e_{y^\qty(1)} } \;Z^\top\,J \bigl(S^{[2]}_{[-1,:]}\bigr)\,
    S^{[1]}\,Z\,e_{y^\qty(1)} \;S^{[1]}_{[-1,:]}\,Z.
\end{align*}
 
\emph{Gradient w.r.t.\ $W_1^{(2)}$.}\;
Substituting $\mathrm{d}A^{[2]} = \mathrm{d}W_1^{(2)}$
into~\eqref{eq:path-i-scalar}:
\[
  -\frac{1}{ f_1\cdot e_{y^\qty(1)} } \;
  \underbrace{S^{[1]}_{[-1,:]}}_{1 \times T}\;
  \mathrm{d}W_1^{(2)}\;
  \underbrace{r}_{T \times 1}.
\]
Again of the form $c^\top\,\mathrm{d}W\,d$
with $c = (S^{[1]}_{[-1,:]})^\top$ and $d = r$.
By the same trace identity:
\[
  \frac{\partial \mathcal{L}_1}{\partial W_1^{(2)}}
  = -\frac{1}{ f_1\cdot e_{y^\qty(1)} } \;d\,c^\top
  = -\frac{1}{ f_1\cdot e_{y^\qty(1)} } \;r\;S^{[1]}_{[-1,:]}
  = -\frac{1}{ f_1\cdot e_{y^\qty(1)} } \;J \bigl(S^{[2]}_{[-1,:]}\bigr)\,
    S^{[1]}\,Z\,e_{y^\qty(1)} \;S^{[1]}_{[-1,:]}.
    \qedhere
\]
\end{proof}

\subsection{Gradient Decoupling of the Shallow MTP Head}
\label{sec:shallow_head}

The shallow MTP loss $\mathcal{L}_2$ is the key ingredient distinguishing MTP from NTP.
Lemma~\ref{lemma:dl1} establishes the central {gradient decoupling} property:
$\mathcal{L}_2$ routes gradients \emph{exclusively} through Layer~1, entirely
bypassing Layer~2. This provides a clean training signal for Layer~1 that does not depend on the
(potentially uninitialised) Layer~2 weights---the mechanism enabling Phase~I of
Theorem~\ref{thm:cascaded}.
Lemma~\ref{lemma:dl1w} makes this explicit per parameter matrix.

\begin{lemma} \label{lemma:dl1}
 The gradient for the shallow head loss $\mathcal{L}_2 = -\log(f_2( {Z}) \cdot e_{y^{(2)}})$ is:
\begin{align*}
\mathrm{d}\mathcal{L}_2  
&= -\frac{1}{f_2 \cdot e_{y^{(2)}}} \cdot \Big(  {z}_T^\top \mathrm{d}W^{(1)}_0  {Z}^\top + \mathrm{d}W^{(1)}_1[-1, :] \Big)  J(S^{[1]}_{[-1,:]})  {Z}   e_{y^{(2)}}.
\end{align*}
\end{lemma}

\begin{proof}[Proof of Lemma~\ref{lemma:dl1}]
We apply the chain rule to
$\mathcal{L}_2 = -\log(f_2 \cdot e_{y^{(2)}} )$.
 
Outer derivative:
\[
  \mathrm{d}\mathcal{L}_2
  = -\frac{1}{f_2 \cdot e_{y^{(2)}} }\;\mathrm{d} \bigl(f_2( {Z}) \cdot e_{y^{(2)}}\bigr).
\]
 
Expand $f_2$:
By Lemma~\ref{lemma:forward},
$f_2( {Z}) = S^{[1]}_{[-1,:]}\, {Z}$, so
\[
  \mathrm{d} \bigl(f_2( {Z}) \cdot e_{y^{(2)}}\bigr)
  = \mathrm{d} \bigl(S^{[1]}_{[-1,:]}\bigr)\; {Z}\,e_{y^{(2)}}.
\]
Note that $ {X}$ is a fixed input matrix (not a function of
parameters), so only $S^{[1]}_{[-1,:]}$ contributes a differential.
 
Softmax Jacobian.
$S^{[1]}_{[-1,:]} = \sigma(A^{[1]}_{[-1,:]})$ where
$A^{[1]}_{[-1,:]}$ is the last-row logit vector
(the causal mask is vacuous for the last row). Therefore:
\[
  \mathrm{d}S^{[1]}_{[-1,:]}
  = \mathrm{d}A^{[1]}_{[-1,:]}\;J \bigl(S^{[1]}_{[-1,:]}\bigr).
\]
 
Expand $\mathrm{d}A^{[1]}_{[-1,:]}$.
From $A^{[1]} =  {Z}\,W_0^{(1)}\, {Z}^\top + W_1^{(1)}$,
the last row is
$A^{[1]}_{[-1,:]} =  {z}_T^\top\,W_0^{(1)}\, {Z}^\top + W_1^{(1)}[T,:]$.
Differentiating with respect to the parameters:
\[
  \mathrm{d}A^{[1]}_{[-1,:]}
  =  {z}_T^\top\;\mathrm{d}W_0^{(1)}\; {Z}^\top
    + \mathrm{d}W_1^{(1)}[T,:].
\]
This involves only $(W_0^{(1)}, W_1^{(1)})$,
confirming that $\mathcal{L}_2$ is independent of Layer~2.
 
{Combining.}
\[
  \mathrm{d}\mathcal{L}_2
  = -\frac{1}{f_2( {Z}) \cdot e_{y^{(2)}}}\;
    \bigl(
      {z}_T^\top\;\mathrm{d}W_0^{(1)}\; {Z}^\top
      + \mathrm{d}W_1^{(1)}[T,:]
    \bigr)\;
    J \bigl(S^{[1]}_{[-1,:]}\bigr)\;
     {Z}\,e_{y^{(2)}}.
    \qedhere
\]
\end{proof}

Note that $\mathrm{d}\mathcal{L}_2$ involves only the Layer~1 parameters
$(W_0^{(1)},\, W_1^{(1)})$ and is \emph{entirely independent of Layer~2}.
This is the central mechanism of MTP: the shallow head provides a {direct, uncontaminated gradient signal} to Layer~1 that does not
need to backpropagate through the (potentially uninitialized or poorly
conditioned) Layer~2 weights.


\begin{lemma} \label{lemma:dl1w}
 {Explicit gradients for $\mathcal{L}_2$ w.r.t. all two parameter matrices.}
\begin{align*}
\frac{\partial \mathcal{L}_2}{\partial W^{(1)}_0} &= -\frac{1}{f_2 \cdot e_{y^{(2)}}} \cdot  {z}_T\;
     e_{y^{(2)}}^\top\, {Z}^\top\,
     J \bigl(S^{[1]}_{[-1,:]}\bigr)\, {Z}. \\
\frac{\partial \mathcal{L}_2}{\partial W^{(1)}_1} &= -\frac{1}{f_2 \cdot e_{y^{(2)}}} \cdot \ e_T\;
     e_{y^{(2)}}^\top\, {Z}^\top\,
     J \bigl(S^{[1]}_{[-1,:]}\bigr).
\end{align*}
\end{lemma}

\begin{proof}[Proof of Lemma~\ref{lemma:dl1w}]
We route the differential from Lemma~\ref{lemma:dl1} through
$A^{[1]} =  {Z}\,W_0^{(1)}\, {Z}^\top + W_1^{(1)}$.
 
{Step 1: Intermediate gradient $\partial\mathcal{L}_2/\partial A^{[1]}$.}
Since only row~$T$ of $A^{[1]}$ affects the loss.
Applying the softmax chain rule:
\[
  \frac{\partial \mathcal{L}_2}{\partial A^{[1]}_{T,:}}
  = -\frac{1}{f_2 \cdot e_{y^{(2)}}} \;
    e_{y^{(2)}}^\top\, {Z}^\top\,J \bigl(S^{[1]}_{[-1,:]}\bigr)
  \;\in \mathbb{R}^{1 \times T},
\]
and $\frac{\partial \mathcal{L}_2}{\partial A^{[1]}_{r,:}} = 0$
for $r \ne T$.  In matrix form:
\begin{equation}\label{eq:dL1-dA0}
  \frac{\partial \mathcal{L}_2}{\partial A^{[1]}}
  = -\frac{1}{f_2 \cdot e_{y^{(2)}}} \;
    e_T\;e_{y^{(2)}}^\top\, {Z}^\top\,
    J \bigl(S^{[1]}_{[-1,:]}\bigr)
  \;\in \mathbb{R}^{T \times T}.
\end{equation}
 
{Step 2: Gradient w.r.t.\ $W_1^{(1)}$.}
Since $A^{[1]} =  {Z}\,W_0^{(1)}\, {Z}^\top + W_1^{(1)}$,
the Jacobian $\partial A^{[1]}_{ij}/\partial (W_1^{(1)})_{ij} = 1$,
so:
\[
  \frac{\partial \mathcal{L}_2}{\partial W_1^{(1)}}
  = \frac{\partial \mathcal{L}_2}{\partial A^{[1]}}
  = -\frac{1}{f_2 \cdot e_{y^{(2)}}}  \;
    e_T\;e_{y^{(2)}}^\top\, {Z}^\top\,
    J \bigl(S^{[1]}_{[-1,:]}\bigr).
\]

This is nonzero only in row~$T$, since the prefactor $e_T$ zeroes
all other rows.
 
{Step 3: Gradient w.r.t.\ $W_0^{(1)}$.}
Using the standard identity for bilinear forms
($A =  {Z}\,W\, {Z}^\top$ implies
$\partial f/\partial W = {Z}^\top\,(\partial f/\partial A)\, {Z}$):
\[
  \frac{\partial \mathcal{L}_2}{\partial W_0^{(1)}}
  =  {Z}^\top\,
    \frac{\partial \mathcal{L}_2}{\partial A^{[1]}}\,
    {Z}
  = -\frac{1}{f_2 \cdot e_{y^{(2)}}}  \;
    \underbrace{ {Z}^\top e_T}_{=  {z}_T}\;
    e_{y^{(2)}}^\top\, {Z}^\top\,
    J \bigl(S^{[1]}_{[-1,:]}\bigr)\, {Z}.
\]
\end{proof}

Together, Lemmas~\ref{lemma:dl0w} and \ref{lemma:dl1w}
provide the complete gradient landscape.

\section{Complete Proof of Theorems} \label{sec:proof}

\subsection{The Reverse Reasoning Circuit} \label{appdix:circuit}

\begin{proof} [Proof of Theorem \ref{thm:stationary}]

By the chain rule, the gradient passing through any softmax distribution~$s$
is premultiplied by its Jacobian $J(s) = \mathrm{diag}(s) - s\,s^\top$.
When $s$ is a saturated one-hot vector $e_k$, the Jacobian evaluates to the
zero matrix: $J(e_k) = 0$.
We show that every active gradient chain passes through at least one such
zero Jacobian.

By Lemma~\ref{lemma:dl1w}, $\nabla_{W^{(1)}}\mathcal{L}_2$ routes
through $J(S^{[1]}_{T,:})$.
Under Condition~1, $S^{[1]}_{T,:} = e_{T-1}^\top$ is one-hot, so
$J(S^{[1]}_{T,:}) = 0$.
Therefore $\nabla_{W^{(1)}}\mathcal{L}_2 = 0$.

By Lemma~\ref{lemma:dl0w}, both
$\partial\mathcal{L}_1/\partial W_0^{(2)}$ and
$\partial\mathcal{L}_1/\partial W_1^{(2)}$ contain the factor
$J(S^{[2]}_{T,:})$.
Under Condition~2, $S^{[2]}_{T,:} = e_{t_{\mathrm{end}}^{\mathrm{ctx}}}^\top$
is one-hot, so $J = 0$.
Therefore $\nabla_{W^{(2)}}\mathcal{L}_1 = 0$.

By Lemma~\ref{lemma:dl0w}, the gradient
$\nabla_{W^{(1)}}\mathcal{L}_1$ receives contributions from three paths:
 
\begin{itemize} 
  \item \emph{Attention path~(i) and query path~(ii):}
    Both are gated by $J(S^{[2]}_{T,:})$, which vanishes under Condition~2
    (same argument as case~2). Both paths are completely severed.
 
  \item \emph{Value path~(iii):}
    The error signal is
    $S^{[2]}_{T,:}\,\mathrm{d}S^{[1]}\, {Z}\,e_{y^{(1)}}$.
    Under Condition~2, $S^{[2]}_{T,:} = e_{t_{\mathrm{end}}^{\mathrm{ctx}}}^\top$,
    which selects only row $t_{\mathrm{end}}^{\mathrm{ctx}}$ of $\mathrm{d}S^{[1]}$.
    To reach the parameters $W^{(1)}$, this signal must pass through the
    row-wise softmax Jacobian
    $J \bigl((S^{[1]})_{t_{\mathrm{end}}^{\mathrm{ctx}},\,:}\bigr)$.
    Under Condition~1,
    $(S^{[1]})_{t_{\mathrm{end}}^{\mathrm{ctx}},\,:}
    = e_{t_v^{\mathrm{ctx}}}^\top$
    is one-hot, so this Jacobian is also zero.
    The value path is severed.
\end{itemize}

All three paths vanish, yielding $\nabla_{W^{(1)}}\mathcal{L}_1 = 0$.

Since Layer~2 selects only row $t_{\mathrm{end}}^{\mathrm{ctx}}$
via the one-hot $S^{[2]}_{T,:}$, all other context rows
$r \neq t_{\mathrm{end}}^{\mathrm{ctx}}$ of~$S^{[1]}$ receive zero
error signal from $\mathcal{L}_1$.
Combined with the vanishing of $\nabla_{W^{(1)}}\mathcal{L}_1$
(which only involves row~$T$), these rows are gradient-free
and therefore unconstrained at the stationary manifold.
\end{proof}

\begin{proof}[Proof of Corollary \ref{cor:construction}]
We trace the forward pass under the proposed weight configuration for
finite $\gamma > 0$ and show that every quantity deviates from its ideal
(saturated) value by at most $O(e^{-\gamma})$.
 
\paragraph{Layer~1: Universal predecessor pointer.}
With $W_0^{(1)}=0$, the content term vanishes and the logit matrix
reduces to $A^{[1]}=W_1^{(1)}=\gamma\,L$.  Under the causal mask, for
each row $t>1$ the entry at position $t-1$ has logit $\gamma$ while all
other visible entries have logit $0$.  The softmax therefore gives
\begin{equation}\label{eq:layer1-finite}
  S_{t,t-1}^{[1]}
    = \frac{e^{\gamma}}{e^{\gamma}+(t-1)}
    = 1 - \frac{t-1}{e^{\gamma}+(t-1)},
  \qquad
  S_{t,j}^{[1]}
    = \frac{1}{e^{\gamma}+(t-1)}
  \quad (j \neq t-1).
\end{equation}
Since $t\le T$ is fixed, every off-diagonal entry is $O(e^{-\gamma})$
uniformly, so
\[
  S_{t,:}^{[1]} = e_{t-1}^{\top} + O(T\,e^{-\gamma}),
  \qquad \forall\; t>1.
\]
In particular, this approximately satisfies both parts of Condition~1:
\begin{itemize}
\item Row $T$ attends to $T-1$:
  $S_{T,:}^{[1]}=e_{T-1}^{\top}+O(T\,e^{-\gamma})$.
\item Row $t_{\mathrm{end}}^{\mathrm{ctx}}$ attends to
  $t_{\mathrm{end}}^{\mathrm{ctx}}-1 = t_{v}^{\mathrm{ctx}}$:
  $S_{t_{\mathrm{end}}^{\mathrm{ctx}},:}^{[1]}
   = e_{t_{v}^{\mathrm{ctx}}}^{\top}+O(T\,e^{-\gamma})$.
\end{itemize}
 
\paragraph{Layer~2: Content matching with self-mask.}
By~\eqref{eq:layer1-finite}, the Layer~1 attention output at the query
position~$T$ is
\[
  (S^{[1]}\,  Z)_T
    = S_{T,:}^{[1]}\,  Z
    = \bigl(1 - O(T\,e^{-\gamma})\bigr)\,e_{u_{\mathrm{end}}}
      + O(T\,e^{-\gamma}).
\]
Hence the effective content query for Layer~2 is
$e_{u_{\mathrm{end}}}+O(e^{-\gamma})$.
 
With $W_0^{(2)}=\gamma\,I$, the content logit from this query to key
position~$j$ is
\[
  \bigl(e_{u_{\mathrm{end}}}+O(e^{-\gamma})\bigr)^{ \top}
    (\gamma\,I)\,z_j
  \;=\;
  \gamma\;\delta_{[z_j = e_{u_{\mathrm{end}}}]}
  \;+\; O(\gamma\,e^{-\gamma}).
\]
In the input sequence, $u_{\mathrm{end}}$ appears at exactly two
positions: the context position $t_{\mathrm{end}}^{\mathrm{ctx}}$ and
the prompt position $T-1$.
 
The positional mask
$(W_1^{(2)})_{T-1,\,T-1}=-\gamma$ reduces the total logit at the
prompt self-match $j=T-1$ from
$\gamma+O(\gamma\,e^{-\gamma})$ to $O(\gamma\,e^{-\gamma})$, while the
logit at $j=t_{\mathrm{end}}^{\mathrm{ctx}}$ remains
$\gamma+O(\gamma\,e^{-\gamma})$.  All other positions have logit
$O(\gamma\,e^{-\gamma})$.  Therefore the Layer~2 softmax at row~$T$
satisfies
\[
  S_{T,\,t_{\mathrm{end}}^{\mathrm{ctx}}}^{[2]}
    = 1 - O(T\,e^{-\gamma}),
  \qquad
  S_{T,j}^{[2]} = O(e^{-\gamma})
  \quad (j \neq t_{\mathrm{end}}^{\mathrm{ctx}}),
\]
which approximately satisfies Condition~2.
 
\paragraph{Deep-head output and loss bound.}
Composing both layers via Lemma~\ref{lemma:forward}:
\begin{align*}
  f_1( Z)\cdot e_{y^{(1)}}
  &= S_{T,:}^{[2]}\;S^{[1]}\;  Z\;e_{y^{(1)}} \\
  &= \bigl(1-O(T\,e^{-\gamma})\bigr)\;
     \bigl(S^{[1]}\,  Z\bigr)_{t_{\mathrm{end}}^{\mathrm{ctx}}}
     \cdot e_{y^{(1)}}
     \;+\; O(T\,e^{-\gamma}) \\
  &= \bigl(1-O(T\,e^{-\gamma})\bigr)\;
     \bigl(e_{t_{v}^{\mathrm{ctx}}}^{\top}\,  Z\,e_{y^{(1)}}
           + O(T\,e^{-\gamma})\bigr)
     \;+\; O(T\,e^{-\gamma}).
\end{align*}
Since
$e_{t_{v}^{\mathrm{ctx}}}^{\top}  Z\,e_{y^{(1)}}
 =   z_{t_{v}^{\mathrm{ctx}}}\cdot e_{y^{(1)}}
 = e_v \cdot e_v = 1$,
we obtain
\[
  f_1(  Z)\cdot e_{y^{(1)}} = 1 - O(T^2\,e^{-\gamma}).
\]
Therefore, since $T=10$ is fixed,
\[
  \mathcal L_{1a}
  = -\log \bigl(f_1(  Z)\cdot e_{y^{(1)}}\bigr)
  = -\log \bigl(1 - O(e^{-\gamma})\bigr)
  = O(e^{-\gamma}).
\]
An identical argument applied to $\mathcal L_2 = -\log(f_2(  Z)\cdot e_{y^{(2)}})$
yields $\mathcal L_2 = O(e^{-\gamma})$.
Hence $\mathcal L = O(e^{-\gamma})$.

Every gradient chain in $\nabla_\theta\mathcal L$ passes through at
least one row-wise softmax Jacobian
$J(s)=\operatorname{diag}(s)-s\,s^{\top}$.  Under the above weight
configuration, each such~$s$ is $O(e^{-\gamma})$-close to a one-hot
vector~$e_k$, so
\[
  \|J(s)\|
  = \|J(e_k + O(e^{-\gamma}))\|
  = O(e^{-\gamma}),
\]
since $J(e_k)=0$ and the Jacobian is Lipschitz in its argument.  All
remaining factors in the gradient expressions (Lemmas~3 and~5) are
bounded by polynomial functions of~$\gamma$ and~$T$.  Therefore
\[
  \|\nabla_\theta\mathcal L\| = O(e^{-\gamma}).
  \qedhere
\]
\end{proof}

\begin{corollary}[Forward Collapse of the AR Head]\label{cor:ar-collapse}
Under the same weight configuration as Corollary~\ref{cor:construction}, the deep head
evaluated on the AR context~$  Z'$ outputs:
\[
  f_1(  Z') = e_{y^{(1)}}^{\top}
  \qquad
  \bigl(\text{instead of its target } e_{y^{(2)}}^{\top}\bigr).
\]
That is, the AR component of the MTP loss incurs nonzero error
$(\mathcal L_{1b}>0)$, yet produces a vanishingly small gradient:
$\|\nabla_\theta \mathcal L_{1b}\| = O(e^{-\gamma})$.
\end{corollary}
 
\begin{proof} [Proof of Collapse \ref{cor:ar-collapse}]
Recall that $  Z'$ differs from $  Z$ only at the last
position: $z_T' = e_{y^{(1)}}$ (i.e.\ $e_v$) instead of
$e_{u_{\mathrm{star}}}$.

Since Layer~1 uses pure positional bias ($W_0^{(1)}=0$), the attention
logits $A^{[1]}=W_1^{(1)}=\gamma\,L$ are content-independent.
Replacing $z_T$ does not change any entry of~$A^{[1]}$, so the
Layer~1 attention distribution on~$Z'$ coincides with that
on~$  X$:
\[
  S_{T,:}^{\prime\,[1]}
  = S_{T,:}^{[1]}
  = e_{T-1}^{\top} + O(T\,e^{-\gamma}).
\]
In particular, Layer~1 still retrieves the token at position $T-1$ as
the effective query for Layer~2.  Since $z_{T-1}'=z_{T-1}=e_{u_{\mathrm{end}}}$
(position $T-1$ is unchanged), Layer~2 receives the same query
$e_{u_{\mathrm{end}}}$.

The prefix of $  Z'$ (positions $1,\ldots,T-1$) is identical to
that of~$  Z$, so all context keys are unchanged.  The Layer~2
logits and attention distribution therefore remain the same:
$S_{T,:}^{\prime\,[2]} = S_{T,:}^{[2]}$.  Composing both layers
exactly as in the proof of Corollary~1:
\[
  f_1(  Z')
  = S_{T,:}^{\prime\,[2]}\;S^{\prime\,[1]}\; Z'\;
  = e_{t_v^{\mathrm{ctx}}}^{\top}\;  Z' + O(e^{-\gamma})
  = e_v^{\top} + O(e^{-\gamma})
  = e_{y^{(1)}}^{\top} + O(e^{-\gamma}).
\]

The AR target is $y^{(2)}=u_{\mathrm{end}} \neq v = y^{(1)}$, so
\[
  f_1(Z')\cdot e_{y^{(2)}}
  = O(e^{-\gamma})
  \quad\Longrightarrow\quad
  \mathcal L_{1b}
  = -\log \bigl(f_1(Z')\cdot e_{y^{(2)}}\bigr)
  = \gamma + O(1)
  > 0.
\]
However, every gradient chain in $\nabla_\theta\mathcal L_{1b}$ passes
through at least one softmax Jacobian $J(s)$ where $s$ is
$O(e^{-\gamma})$-close to a one-hot vector.  As established in the
proof of Corollary~1, each such Jacobian satisfies
$\|J(s)\|=O(e^{-\gamma})$, so
\[
  \|\nabla_\theta \mathcal L_{1b}\| = O(e^{-\gamma}).
\]
The error is effectively trapped: the loss is $\Theta(\gamma)$, but
the gradient is exponentially small.
\end{proof}

\subsection{Why MTP Finds It and NTP Cannot}

\begin{proof} [Proof of Theorem \ref{thm:cascaded}]

By Lemma~\ref{lemma:dl1},
\[
  \nabla_{\Wzop}\mathcal{L}_2 \;=\; -\frac{1}{f_2 \cdot e_{y^{(2)}}} 
   e_T e^\top_{y^{(2)}} Z^\top J(S^{[1]}_{T,:}),
\]
and this gradient is \emph{entirely independent of} $\Wonc$ and $\Wonp$.  Since only
$\mathcal{L}_2$ is used in Phase~1, the Layer~1 weights exert no influence on the
dynamics, regardless of their frozen values.

Under the Toeplitz constraint $W_{i,j}=w(i-j)$, all positions at the same
offset~$k$ share the weight~$w(k)$.  Under random graph labeling, the expected
gradient on $w(k)$ depends only on the offset~$k$, so the Toeplitz structure is
preserved exactly throughout Phase~1 (it is not merely approximately maintained).

$\mathcal{L}_2$ targets $y^{(2)}=u_{\mathrm{end}}$, which is always at position $T-1=9$.
With $T=10$, the normalisation is $s_p+8s_c+s_q=1$.  Let $s_p$ be the softmax
weight of the predecessor (offset $k=1$), $s_c$ the weight of each of the 8
context offsets $k\in\{2,\dots,9\}$, and $s_q$ the self-weight.
The target indicator satisfies $u_9=1$, $u_{10}=0$, and
$\E[u_t]=\frac{1}{8}$ for context positions $t\leq 8$.
 
The gradient flow equations for the Toeplitz weights are:
\begin{align}
  \dot w_p &= -s_p + \frac{s_p}{s_p+s_c}, \label{eq:wp}\\
  \dot w_c &= -s_c + \frac{1}{8}\cdot\frac{s_c}{s_p+s_c}. \label{eq:wc}
\end{align}
The first term in each equation arises from the softmax normalisation Jacobian;
the second term is the expected reward signal (the predecessor always hosts the
target; each of the 8 context offsets hosts it with probability $\frac{1}{8}$).

Let $x=s_p/s_c\geq 1$.  Subtracting \eqref{eq:wc} from \eqref{eq:wp}:
\[
  \dot\Delta \;=\; -s_c(x-1)+\frac{x-\frac{1}{8}}{x+1}.
\]
Using the normalization bound $s_c<\frac{1}{x+8}$ (since $s_q>0$, with $T-2=8$):
\begin{align*}
  \dot\Delta
  &> -\frac{x-1}{x+8}+\frac{x-\frac{1}{8}}{x+1}.
\end{align*}
Cross-multiplying (both denominators positive) and expanding, positivity is
equivalent to
\[
  \Bigl(x-\tfrac{1}{8}\Bigr)(x+8) - (x-1)(x+1)
  \;=\; x\underbrace{\Bigl(8-\tfrac{1}{8}\Bigr)}_{=\,\frac{63}{8}\,>\,0}
  \;=\; \frac{63x}{8}
  \;>\;0
\]
for all $x>0$.  Hence $\dot\Delta>0$ globally and unconditionally for $T=10$.

Note that at initialization $x=1$ (i.e.\ $w=0$), the gap derivative evaluates
  concretely to
  \[
    \dot\Delta\big|_{x=1} \;=\; \frac{T-3}{2(T-2)}\bigg|_{T=10}
    \;=\; \frac{7}{16} \;>\; 0,
  \]
  confirming that the predecessor pointer begins to emerge immediately from
  the zero initialization.

Since $\dot\Delta>0$ uniformly, $w(1)-w(k)\to+\infty$ for all $k\geq 2$.
Under the causal mask, row~$t$ of $\Azero=\Wzop$ has a unique maximum at
position~$t-1$ (offset~1) as $\gamma:=w(1)\to\infty$.  Therefore:
\[
  \Szero_{t,:} \;\to\; e^\top_{t-1} \qquad\forall\,t\geq 2.
\]
With $T=10$ this gives in particular: $\Szero_{10,:}=e^\top_9$ (Condition~1)
and $\Szero_{\tctxend,:}=e^\top_{\tctxv}$ for $\tctxend\in\{2,\dots,8\}$
(Condition~1).

With $\Wzop=\gamma L$ ($\gamma\to\infty$) frozen, we have $\Szero_{t,:}=e^\top_{t-1}$
for all $t\geq 2$.  With $T=10$, the effective Layer~1 query at the last position
is
\[
  (\Szero Z)_{10} \;=\; \Szero_{10,:} Z
  \;=\; e^\top_9 Z \;=\; z_9 \;=\; e_{u_{\mathrm{end}}}.
\]
The self-distractor (prompt position encoding $u_\text{end}$) is always at
position $T-1=9$.  The positional self-mask entry to be learned is therefore
$(\Wonp)_{9,9}$.

Since $\Szero=L$ is frozen, the deep head output becomes
$f_1=\Sone_{T,:}(LZ)$, where $(LZ)_j=z_{j-1}$.
The target $y^{(1)}=v$ satisfies $z_{j-1}=e_v$ iff $j=\tctxv+1=\tctxend$.
Hence:
\[
  \Ldeep \;=\; -\log \Sone_{T,\tctxend},
\]
a standard 1-hop contrastive problem over the Layer~2 attention distribution.
 
By gradient structure (from Lemma~\ref{lemma:dl0w}).
\begin{align}
  \nabla_{\Wonc}\Ldeep
    &= -\frac{1}{f_1 \cdot e_{y^{(1)}}}  \,Z^\top J(\Sone_{T,:})\,e_{\tctxend}\,e^\top_{u_{\mathrm{end}}},
    \label{eq:gradW1c}\\
  \nabla_{\Wonp}\Ldeep
    &= -\frac{1}{f_1 \cdot e_{y^{(1)}}}  \,J(\Sone_{T,:})\,e_{\tctxend}\,e^\top_{T-1}.
    \label{eq:gradW1p}
\end{align}
Both are rank-1 outer products.  Equation~\eqref{eq:gradW1c} modifies only
row~$u_{\mathrm{end}}$ of $\Wonc$; equation~\eqref{eq:gradW1p} modifies only
row~$T-1$ of $\Wonp$.  Starting from $\Wonc=0$ and $\Wonp=0$, all other
columns remain identically zero throughout Phase~2.  This establishes {Conclusion~(i)}.

The logit from position~$10$ to position~$j$ in Layer~1 is:
\[
  \ell_j \;=\; \underbrace{e^\top_{u_{\mathrm{end}}}\Wonc z_j}_{\text{content}}
           \;+\; \underbrace{(\Wonp)_{9,j}}_{\text{positional}}.
\]
The positional term indexes \emph{row~$9$} of $\Wonp$ (not row~$10$), because
$\Szero_{10,:}=e^\top_9$ shifts the effective query row by one.  In particular,
the self-mask entry is $(\Wonp)_{9,9}$.

For a position $k$ such that $z_k\neq e_{u_{\mathrm{end}}}$ (a content distractor),
the logit gap gradient is:
\[
  \dot\Delta_k \;=\; \dot\ell_{\tctxend}-\dot\ell_k
  \;\geq\; (1-s_{j^*}-s_{\mathrm{self}})+s_k
  \;\geq\; 2s_k \;>\; 0,
\]
where $j^*=\tctxend$ and the bound uses $1-s_{j^*}-s_{\mathrm{self}}\geq\sum_{k'}s_{k'}\geq s_k$
from the normalisation constraint.  Thus all content distractor gaps diverge.

Position $k=9$ ($=T-1$) encodes $z_9=e_{u_{\mathrm{end}}}$ (same node as the query).
Content updates to $\Wonc$ cancel for this position (query equals key).  The gap
is governed entirely by the positional bias $(\Wonp)_{9,9}$:
\[
  \dot\ell_{j^*}  = (1-s_{j^*}-s_{\mathrm{self}})+(1-s_{j^*}),\qquad
  \dot\ell_{\mathrm{self}} = (1-s_{j^*}-s_{\mathrm{self}})-s_{\mathrm{self}}.
\]
\[
  \dot\Delta_{\mathrm{self}} \;=\; \dot\ell_{j^*}-\dot\ell_{\mathrm{self}}
  \;=\; 1-s_{j^*}+s_{\mathrm{self}} \;>\; 0.
\]
This drives $(\Wonp)_{9,9}\to-\infty$, establishing {Conclusion~(iii)}.

From \eqref{eq:gradW1c}, the update to entry~$(i,u_{\mathrm{end}})$ of $\Wonc$ is
\[
  \dot{(\Wonc)}_{i,u_{\mathrm{end}}}
  = \frac{s_{\tctxend}}{\mu}\bigl(\delta_{i,u_{\mathrm{end}}}
    - (Z^\top s)_{u_{\mathrm{end}}}\bigr).
\]
The diagonal entry $(i=u_{\mathrm{end}})$ receives a positive update of
$\frac{s_{j^*}(1-s_{j^*})}{\mu}>0$ and diverges to $+\infty$.  For off-diagonal
entries $i\neq u_{\mathrm{end}}$: as $\Wonc {u_{\mathrm{end}},u_{\mathrm{end}}}\to+\infty$,
the attention weight on non-$u_{\mathrm{end}}$ content positions vanishes,
$(Z^\top s)_i\to 0$.  Therefore
$\int_0^\infty|\dot{(\Wonc)}_{i,u_{\mathrm{end}}}|\,dt<\infty$, and off-diagonal
entries converge to a finite negative value $B_{i,u_{\mathrm{end}}}<0$.
By graph-labeling symmetry, the same argument holds for every column, giving
$\Wonc\to\gamma' I+B$ with $B_{ii}=0$ and $B_{ij}<0$ for $i\neq j$.
This establishes {Conclusion~(ii)}.

\end{proof}

\begin{proof} [Proof of Theorem \ref{thm:ntp-misdirected}]

The full derivative of $\mu = \frac{1}{T}\sum_{r=1}^T (\Szero)_{r,:}\,u$
with respect to $w(k)$ is obtained by differentiating each row:
\[
  \frac{\partial \mu}{\partial w(k)}
  = \frac{1}{T}\sum_{r=k+1}^{T}
    \frac{\partial\bigl[(\Szero)_{r,:}\,u\bigr]}{\partial w(k)}.
\]
(Rows $r \leq k$ have no position at offset $k$ within their visible window
and contribute zero.)
 
At $w \equiv 0$, row $r$ has uniform attention $s_r = \frac{1}{r} {1}$
over positions $1,\ldots,r$.  The Jacobian of softmax at $s_r$ is
$J(s_r) = \frac{1}{r}I - \frac{1}{r^2} {1} {1}^\top$.
Since $\frac{\partial \Azero_{r,:}}{\partial w(k)} = e_{r-k}^\top$
(only entry $j=r-k$ is affected), the chain rule gives:
\begin{equation}
  \frac{\partial\bigl[(\Szero)_{r,:}\,u\bigr]}{\partial w(k)}
  = \left(J(s_r)\,e_{r-k}\right)^\top u
  = \frac{1}{r} \left(u_{r-k} - \frac{1}{r}\sum_{j=1}^r u_j\right).
  \label{eq:row-deriv}
\end{equation}

We take expectation over random graph labeling.
 
{Rows $r \leq T-2 = 8$.}
For these rows, $r - k \leq T - 3 \leq T-2$, so $\E[u_{r-k}] = \frac{1}{T-2}$
for all $k \geq 1$.  Also:
\[
  \E \left[\frac{1}{r}\sum_{j=1}^r u_j\right]
  = \frac{1}{r}\cdot\frac{r}{T-2} = \frac{1}{T-2}.
\]
Therefore $\E \left[\frac{1}{r}\bigl(u_{r-k} - \frac{1}{r}\sum_{j=1}^r u_j\bigr)\right] = 0$.
 
Every row $r \leq T-2$ contributes {exactly zero} in expectation.
This is the symmetric regime: the target is equally likely at every
context position, so no offset is favoured.
 
{Row $r = T-1 = 9$.}
Here $\sum_{j=1}^{T-1}\E[u_j] = 1$ (since $\E[u_{T-1}]=0$ and
$\E[u_T]=0$, all probability mass is on positions $1,\ldots,T-2$).
For all offsets $k \geq 1$: $r - k = T-1-k \leq T-2$, so
$\E[u_{T-1-k}] = \frac{1}{T-2}$.  Equation~\eqref{eq:row-deriv} gives:
\[
  \E \left[\frac{\partial\bigl[(\Szero)_{T-1,:}\,u\bigr]}{\partial w(k)}\right]
  = \frac{1}{T-1} \left(\frac{1}{T-2} - \frac{1}{T-1}\right)
  = \frac{1}{(T-1)^2(T-2)} > 0
\]
for \emph{all} $k \geq 1$. Row $T-1$ contributes the {same positive term}
for every offset.
 
{Row $r = T = 10$.}
Here $\sum_{j=1}^{T}\E[u_j] = 1$.  Two sub-cases:
 
\begin{itemize}
  \item {$k=1$ (predecessor offset):} $r-k = T-1$, so
        $\E[u_{T-1}] = 0$ (position $T-1$ encodes $u_{\mathrm{end}} \neq v$).
        Equation~\eqref{eq:row-deriv} gives:
        \[
          \frac{1}{T} \left(0 - \frac{1}{T}\right) = -\frac{1}{T^2} < 0.
        \]
 
  \item  {$k \geq 2$ (context offsets):} $r-k = T-k \leq T-2$, so
        $\E[u_{T-k}] = \frac{1}{T-2}$.  Equation~\eqref{eq:row-deriv} gives:
        \[
          \frac{1}{T} \left(\frac{1}{T-2} - \frac{1}{T}\right)
          = \frac{2}{T^2(T-2)} > 0.
        \]
\end{itemize}

Combining all three row regimes, the full expected gradient of $\mu$ is:
\begin{align}
  \E \left[\frac{\partial \mu}{\partial w(1)}\right]
  &= \frac{1}{T} \left[
      \underbrace{0}_{\text{rows }r\leq T-2}
    + \underbrace{\frac{1}{(T-1)^2(T-2)}}_{\text{row }T-1}
    - \underbrace{\frac{1}{T^2}}_{\text{row }T}
    \right], \label{eq:dmu1}\\[6pt]
  \E \left[\frac{\partial \mu}{\partial w(k)}\right]
  &= \frac{1}{T} \left[
      \underbrace{0}_{\text{rows }r\leq T-2}
    + \underbrace{\frac{1}{(T-1)^2(T-2)}}_{\text{row }T-1}
    + \underbrace{\frac{2}{T^2(T-2)}}_{\text{row }T}
    \right] > 0, \quad k\geq 2. \label{eq:dmuc}
\end{align}

From \eqref{eq:dmu1}, the predecessor gradient has negative sign iff
$(T-1)^2(T-2) > T^2$.  Expanding:
\[
  (T-1)^2(T-2) - T^2
  = T^3 - 5T^2 + 5T + 2.
\]
For $T =10$: $T^3 - 5T^2 + 5T - 2 > 0$.  Therefore
$(T-1)^2(T-2) > T^2$, giving:
\[
  \E \left[\frac{\partial \mu}{\partial w(1)}\right]
  = \frac{1}{T} \left(\frac{1}{(T-1)^2(T-2)} - \frac{1}{T^2}\right) < 0.
\]

From $\mathcal{L}_1 = -\log\mu$ and the chain rule:
\[
  \E \left[\frac{\partial \mathcal{L}_1}{\partial w(k)}\right]
  = -\frac{1}{\mu_0}\E \left[\frac{\partial \mu}{\partial w(k)}\right].
\]
Substituting \eqref{eq:dmu1} and \eqref{eq:dmuc}:
 
\begin{align*}
  \E \left[\frac{\partial \mathcal{L}_1 }{\partial w(1)}\right]
  &= \frac{1}{T\mu_0} \left(\frac{1}{T^2} - \frac{1}{(T-1)^2(T-2)}\right)
  > 0, \\[6pt]
  \E \left[\frac{\partial \mathcal{L}_1}{\partial w(k)}\right]
  &= -\frac{1}{T\mu_0} \left(\frac{1}{(T-1)^2(T-2)}
     + \frac{2}{T^2(T-2)}\right)
  < 0, \quad k \geq 2.
\end{align*}

At $T=10$: $T-1=9$, $T-2=8$, $(T-1)^2(T-2) = 81 \times 8 = 648$.
 
\[
  \E \left[\frac{\partial L_0}{\partial w(1)}\right]
  = \frac{1}{10\mu_0} \left(\frac{1}{100} - \frac{1}{648}\right)
  = \frac{1}{10\mu_0}\cdot\frac{648-100}{64800}
  = \frac{548}{648{,}000\,\mu_0} > 0.
\]
 
\[
  \E \left[\frac{\partial L_0}{\partial w(k)}\right]
  = -\frac{1}{10\mu_0} \left(\frac{1}{648} + \frac{2}{800}\right)
  = -\frac{1}{10\mu_0}\cdot\frac{800 + 2\times 648}{648\times 800}
  = -\frac{2096}{5{,}184{,}000\,\mu_0} < 0.
\]
 
Gradient descent therefore:
\begin{itemize}
  \item  {Decreases} $w(1)$: the predecessor weight is \emph{actively
        suppressed}.
  \item {Increases} $w(k)$ for $k\geq 2$: attention is
        \emph{diffused} toward context positions.
\end{itemize}
This is precisely the opposite of the predecessor pointer
$w(1) \to +\infty$, $w(k) \to -\infty$ for $k \geq 2$ required by the
stationary manifold of Theorem~\ref{thm:stationary}.

\end{proof}

\section{Additional Experiments Details}

\subsection{Experimental Setup for \cref{sec:exp}}
\label{app:exp_dense}

In this section, we detail the experimental configurations for our standard multi-layer transformer models. All models are trained using the AdamW optimizer \citep{loshchilov2018decoupled} with a weight decay of $0$. 

\textbf{Graph.} We conduct scaling experiments on 2-path 5-node star graphs and 5-node binary trees. To isolate the scaling effect of interest, when scaling one dimension (e.g., model capacity), we ensure the complementary dimension (e.g., dataset size) is sufficient to achieve performance saturation. Models are trained for 100 epochs with a batch size of 256 and a learning rate of $1 \times 10^{-4}$.

\begin{remark}
    For a fixed graph structure (e.g., the binary tree in \cref{fig:star_vs_bin}), the set of possible problem instances is finite. 
    With sufficient data (e.g., $0.5$M samples) and large models (e.g., $10.65$M parameters), NTP might memorize all possible problems and achieve perfect test performance. 
    In this regime, the test set no longer reflects truly unseen problems. 
    Thus, comparing MTP and NTP is more meaningful under limited data or model capacity, where memorization is unlikely and planning must be learned.
\end{remark}

\textbf{Countdown.} For the countdown task, we utilize a 4-level arithmetic setting featuring an 85M parameter model and a dataset of 1M samples. All numerical inputs and intermediate results are constrained to be less than 100. Training is performed for 100 epochs with a batch size of 1024 and a learning rate of $1 \times 10^{-4}$. Given that the model typically reaches peak performance early in the training process, we report the best performance achieved.

\textbf{3-SAT.} In the 3-SAT task, we focus on solving Boolean formulas featuring 7 variables and 45 clauses. We train a 5.3M parameter model on 50k samples for 300 epochs, using a batch size of 512 and a learning rate of $1 \times 10^{-3}$. We report the model's final performance at the conclusion of training.

\subsection{Experimental Setup for \cref{sec:theory}}
\label{app:exp_sparse}

In this section, we detail the experimental setup for our 2-layer disentangled transformer models trained on 2-path 3-node star graphs. All models are trained using the AdamW optimizer with a batch size of 512 and a learning rate of $5 \times 10^{-4}$ on 1M generated examples for 300 epochs.

\subsection{Empirical Validation on the Full Standard Architecture} \label{sec:valid_full}

In the main paper, our theoretical analysis relied on a disentangled mechanism and further mathematical simplifications to achieve tractability. In this section, we relax these structural assumptions and empirically evaluate whether our core findings generalize to a standard, full Transformer architecture.

\textbf{Experimental Setup.}
We instantiate a standard 8-layer 8-head Transformer model with 25M parameters. The model is trained on a 5-path 5-node star graph dataset consisting of 1M training examples. We train the model for 40 epochs using a batch size of 1024 and a learning rate of $1\times 10^{-4}$. An example of the data format is provided below:

\begin{tcolorbox}[colback=gray!5!white, colframe=gray!50!black, title=Data Example]
74,64 $|$ 3,36 $|$ 49,63 $|$ 40,16 $|$ 31,73 $|$ 73,18 $|$ 51,22 $|$ 49,46 $|$ 38,19 $|$ 13,27 $|$ 46,40 $|$ 49,74 $|$ 63,31 $|$ 65,13 $|$ 64,3 $|$ 49,61 $|$ 19,51 $|$ 61,65 $|$ 49,38 $|$ 16,41 / 49,18 = 49,63,31,73,18
\end{tcolorbox}

To understand the internal mechanisms driving the model's performance, we extract the attention maps specifically during the generation of the ``difficult'' token. The visualizations presented in this section are obtained by averaging the attention weights across 8 attention heads and 544 test examples.

\textbf{Attention Analysis.}
We compare the standard NTP against 3-MTP. The results reveal a stark contrast in both generalization capability and internal attention routing.

As shown in Figure~\ref{fig:ntp_full_model}, the model trained with standard NTP suffers from severe overfitting. While it achieves 97\% training accuracy, the test accuracy only reaches 20\%. Analyzing the attention heatmaps, we observe that the NTP model primarily attends to the \textit{start node}, specifically in Layers 7 and 8.

\begin{figure}[tb!]
    \centering
    \begin{subfigure}[b]{0.23\textwidth}
        \centering
        \caption{Layer 1 Attention}
        \includegraphics[width=\textwidth]{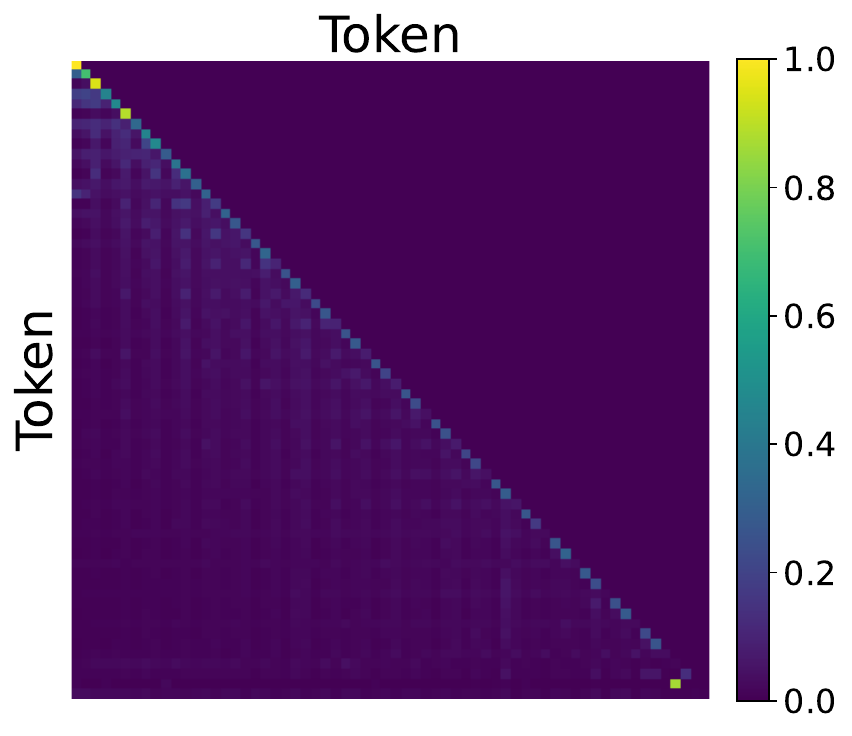}
    \end{subfigure}
    \hfill
    \begin{subfigure}[b]{0.23\textwidth}
        \centering
        \caption{Layer 2 Attention}
        \includegraphics[width=\textwidth]{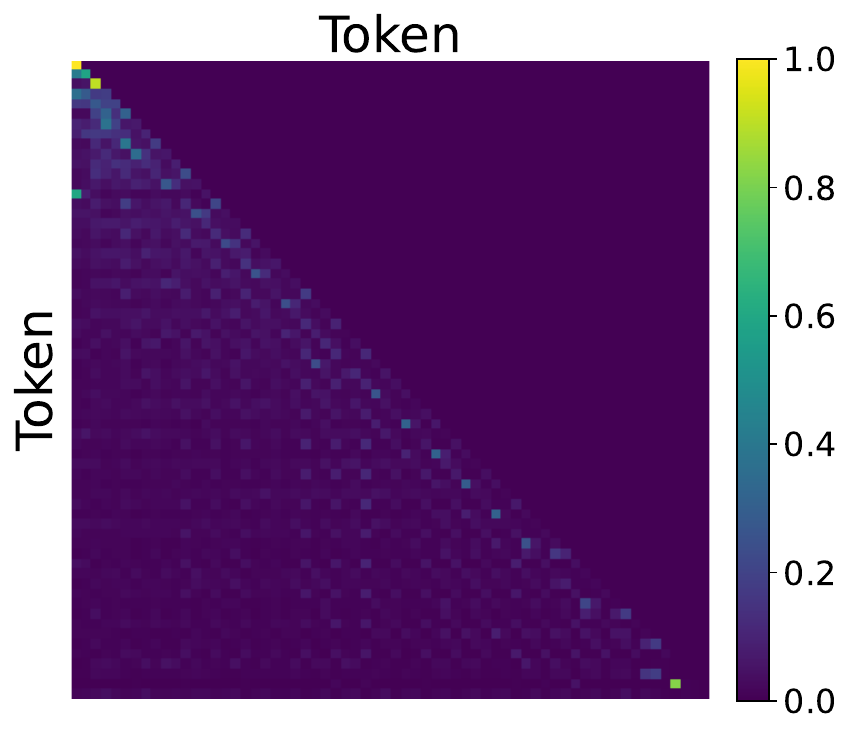}
    \end{subfigure}
    \hfill
    \begin{subfigure}[b]{0.23\textwidth}
        \centering
        \caption{Layer 3 Attention}
        \includegraphics[width=\textwidth]{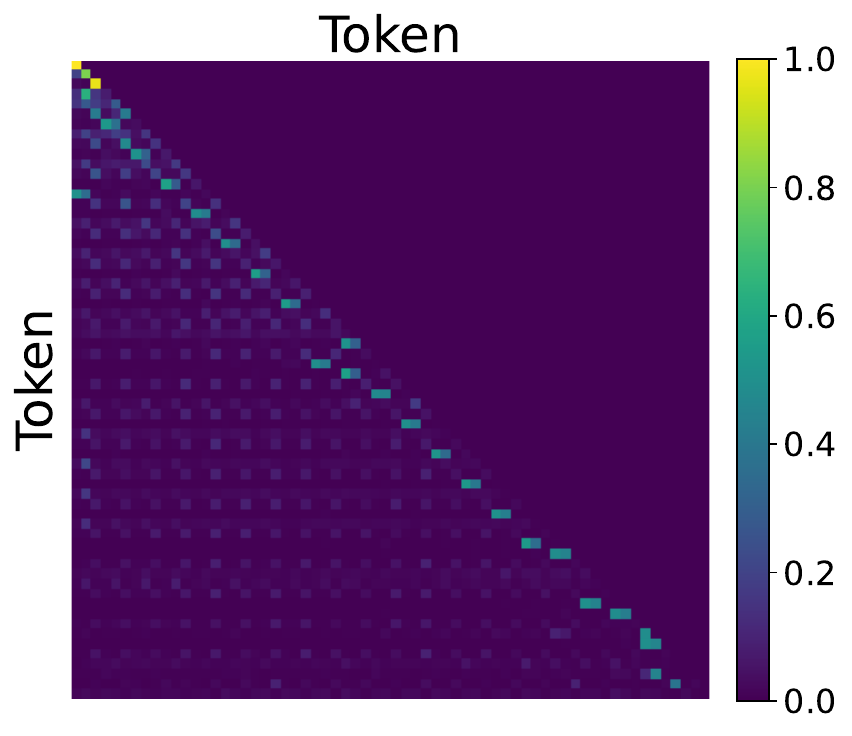}
    \end{subfigure}
    \hfill
    \begin{subfigure}[b]{0.23\textwidth}
        \centering
        \caption{Layer 4 Attention}
        \includegraphics[width=\textwidth]{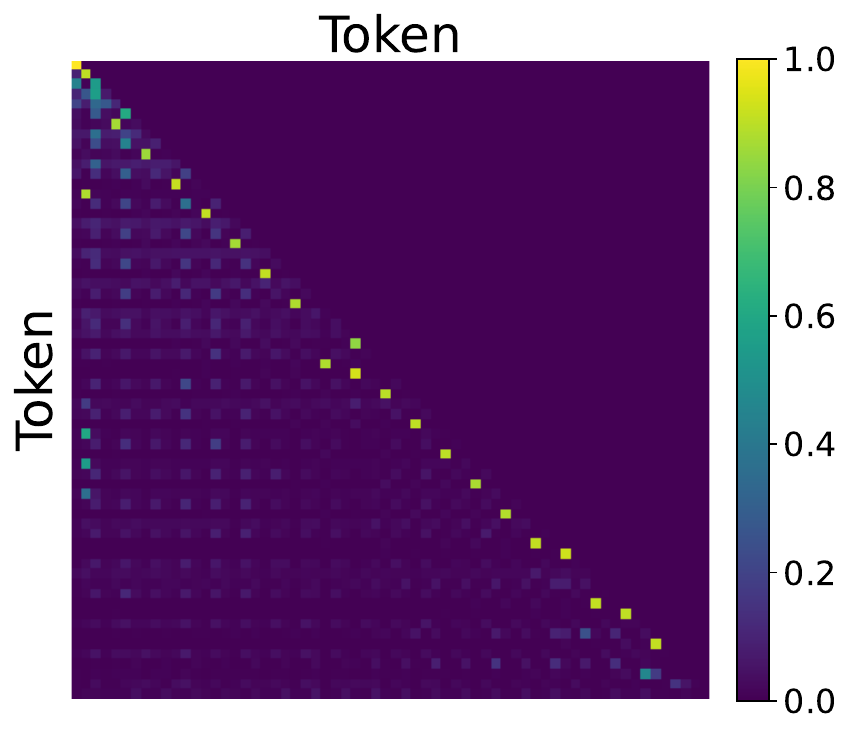}
    \end{subfigure}

    \vspace{1em}

    \begin{subfigure}[b]{0.23\textwidth}
        \centering
        \caption{Layer 5 Attention}
        \includegraphics[width=\textwidth]{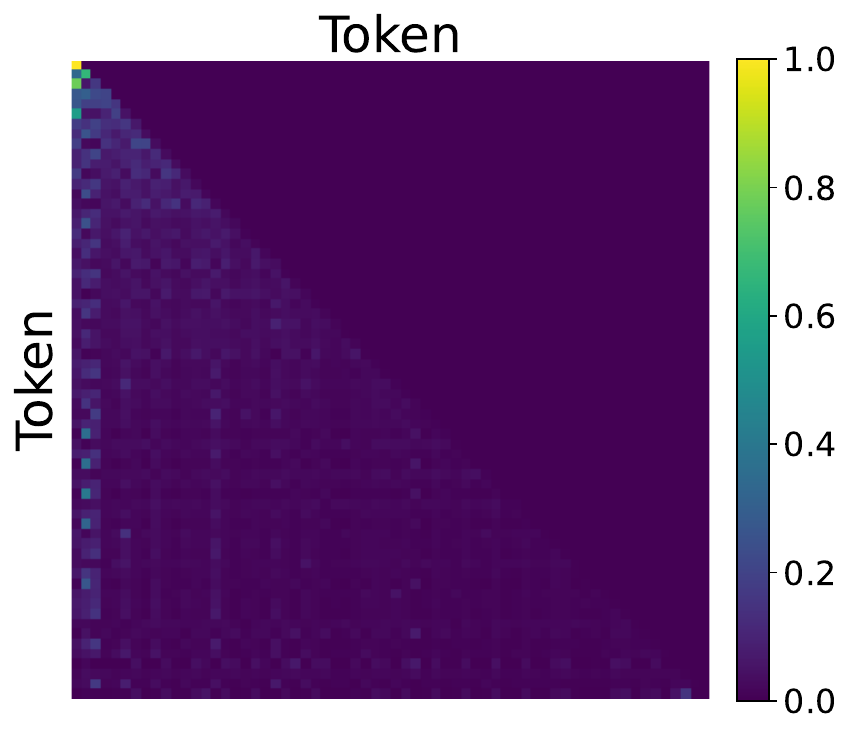}
    \end{subfigure}
    \hfill
    \begin{subfigure}[b]{0.23\textwidth}
        \centering
        \caption{Layer 6 Attention}
        \includegraphics[width=\textwidth]{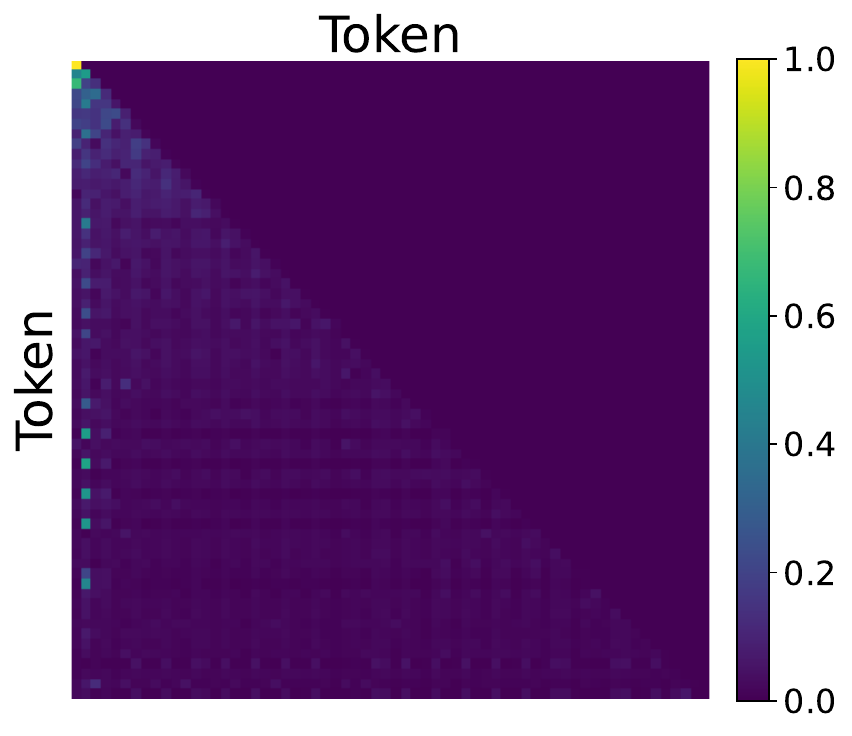}
    \end{subfigure}
    \hfill
    \begin{subfigure}[b]{0.23\textwidth}
        \centering
        \caption{Layer 7 Attention}
        \includegraphics[width=\textwidth]{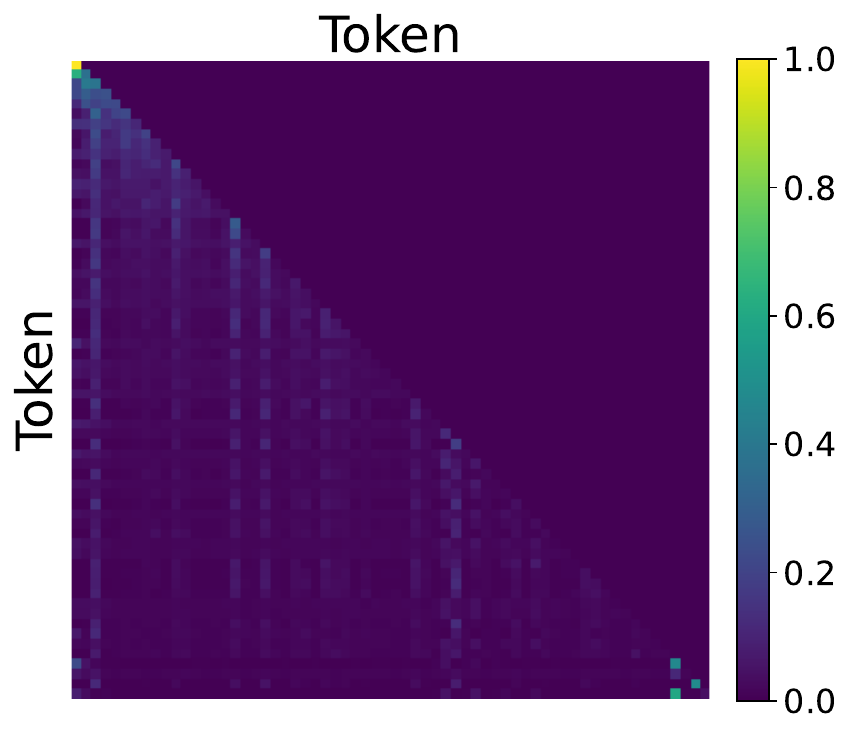}
    \end{subfigure}
    \hfill
    \begin{subfigure}[b]{0.23\textwidth}
        \centering
        \caption{Layer 8 Attention}
        \includegraphics[width=\textwidth]{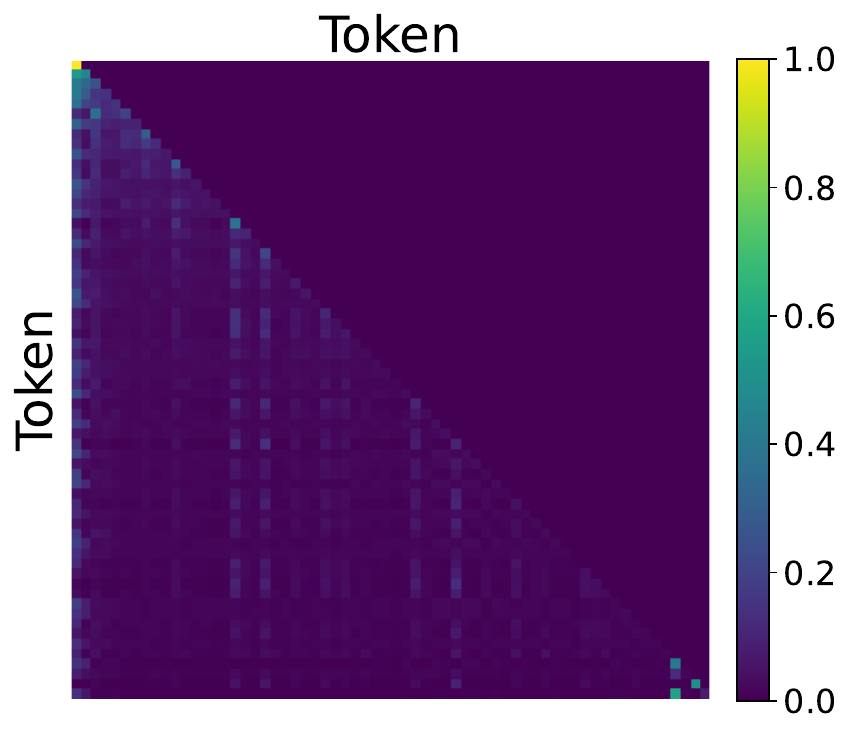}
    \end{subfigure}

    \caption{Attention heatmaps for the model trained with the standard NTP objective. The NTP model exhibits severe overfitting, achieving a training accuracy of 97\% while the test accuracy only reaches 20\%. Furthermore, the visualization demonstrates that the model primarily attends to the start node, specifically in Layers 7 and 8.}
    \label{fig:ntp_full_model}
\end{figure}

In contrast, Figure~\ref{fig:mtp_full_model} illustrates the behavior of the model trained with 3-MTP. This model achieves perfect generalization, reaching 100\% accuracy on both the training and test sets. Crucially, the attention mechanism operates entirely differently: rather than attending to the start node, the 3-MTP model identifies and heavily attends to the \textit{end node}, predominantly in Layers 3 and 4.

\begin{figure}[tb!]
    \centering
    \begin{subfigure}[b]{0.23\textwidth}
        \centering
        \caption{Layer 1 Attention}
        \includegraphics[width=\textwidth]{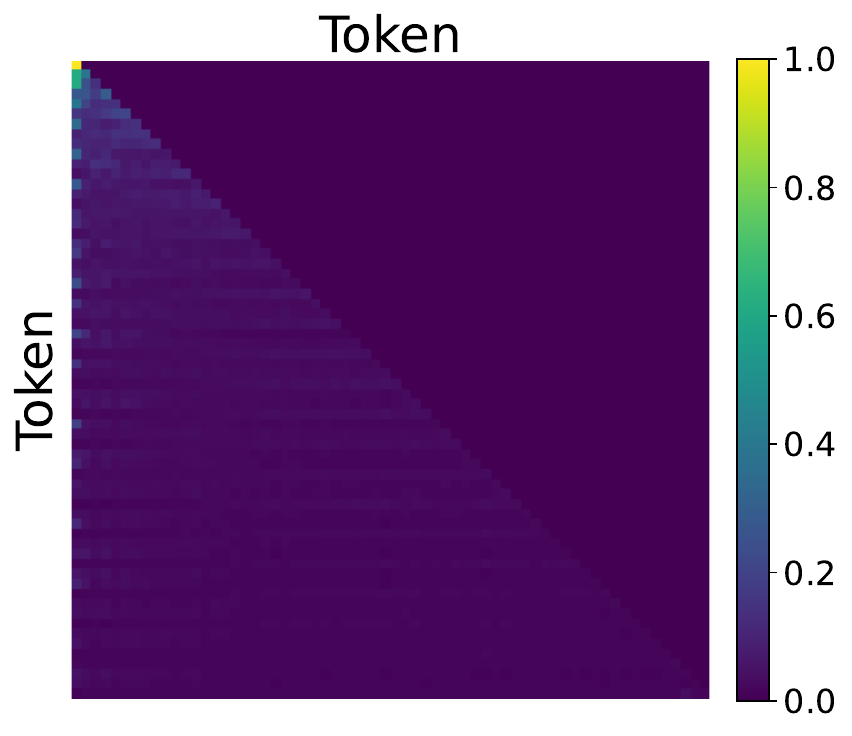}
    \end{subfigure}
    \hfill
    \begin{subfigure}[b]{0.23\textwidth}
        \centering
        \caption{Layer 2 Attention}
        \includegraphics[width=\textwidth]{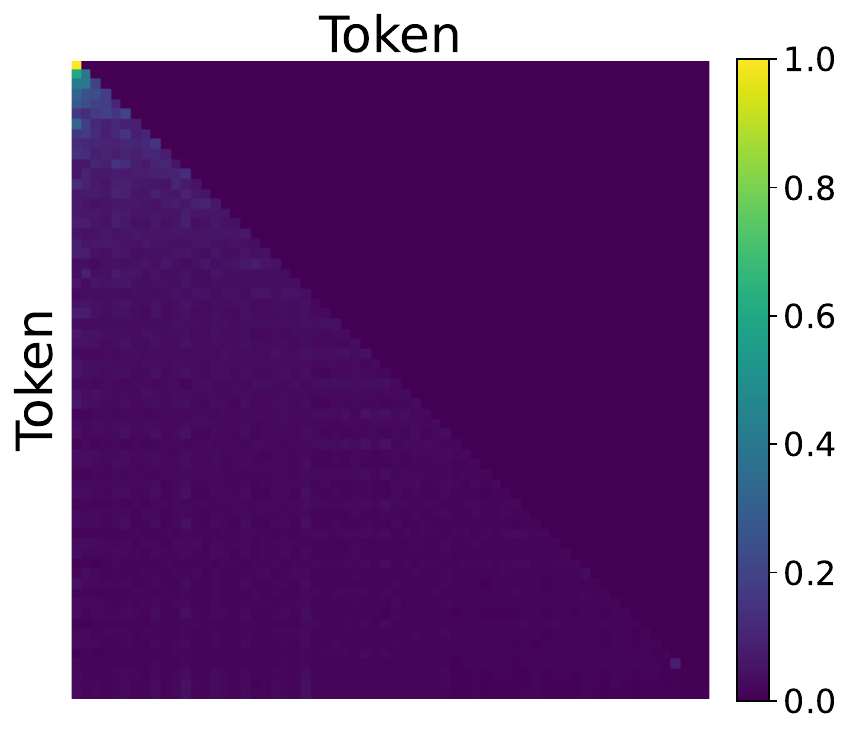}
    \end{subfigure}
    \hfill
    \begin{subfigure}[b]{0.23\textwidth}
        \centering
        \caption{Layer 3 Attention}
        \includegraphics[width=\textwidth]{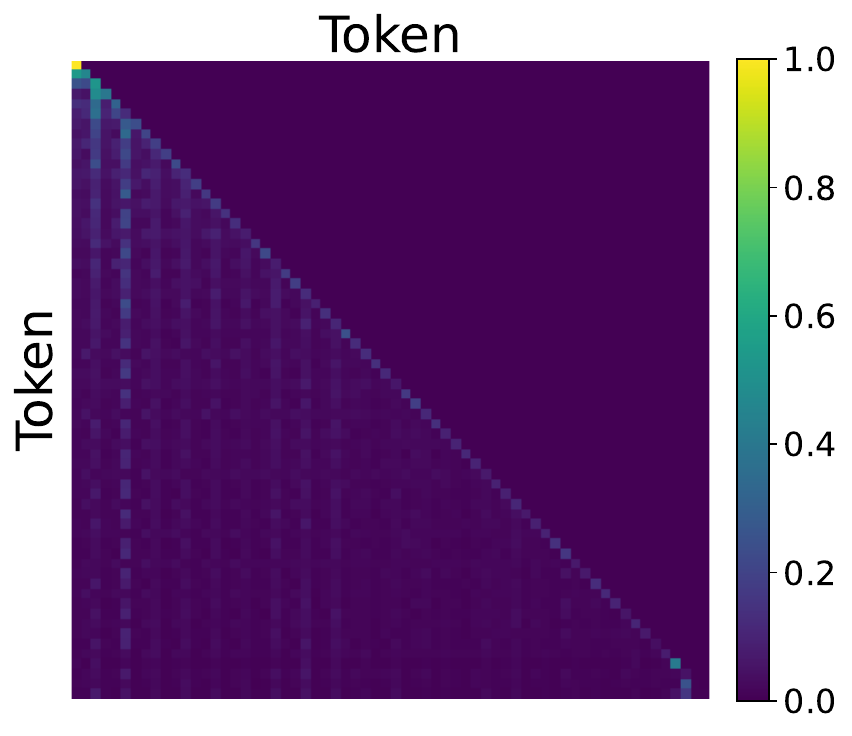}
    \end{subfigure}
    \hfill
    \begin{subfigure}[b]{0.23\textwidth}
        \centering
        \caption{Layer 4 Attention}
        \includegraphics[width=\textwidth]{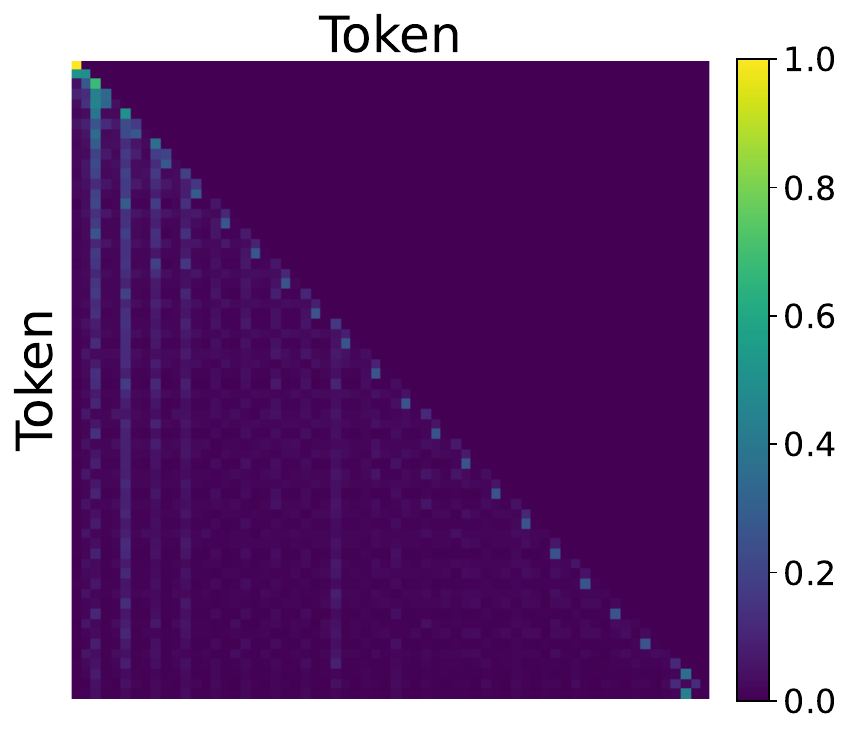}
    \end{subfigure}

    \vspace{1em}

    \begin{subfigure}[b]{0.23\textwidth}
        \centering
        \caption{Layer 5 Attention}
        \includegraphics[width=\textwidth]{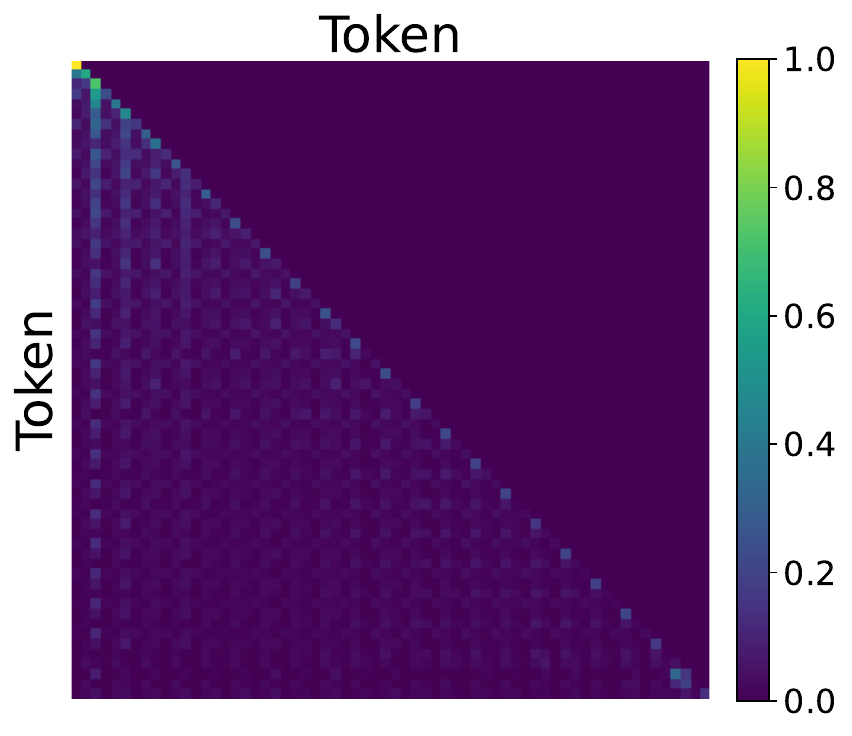}
    \end{subfigure}
    \hfill
    \begin{subfigure}[b]{0.23\textwidth}
        \centering
        \caption{Layer 6 Attention}
        \includegraphics[width=\textwidth]{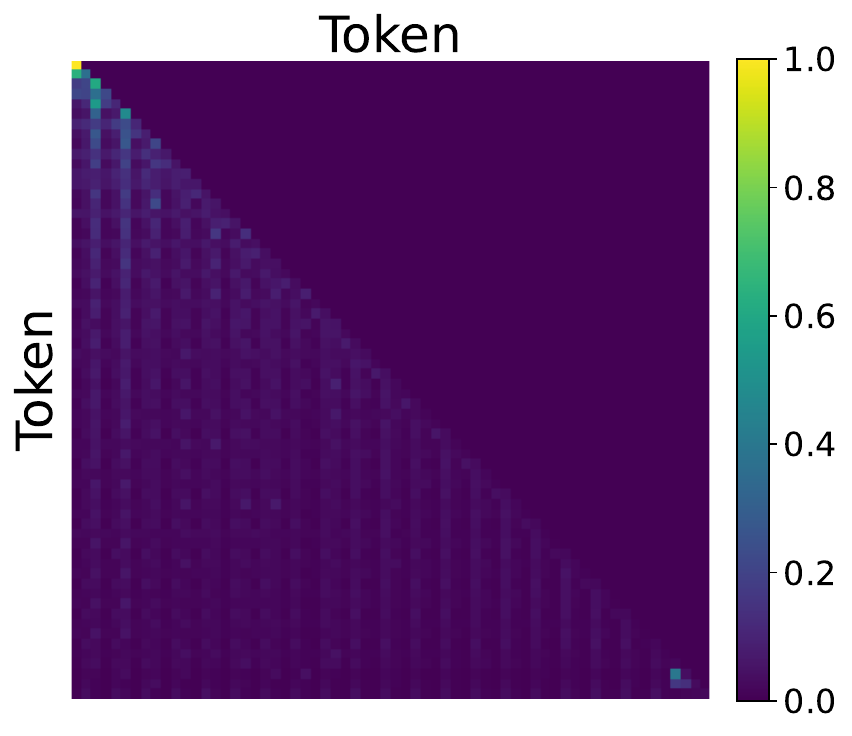}
    \end{subfigure}
    \hfill
    \begin{subfigure}[b]{0.23\textwidth}
        \centering
        \caption{Layer 7 Attention}
        \includegraphics[width=\textwidth]{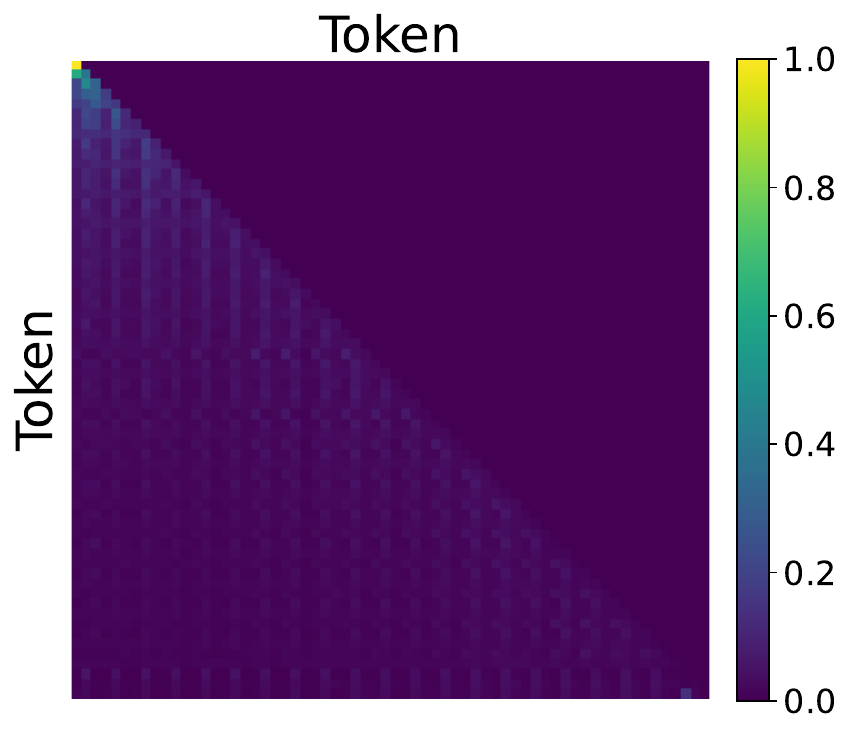}
    \end{subfigure}
    \hfill
    \begin{subfigure}[b]{0.23\textwidth}
        \centering
        \caption{Layer 8 Attention}
        \includegraphics[width=\textwidth]{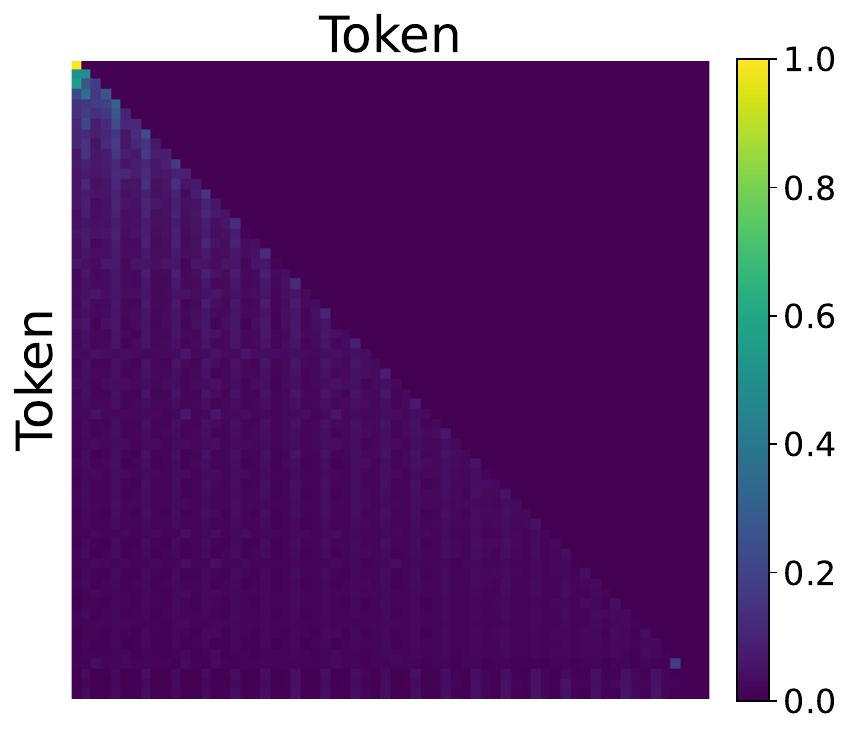}
    \end{subfigure}

    \caption{Attention heatmaps for the model trained with the 3-MTP objective. In contrast to the NTP model, the 3-MTP model achieves perfect generalization with 100\% accuracy on both the training and test sets. Notably, the visualization highlights that the 3-MTP model primarily attends to the end node, predominantly in Layers 3 and 4.}
    \label{fig:mtp_full_model}
\end{figure}

These empirical results perfectly corroborate our theoretical findings. The observed structural shift in attention—specifically, prioritizing the end node in Layers 3 and 4—confirms that the MTP objective successfully induces a \emph{reverse reasoning} mechanism, which is the key driver of the perfect generalization.

\end{document}